\documentclass{article}

\usepackage[numbers]{natbib}

\usepackage{amsmath}
\usepackage{amssymb,amsfonts}
\usepackage{bm}

\usepackage[paper=a4paper, hscale=0.75, vscale=0.79]{geometry} %

\usepackage{graphicx} %
\graphicspath{{./img/}} %
\DeclareGraphicsExtensions{.pdf,.png,.jpg,.eps} %
\usepackage{subfig} %

\newcommand{\overbar}[1]{\mkern 1.5mu\overline{\mkern-1.5mu#1\mkern-1.5mu}\mkern 1.5mu}

\newcommand*\dif{\mathop{}\!\mathrm{d}}
\newcommand{\Abs}[1]{\lVert{#1}\rVert} 
\newcommand{\AbsLR}[1]{\left\lVert{#1}\right\rVert} 
\newcommand{\abs}[1]{\lvert{#1}\rvert} 
\def\cond{\,\vert\,} 
\def\ddist{\mu} 
\newcommand{\dv}[1]{{#1}^{\prime}} 
\newcommand{\ddv}[1]{{#1}^{\prime\prime}} 
\DeclareMathOperator{\Div}{Div} 
\def\DRO{\textup{DRO}} 
\DeclareMathOperator{\dsp}{D} 
\DeclareMathOperator{\exx}{\mathbf{E}} 
\def\HH{\mathcal{H}} 
\DeclareMathOperator{\indic}{I} 
\DeclareMathOperator{\loss}{\mathsf{L}} 
\def\LL{\mathcal{L}} 
\DeclareMathOperator{\mloc}{M} 
\def\mt{\theta} 
\def\OO{\mathcal{O}} 
\DeclareMathOperator{\prr}{\mathbf{P}} 
\DeclareMathOperator{\quant}{Q} 
\newcommand{\rdv}[1]{\mathsf{#1}} 
\DeclareMathOperator{\risk}{R} 
\def\RR{\mathbb{R}} 
\def\samplespace{\Omega} 
\def\sigmafield{\mathcal{F}} 
\DeclareMathOperator{\sign}{sign} 
\def\smooth{\lambda} 
\newcommand{\tran}[1]{{#1}^{\top}} 
\def\UU{\mathcal{U}} 
\DeclareMathOperator{\vaa}{var} 
\def\VV{\mathcal{V}} 
\def\XX{\mathcal{X}} 

\DeclareMathOperator*{\argmin}{arg\,min}

\usepackage{enumitem}
\makeatletter
\def\namedlabel#1#2{\begingroup
    #2%
    \def\@currentlabel{#2}%
    \phantomsection\label{#1}\endgroup
}
\makeatother
\newcommand*{\defeq}{\mathrel{\vcenter{\baselineskip0.5ex \lineskiplimit0pt     
                     \hbox{\scriptsize.}\hbox{\scriptsize.}}}=}

\usepackage{amsthm} 
\theoremstyle{definition} \newtheorem{defn}{Definition}
\theoremstyle{plain} \newtheorem{prop}[defn]{Proposition}
\theoremstyle{plain} \newtheorem{thm}[defn]{Theorem}
\theoremstyle{plain} \newtheorem{lem}[defn]{Lemma}
\theoremstyle{plain} \newtheorem{cor}[defn]{Corollary}
\theoremstyle{remark} \newtheorem{rmk}[defn]{Remark}
\theoremstyle{remark} 

\usepackage[utf8]{inputenc} 
\usepackage[T1]{fontenc}    
\usepackage[pdfauthor={Matthew J. Holland},%
colorlinks=true,%
linkcolor=blue,%
citecolor=blue]{hyperref}
\hypersetup{pdftitle={Flexible risk design using bi-directional dispersion}} 
\usepackage{url}            
\usepackage{booktabs}       
\usepackage{nicefrac}       
\usepackage{microtype}      
\usepackage[dvipsnames]{xcolor}         
\usepackage{algorithm} 
\usepackage{algpseudocode} 
\newcommand{\term}[1]{\textit{{#1}}}

\title{\textbf{Flexible risk design using bi-directional dispersion}}
\author{
  Matthew J.~Holland\\
  Osaka University
}
\date{} 

\begin{document}

\maketitle

\begin{abstract}
Many novel notions of ``risk'' (e.g., CVaR, tilted risk, DRO risk) have been proposed and studied, but these risks are all at least as sensitive as the mean to loss tails on the upside, and tend to ignore deviations on the downside. We study a complementary new risk class that penalizes loss deviations in a bi-directional manner, while having more flexibility in terms of tail sensitivity than is offered by mean-variance. This class lets us derive high-probability learning guarantees without explicit gradient clipping, and empirical tests using both simulated and real data illustrate a high degree of control over key properties of the test loss distribution incurred by gradient-based learners.
\end{abstract}

\tableofcontents

\section{Introduction}\label{sec:intro}

What does it mean for a learner to successfully generalize? Broadly speaking, this is an ambiguous property of learning systems that can be defined, measured, and construed in countless ways. In the context of machine learning, however, the notion of ``success'' in off-sample generalization is almost without exception formalized as minimizing the expected value of a random loss $\exx_{\ddist} \loss(h)$, where $h$ is a candidate parameter, model, or decision rule, and $\loss(h)$ is a random variable on a probability space $(\samplespace,\sigmafield,\ddist)$ \citep{mohri2012Foundations,shalev2014a}. The idea of quantifying the \emph{risk} of an unexpected outcome (here, a random loss) using the expected value dates back to the Bernoullis and Gabriel Cramer in the early 18th century \citep{bassett1987a,hacking2006ProbHistory}. In a more modern context, the emphasis on \emph{average} performance is the ``general setting of the learning problem'' of \citet{vapnik1999NSLT}, and plays a central role in the decision-theoretic learning model of \citet{haussler1992a}. Use of the expected loss to quantify off-sample generalization has been essential to the development of both the statistical and computational theories of learning \citep{devroye1996ProbPR,kearns1994CLTIntro}.

While the expected loss still remains pervasive, important new lines of work on \term{risk-sensitive learning} have begun exploring novel feedback mechanisms for learning algorithms, in some cases derived directly from new risk functions that replace the expected loss. Learning algorithms designed using conditional value-at-risk (CVaR) \citep{curi2020a,holland2021c} and tilted (or ``entropic'') risk \citep{follmer2011a,li2021b,li2021a} are well-known examples of location properties which emphasize loss tails in \emph{one direction} more than the mean itself does. This is often used to increase sensitivity to ``worst-case'' events \citep{kashima2007a,takeda2008a}, but in special cases where losses are bounded below, sensitivity to tails on the downside can be used to realize an insensitivity to tails on the upside \citep{lee2020a}. This strong asymmetry is not specific to the preceding two risk function classes, but rather is inherent in much broader classes such as optimized certainty equivalent (OCE) risk \citep{bental1986a,bental2007a,lee2020a} and distributionally robust optimization (DRO) risk \citep{bental2013a,duchi2018b,duchi2019a,gotoh2018a}. Unsurprisingly, naive empirical estimators of these risks are particularly fragile under outliers coming from the ``sensitive direction,'' as is evidenced by the plethora of attempts in the literature to design robust modifications \citep{holland2022c,prashanth2020a,zhai2021a}. In general, however, loss distributions can display long tails in either direction over the learning process (see Figure \ref{fig:unbounded_below_data}), particularly when losses are unbounded below (e.g., negative rewards \citep{sutton2018RLIntro}, unhinged loss \citep{vanrooyen2016a}), and loss functions whose empirical mean has no minimum appear frequently (e.g., separable logistic regression \citep{albert1984a,rousseeuw2003a}). Since the tail behavior of stochastic losses and gradients is well-known to play a critical role in the stability and robustness of learning systems \citep{simsekli2019a,zhang2020b}, the inability to control tail sensitivity in both directions represents a genuine limitation to machine learning methodology.

A natural alternative class of risk functions that gives us control over tail sensitivity in \emph{both} directions is that of the ``M-location'' of the loss distribution, namely any value in
\begin{align}\label{eqn:mloc_general}
\argmin_{\theta \in \RR} \exx_{\ddist} \rho\left(\loss(h)-\theta\right) \subset \RR,
\end{align}
where $\rho:\RR \to \RR_{+}$ is assumed to be such that this set of minimizers is non-empty. Here various special choices of $\rho$ let us recover well-known locations, such as the mean (with $\rho(\cdot)=(\cdot)^{2}$), median ($\rho(\cdot)=\abs{\cdot}$), arbitrary quantiles (via ``pinball'' function \citep{takeuchi2006a}), and even further beyond to ``expectiles'' (using curved variants of the pinball function \citep{gneiting2011a}). The obvious limitation here is that while computing (\ref{eqn:mloc_general}) using empirical estimates is easy, minimization as a function of $h$ is in general a difficult bi-level programming problem. As an alternative approach, in this paper we study the potential benefits and tradeoffs that arise in using performance criteria of the form
\begin{align}\label{eqn:tloc_general}
\min_{\theta \in \RR} \left[ \eta\theta + \exx_{\ddist} \rho\left(\loss(h)-\theta\right) \right]
\end{align}
where $\eta \in \RR$. By making a sacrifice of fidelity to the M-location (\ref{eqn:mloc_general}), we see that the criterion in (\ref{eqn:tloc_general}) suggests a congenial objective function (joint in $(h,\theta)$). Intuitively, one minimizes the sum of generalized ``location'' and ``dispersion'' properties, and the nature of this dispersion impacts the fidelity of the location term to the original M-location induced by $\rho$. These two locations align perfectly in the special case where we set $\rho(\cdot)=(\cdot)^{2}/2$ and $\eta=1$, since (\ref{eqn:tloc_general}) is equivalent to the mean-variance objective $\exx_{\ddist}\loss(h) + \vaa_{\ddist}\loss(h)/2$, but more generally, it is clear that allowing for more diverse choices of $\rho$ gives us new freedom in terms of tail control with respect to both location and dispersion. We consider a concrete yet flexible class of risk functions that generalizes beyond (\ref{eqn:tloc_general}), allows for easy implementation, and is analytically tractable from the standpoint of providing formal learning guarantees. Our main contributions are as follows:
\begin{itemize}
\item A new class of ``threshold risks'' (T-risks; defined in \S{\ref{sec:trisk}}) that provide a tractable alternative to M-locations, and a bi-directional complement to OCE/DRO risks (reviewed in \S{\ref{sec:background_risk_review}}).
\item A stochastic learning algorithm for T-risks that enjoys high-probability guarantees of convergence to a stationary point under heavy-tailed losses/gradients, without manual clipping (details in \S{\ref{sec:learning_theory}}, Theorem \ref{thm:sgd_convergence}).
\item Strong empirical evidence of the flexibility and utility inherent in T-risk learners. In particular: robustness to unbalanced noisy class labels without regularization (Figure \ref{fig:unbounded_below_unhinged}), sharp control over sensitivity to outliers in regression with convex base losses (Figure \ref{fig:tail_control}), and smooth interpolation between mean and mean-variance minimizers on clean, normalized benchmark classification datasets (Figures \ref{fig:real_emnist_balanced}, \ref{fig:real_covtype} and \ref{fig:real_adult}--\ref{fig:real_protein}).
\end{itemize}
The overall flow of the paper is as follows. Background information on notation and related literature is given in \S{\ref{sec:background}}, and we introduce the new risk class of interest in \S{\ref{sec:trisk}}. Formal aspects of the learning problem using these risks are treated in \S{\ref{sec:learning_theory}}, and empirical findings are explained and discussed in \S{\ref{sec:applications}}. All formal proofs and supplementary results are organized in \S{\ref{sec:appendix_summary}}--\S{\ref{sec:additional_facts}} of the appendix, and code for reproducing all the results in this paper is provided in an online repository.\footnote{\url{https://github.com/feedbackward/bdd}}

\section{Background}\label{sec:background}

\subsection{Notation}\label{sec:background_notation}

\paragraph{Random quantities}
To start, let us clarify the nature of the random losses we consider. With the context of the underlying probability space $(\samplespace,\sigmafield,\ddist)$, we write $\loss(h) \defeq \loss(h;\cdot): \samplespace \to \RR$ to refer to a random variable (i.e., a $\sigmafield$-measurable function) on $\samplespace$, though we only use the form $\loss(h)$ in the body of this paper. When we talk about ``sampling'' losses or a ``random draw'' of the losses, this amounts to computing a realization $\loss(h;\omega) \in \RR$. We use standard notation for taking expectation, e.g., $\exx_{\ddist}\loss(h) \defeq \int_{\Omega} \loss(h;\omega)\,\ddist(\dif \omega)$. These conventions extend to random quantities based on the losses (e.g., the gradient $\dv{\loss}(h)$ considered in \S{\ref{sec:learning_theory_gradients}}). Similarly, we will use $\prr$ as a general-purpose probability function, representing both $\ddist$ and product measures; when the source of randomness is not immediate from the context, it will be stated explicitly. We will use $\risk(\cdot)$ as a generic symbol for risk functions (often modified with subscripts), with the understanding that $\risk(\cdot)$ maps random losses $\loss(h)$ to real values. We will overload this notation, writing $\risk(\loss)$ when the role of $h$ is unimportant, and writing $\risk(h) \defeq \risk(\loss(h))$ when we want to emphasize the dependence on $h$. This convention will be applied to other quantities as well, such as writing $\dsp_{\rho}(h) \defeq \dsp_{\rho}(\loss(h)) \defeq \exx_{\ddist}\rho(\loss(h)-\theta)$ for the expected dispersion induced by $\rho$, first defined in (\ref{eqn:dsp_defn}).

\paragraph{Norms}
We will use $\Abs{\cdot}$ as a general-purpose notation for all norms that appear in this paper. That is, we do not use different notation to distinguish different norm spaces. The reason for this is that we will never consider two distinct norms on the same set; each norm is associated with a distinct set, and thus as long as it is clear which set a particular element belongs to, there should be no confusion. The only exception to this rule is the special case of $\RR$, in which we write $\abs{\cdot}$ for the absolute value, as is traditional. For a function $f:\RR \to \RR$ in one variable, we use $\dv{f}$ to denote the usual derivative. More general notions (e.g., Gateaux or Fr\'{e}chet differentials) only make an appearance in \S{\ref{sec:learning_theory}}, and the generality they afford us is not crucial to the main narrative, so the details can be easily skipped over if the reader is unfamiliar with such concepts. All the other undefined notation we use is essentially standard, and can be found in most introductory analysis textbooks.

\paragraph{Miscellaneous}
For a function $f:\RR \to \RR$ in one variable, we use $\dv{f}$ to denote the usual derivative. More general notions (e.g., Gateaux or Fr\'{e}chet differentials) only make an appearance in \S{\ref{sec:learning_theory}}, and the generality they afford us is not crucial to the main narrative, so the details can be easily skipped over if the reader is unfamiliar with such concepts. All the other undefined notation we use is essentially standard, and can be found in most introductory analysis textbooks. Particularly in formal proofs, we will frequently make use of the notation $\rho_{\sigma}(x) \defeq \rho(x/\sigma)$, where $\sigma > 0$ is a scaling parameter.

\subsection{Review of key risk functions}\label{sec:background_risk_review}

\paragraph{OCE-type risks}
As a computationally convenient way to interpolate between the mean and the extreme values of $\loss(h)$, the \term{tilted risk} \citep{li2021b,li2021a} is a natural choice, defined for $\gamma \neq 0$ as
\begin{align}\label{eqn:tilted_defn}
\risk_{\textup{tilt}}(h;\gamma) \defeq \frac{1}{\gamma}\log\left(\exx_{\ddist}\mathrm{e}^{\gamma\loss(h)}\right).
\end{align}
This is simply a re-scaling of the cumulant distribution function of $\loss(h)$, viewed as a function of $h$, where taking $\gamma \to \infty$ and $\gamma \to -\infty$ lets us approach the supremum and infimum of $\loss(h)$, respectively. Another important class of risk functions is based upon the \term{conditional value-at-risk} (CVaR) \citep{rockafellar2000a}, defined for $0 \leq \beta < 1$ as
\begin{align}
\risk_{\textup{CVaR}}(h;\beta) \defeq \exx_{\ddist}\left[ \loss(h) \cond \loss(h) \geq \quant_{\beta}(h)\right].
\end{align}
This is the expected loss at $h$, conditioned on the event that the loss exceeds the \term{$\beta$-quantile} of $\loss(h)$, denoted here by $\quant_{\beta}(h) \defeq \inf\left\{ x \in \RR: \prr\{\loss(h) \leq x\} \geq \beta \right\}$. Both of these risk functions can be re-written in a form similar to that of (\ref{eqn:tloc_general}), namely
\begin{align}\label{eqn:oce_generalized}
h \mapsto \inf_{\theta \in \RR} \left[ \theta + \exx_{\ddist}\phi(\loss(h)-\theta) \right]
\end{align}
where $\phi(x) = (\mathrm{e}^{\gamma x}-1)/\gamma$ yields $\risk_{\textup{tilt}}$ (basic calculus), and $\phi(x) = \max\{0,x\}/(1-\beta)$ yields $\risk_{\textup{CVaR}}$ (see \citet{rockafellar2000a,rockafellar2002a}). When $\phi:\RR \to \RR$ is restricted to be a non-decreasing, closed, convex function which satisfies both $\phi(0)=0$ and $1 \in \partial\phi(0)$, the mapping given in (\ref{eqn:oce_generalized}) is called an \term{optimized certainty equivalent (OCE) risk} \citep{bental1986a,bental2007a,lee2020a}. The class of OCE risks strictly generalizes the expected value (noting $\phi(x)=x$ is valid), and includes $\risk_{\textup{tilt}}(\cdot;\gamma)$ when $\gamma > 0$, as well as $\risk_{\textup{CVaR}}$. The mean-variance is sometimes stated to be an OCE risk \citep[Table 1]{lee2020a}, but this fails to hold when losses are unbounded above and below.

\paragraph{DRO-type risks}
Another important class of risk functions is that of \term{robustly regularized} risks which are designed to ensure that risk minimizers are robust to a certain degree of divergence in the underlying data model. Making this concrete, it is typical to assume the random losses are the outputs of a loss function $\ell$ depending on the candidate $h$ and some random data $\rdv{Z}$, i.e., $\loss(h) = \ell(h;\rdv{Z})$, with $\rdv{Z} \sim \ddist$ as our reference model. To measure divergence from this reference model, it is convenient to use the Cressie-Read family of divergence functions \citep{duchi2018b,zhai2021a}, namely for any $c > 1$ and assuming $\nu \ll \ddist$ (absolutely continuity) holds, functions of the form
\begin{align}
\Div_{c}(\nu;\ddist) \defeq \exx_{\ddist}f_{c}\left(\frac{\dif\nu}{\dif\ddist}\right), \text{ where } f_{c}(x) \defeq \frac{x^{c}-cx + c - 1}{c(c-1)}
\end{align}
and $\dif\nu/\dif\ddist$ is the Radon-Nikodym density of $\nu$ with respect to $\ddist$.\footnote{For background on absolute continuity and density functions, see \citet[\S{2.2}]{ash2000a}.} The resulting robustly regularized risk, called the \term{DRO risk}, is defined as
\begin{align}\label{eqn:dro_defn}
\risk_{\DRO}(h) \defeq \sup\left\{ \exx\loss(h): \loss \in \LL \right\}
\end{align}
where the constrained set of random losses $\LL$, determined by $c > 1$ and $a > 0$, is defined as
\begin{align}\label{eqn:dro_conset}
\LL \defeq \left\{ \loss(\cdot) = \ell(\cdot;\rdv{Z}): \rdv{Z} \sim \nu \text{ and } \Div_{c}(\nu;\ddist) \leq a \right\}.
\end{align}
For this particular family of divergences, the risk can be characterized as the optimal value of a simple optimization problem \citep{duchi2018b}, namely we have that
\begin{align}\label{eqn:dro_convenient}
\risk_{\DRO}(h) = \inf_{\theta \in \RR}\left[ \theta + \left(1+c(c-1)a\right)^{1/c}\left(\exx_{\ddist}\left(\loss(h)-\theta\right)_{+}^{c_{\ast}}\right)^{1/c_{\ast}} \right]
\end{align}
where $c_{\ast} \defeq c / (c-1)$. While strictly speaking this is not an OCE risk, note that if we set $\phi(x) = (1+c(c-1)a)^{c_{\ast}/c}(x)_{+}^{c_{\ast}}$, then the DRO risk can be written as
\begin{align}\label{eqn:dro_like_tloc}
\risk_{\DRO}(h) = \inf_{\theta \in \RR}\left[ \theta + [\exx_{\ddist}\phi(\loss(h)-\theta)]^{1/c_{\ast}} \right],
\end{align}
giving us an expression of this risk as the sum of a threshold and an asymmetric dispersion. When we set $c=2$, this yields the well-known special case of \term{$\chi^{2}$-DRO risk} \citep{hashimoto2018a,zhai2021a}. In addition to the one-directional nature of the dispersion term in these risks, all of these risks are at least as sensitive to loss tails (on the upside) as the classical expected loss $\exx_{\ddist}\loss(h)$ is; this holds for CVaR (with $\beta > 0$), tilted risk (with $\gamma > 0$), and even robust variants of DRO risk \citep{zhai2021a}.

\paragraph{Key differences}
While it is clear that the form of the preceding risk classes given in (\ref{eqn:oce_generalized}) and (\ref{eqn:dro_like_tloc}) based on various choices of $\phi(\cdot)$ is the same as our $\rho(\cdot)$-based risk of interest in (\ref{eqn:tloc_general}), they are fundamentally different in that \emph{none} of the choices of $\phi(\cdot)$ induce a meaningful M-location; since all these $\phi(\cdot)$ are monotonic on $\RR$, both minimization and maximization of $\theta \mapsto \exx_{\ddist}\phi(\loss(h)-\theta)$ is trivially accomplished by taking $\abs{\theta} \to \infty$. In stark contrast, $\rho(\cdot)$ is assumed to be such that the solution set in (\ref{eqn:mloc_general}) is a subset of the real line. We will introduce a concrete and flexible class from which $\rho$ will be taken in \S{\ref{sec:trisk}}, and in Figure \ref{fig:dispersion_compare} give a side-by-side comparison with the $\phi$ functions discussed in the preceding paragraphs.

\subsection{Closely related work}

This work falls into the broad context of machine learning driven by novel risk functions \citep{holland2021survey}. Of all the papers cited above, the works of \citet{lee2020a} on OCE risks, and \citet{li2021a,li2021b} on tilted risks are of a similar nature to our paper, with the obvious difference being that the class of risks is fundamentally different, as described in the preceding paragraphs. Indeed, many of our empirical tests involve direct comparison with the risk classes studied in these works (e.g., Figures \ref{fig:fliptest_overshape_lognormal}, \ref{fig:unbounded_below_unhinged}, and \ref{fig:outliers_phones_loss_tails}), and so they provide critical context for our work. Previous work by \citet{holland2022c} studies a rudimentary special case of what we call ``minimal T-risk'' here; the focus in that work was on obtaining learning guarantees (in expectation) when the risk is potentially non-convex and non-smooth in $h$, but with convex $\rho$, and no comparison was made with OCE/DRO risk classes. We build upon these results here, considering a broad class of dispersions $\rho$ which are differentiable but need not be convex (see (\ref{eqn:barron})); we how such risk classes can readily admit high-probability learning guarantees for stochastic gradient-based algorithms (Theorem \ref{thm:sgd_convergence}), provide bounds on the average loss incurred by empirical risk minimizers using our risk (Proposition \ref{prop:risk_relations}), and make detailed empirical comparisons with each of the key existing risk classes.

\section{Threshold risk}\label{sec:trisk}

To ground ourselves conceptually, let us refer to $\loss(h)$ as the \term{base loss} incurred by $h$. The exact nature of $h$ is left completely abstract for the moment, as all that matters is the probability distribution of this base loss. By selecting an arbitrary \term{threshold} $\theta \in \RR$, we define a broad class of properties as
\begin{align}\label{eqn:trisk_general}
\risk_{\rho}(h;\theta,\eta) \defeq \eta\theta + \exx_{\ddist}\rho\left(\loss(h)-\theta\right).
\end{align}
Here $\eta \in \RR$ is a weighting parameter allowed to be negative, and as a bare minimum, $\rho$ is assumed to be such that the resulting M-location(s) are well-defined in the sense that the inclusion in (\ref{eqn:mloc_general}) holds. We call $\rho\left(\loss(h)-\theta\right)$ the (random) \term{dispersion} of the base loss, taken with respect to threshold $\theta$, and we refer to $\risk_{\rho}(h;\theta,\eta)$ in (\ref{eqn:trisk_general}) as the \term{threshold risk} (or simply \term{T-risk}) under $\rho$.

\subsection{Minimal T-risk and M-location}

Arguably the most intuitive special case of T-risk is the \term{minimal T-risk}, where minimization is with respect to the threshold $\theta \in \RR$. Let us denote this risk and the \term{optimal threshold} set as
\begin{align}\label{eqn:trisk_minimal}
\underline{\risk}_{\rho}(h;\eta) \defeq \inf_{\theta \in \RR} \risk_{\rho}(h;\theta,\eta), \qquad \mt_{\rho}(h;\eta) \defeq \argmin_{\theta \in \RR} \risk_{\rho}(h;\theta,\eta).
\end{align}
Clearly, if $\rho$ is bounded above or grows too slowly, we will have $\underline{\risk}_{\rho}(h;\eta) = -\infty$ and no real-valued minimizers, i.e., $\mt_{\rho}(h;\eta) = \emptyset$. Letting $\mloc_{\rho}(h)$ denote the M-locations in (\ref{eqn:mloc_general}), for $\eta \neq 0$ we have
\begin{align}
\mt_{\rho}(h;\eta) \neq \emptyset \implies \mloc_{\rho}(h) \neq \emptyset,
\end{align}
although the converse does not hold in general.\footnote{For example, consider choices of $\rho$ that are ``re-descending'' \citep{huber2009a}.} When $\eta = 0$, these two solution sets align, i.e., we have $\mt_{\rho}(h;0) = \mloc_{\rho}(h)$. More generally, depending on the sign of $\eta$, the optimal thresholds can be either larger or smaller than the corresponding M-locations. More precisely, for any $\theta^{\prime} \in \mloc_{\rho}(h)$, as long as $\mt_{\rho}(h;\eta)$ is non-empty, there exists $\theta \in \mt_{\rho}(h;\eta)$ such that $\theta \sign(\eta) \leq \theta^{\prime}\sign(\eta)$.

\paragraph{Special case minimized by quantiles}
The form given in (\ref{eqn:trisk_general}) is very general, but it can be understood as a straightforward generalization of the convex objective used to characterize quantiles. More precisely, taking $\beta \in (0,1)$ and denoting the \term{$\beta$-quantile} of the base loss using $\quant_{\beta}(h) \defeq \inf\left\{ x \in \RR: \prr\{\loss(h) \leq x\} \geq \beta \right\}$, it is well-known that in the special case of $\rho(\cdot)=\abs{\cdot}$, we have
\begin{align}\label{eqn:locations_quantile}
\quant_{\beta}(h) \in \mt_{\rho}(h;\theta,1-2\beta)
\end{align}
for any choice of $0 < \beta < 1$, as long as $\exx_{\ddist}\abs{\loss(h)}$ is finite \citep{koltchinskii1997a}. The T-risk in (\ref{eqn:trisk_general}) simply allows for a more flexible choice of $\rho$, and thus generalizes the dispersion term in this objective function.

\subsection{T-risk with scaled Barron dispersion}\label{sec:trisk_barron}

\begin{figure}[t]
\centering
\includegraphics[width=1.0\textwidth]{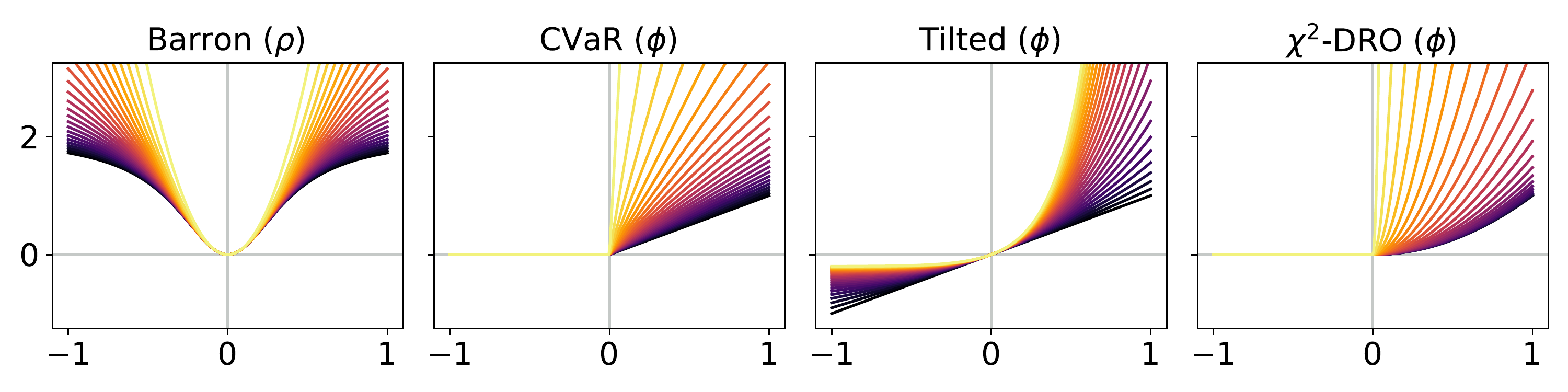}\\
\hfill\includegraphics[width=0.25\textwidth]{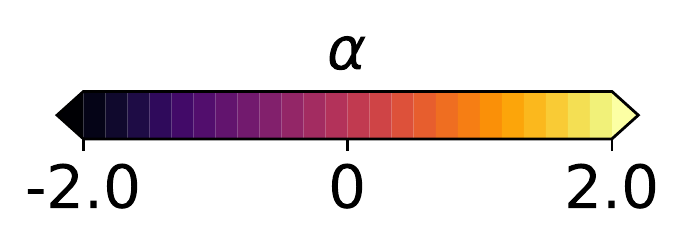}\includegraphics[width=0.25\textwidth]{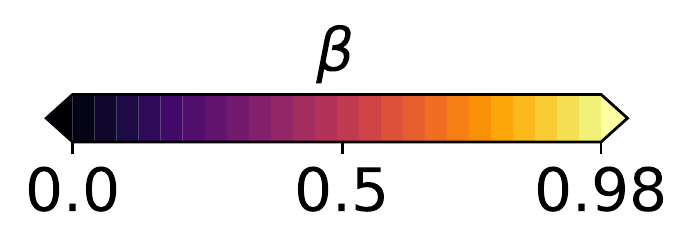}\includegraphics[width=0.25\textwidth]{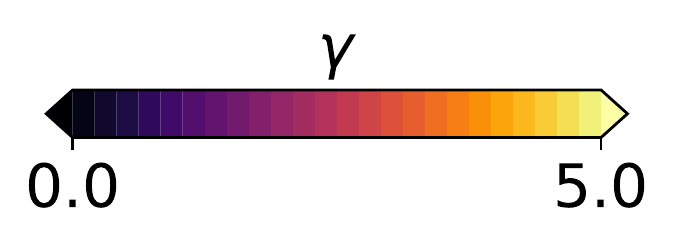}\includegraphics[width=0.25\textwidth]{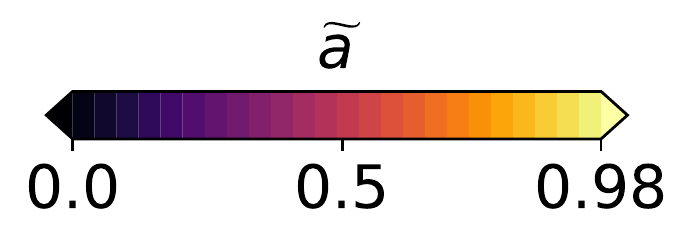}
\caption{Left-most plot: the graph of $x \mapsto \rho(x/\sigma;\alpha)$ from \S{\ref{sec:trisk_barron}} for varying choices of $\alpha$, with $\sigma$ fixed to $\sigma = 0.2$ for visual ease. Middle two plots: graphs of $\phi(x)$ in (\ref{eqn:oce_generalized}) for CVaR and tilted risk, respectively over different choices of $\beta$ and $\gamma$. Right-most plot: graph of $\phi(x)$ in (\ref{eqn:dro_like_tloc}) for the $\chi^{2}$-DRO risk, where $a \geq 0$ is re-parametrized using $0 \leq \widetilde{a} < 1$ using the relation $a = ((1-\widetilde{a})^{-1}-1)^{2}/2$.}
\label{fig:dispersion_compare}
\end{figure}

In order to capture a range of sensitivities to loss tails in both directions, we would like to select $\rho$ from a class of functions that gives us sufficient control over scale, boundedness, and growth rates. As a concrete choice, we propose to set $\rho$ in (\ref{eqn:trisk_general}) as $\rho(x) = \rho(x/\sigma;\alpha)$, where $\sigma > 0$ is a scaling parameter, and $\rho(\cdot;\alpha)$ with shape $\alpha \in [-\infty,2]$ is a family of functions that ranges from bounded and logarithmic growth on the lower end to quadratic growth on the upper end, defined as:
\begin{align}\label{eqn:barron}
\rho(x;\alpha) \defeq
\begin{cases}
x^{2}/2, & \text{if } \alpha = 2\\
\log\left(1 + x^{2}/2\right), & \text{if } \alpha = 0\\
1 - \exp\left(-x^{2}/2\right), & \text{if } \alpha = -\infty\\
\frac{\abs{\alpha-2}}{\alpha}\left(\left(1 + \frac{x^{2}}{\abs{\alpha-2}}\right)^{\alpha/2} - 1\right), & \text{otherwise}.
\end{cases}
\end{align}
At a high level, $\rho(\cdot;\alpha)$ is approximately quadratic near zero for any choice of shape $\alpha$, but its growth as one deviates far from zero depends greatly on $\alpha$. We refer to (\ref{eqn:barron}) as the \term{Barron class} of functions for computing dispersion.\footnote{The reason for this naming is that \citet{barron2019a} recently studied this class in the context of designing \emph{loss functions} for computer vision applications. We remark that this differs considerably from our usage in computing the \emph{dispersion of random losses}, where the loss function underlying the base loss is left completely arbitrary.} Recalling the risks reviewed in \S{\ref{sec:background_risk_review}}, since $\rho(\cdot;\alpha)$ is flat at zero and symmetric about zero, the Barron class clearly takes us beyond the functions $\phi(\cdot)$ allowed by OCE risks (\ref{eqn:oce_generalized}) and used in typical DRO risk definitions (\ref{eqn:dro_like_tloc}); see Figure \ref{fig:dispersion_compare} for a visual comparison.

As mentioned in \S{\ref{sec:background_notation}}, we will often use the generic shorthand notation $\rho_{\sigma}(x) \defeq \rho(x/\sigma)$, and drop the dependence on $\alpha$ when clear from context. The shape parameter $\alpha$ gives us direct control over the conditions needed for a finite T-risk $\risk_{\rho}(h;\theta,\eta)$, as the following lemma shows.
\begin{lem}[Finiteness and shape]\label{lem:dsp_finite}
Let $\rho$ be from the Barron class (\ref{eqn:barron}). Then in order to ensure $\exx_{\ddist}\rho_{\sigma}(\loss(h)-\theta) < \infty$ holds for all $\theta \in \RR$, each of the following conditions (depending on the value of $\alpha$) is sufficient. For $0 < \alpha \leq 2$, let $\exx_{\ddist}\lvert \loss(h) \rvert^{\alpha} < \infty$. For $\alpha = 0$, let $\exx_{\ddist}\lvert \loss(h) \rvert^{c} < \infty$ for some $c > 0$. For $-\infty \leq \alpha < 0$, let $\loss(h)$ be $\sigmafield$-measurable. Furthermore, for the cases where $\alpha \neq 0$, the above conditions are also necessary.
\end{lem}
Assuming $\ddist$-integrability as in Lemma \ref{lem:dsp_finite}, the Barron class furnishes a non-empty set of M-locations $\mloc_{\rho}(h)$ for any choice of $\alpha$, and when restricted to $\alpha \geq 1$ with appropriate settings of $\eta$ and $\sigma$, the optimal threshold set $\mt_{\rho}(h;\eta)$ contains a single unique solution (see Lemma \ref{lem:optimal_threshold_existence}). For any valid choice of $\alpha$, the function $\rho(\cdot;\alpha)$ is twice continuously differentiable on $\RR$ (see \S{\ref{sec:proofs_barron_derivs}} for exact expressions). All the limits in $\alpha$ behave as we would expect: $\rho(x;\alpha) \to \rho(x;c)$ as $\alpha \to c$ for $c \in \{-\infty,0,2\}$ (see \S{\ref{sec:proofs_barronlimits}} for details). For $\alpha \geq 0$, the dispersion function is unbounded, with growth ranging from logarithmic to quadratic depending on the choice of $\alpha$. For $\alpha < 0$, the dispersion function is bounded. The mapping $x \mapsto \rho_{\sigma}(x;\alpha)$ is convex on $\RR$ for $\alpha \geq 1$, and for $\alpha < 1$ it is only convex between $\pm \sigma\sqrt{\abs{\alpha-2}/(1-\alpha)}$, and concave elsewhere (see Lemma \ref{lem:basic_barron}). The class $\risk_{\rho}(h;\theta,\eta)$ of T-risks (\ref{eqn:trisk_general}) under the scaled Barron dispersion $\rho_{\sigma}(x;\alpha)$ is the central focus of this paper.

\subsection{Sensitivity to outliers and tail direction}\label{sec:trisk_expressive}

Before we consider the learning problem, which typically involves the evaluation of many different loss distributions over the course of training, here we consider a fixed distribution, and numerically compare the T-risk (\ref{eqn:trisk_general}) and M-location (\ref{eqn:mloc_general}) induced by $\rho$ from the Barron class in \S{\ref{sec:trisk_barron}}, along with the key OCE and DRO risks discussed in \S{\ref{sec:background_risk_review}}.

\paragraph{Experiment setup}
We generate random values to simulate loss distributions, and evaluate how the values returned by each risk function change as we modify their respective parameters. Letting $\loss$ denote the random base loss we are simulating, we specify a parametric distribution for $\loss$, from which we take an independent sample $\{\loss_{1},\ldots,\loss_{m}\}$. In all cases, we center the true distribution such that $\exx_{\ddist}\loss = 0$. We use this common sample to compare the values returned by each of the aforementioned risks, as well as the optimal choice of threshold parameter $\theta$. To ensure that key trends are consistent across samples, we take a large sample size of $m=10^{4}$. For T-risk, we adjust $\alpha$ and $\eta$, and for M-location, just $\alpha$. In both cases, we leave $\sigma=0.5$ fixed. For CVaR, we modify the quantile level $0 < \beta < 1$. For tilted risk, we modify the parameter $\gamma \in \RR$. For $\chi^{2}$-DRO risk, we modify $0 < \widetilde{a} < 1$, having re-parameterized $a$ in (\ref{eqn:dro_conset}) by $a \defeq ((1-\widetilde{a})^{-1}-1)^{2}/2$, as is common practice \citep{zhai2021a}.

\paragraph{Representative results}
An illustrative example is given in Figure \ref{fig:fliptest_overshape_lognormal}, where we look at how each risk class behaves under a centered asymmetric distribution, before and after flipping it (i.e., under $\loss$ and \textminus$\loss$). Starting from the two left-most plots, we show $\underline{\risk}_{\rho}(\loss;\eta)$ (dashed curves) and $\mt_{\rho}(\loss;\eta)$ (solid curves) from (\ref{eqn:trisk_minimal}) as a function of $\alpha$, coloring the area between these graphs in gray. The first plot corresponds to $\eta = 1$, the second to $\eta=-1$. Similarly for the M-location we plot $\mloc_{\rho}(\loss)$ (solid) and $\mloc_{\rho}(\loss)+\exx_{\ddist}\rho(\loss-\mloc_{\rho}(\loss))$ (dashed). Analogous values are plotted for each of the other classes; note that for the tilted risk (\ref{eqn:tilted_defn}) with $\gamma > 0$, the optimal threshold and the risk value are in fact the same value (see \S{\ref{sec:proofs_tilted}}). The right-most plot is a histogram of the random sample $\{\loss_{1},\ldots,\loss_{m}\}$, here from a centered log-Normal distribution. All plots share a common vertical axis, and horizontal rules are drawn at the median (red, solid) and at the mean (gray, dotted; always zero due to centering). The critical point to emphasize here is how all the OCE and DRO risks here are highly asymmetric in terms of their tail sensitivity, in stark contrast with both the M-location and the T-risk. Turning tail sensitivity high enough in each of these classes (e.g., $\alpha > 1.5$, $\beta > 0.5$, $\gamma > 1.0$, $\widetilde{a} > 0.5$), note how flipping the distribution tails from the upside (top row) to the downside (bottom row) leads to a dramatic decrease in all risks but the T-risk and M-location. Finally, note how the T-risk thresholds $\mt_{\rho}(\loss;\eta)$ close in on the M-location $\mloc_{\rho}(\loss)$ from above/below depending on whether $\eta$ is negative/positive. Results for numerous distributions are available in our online repository (cf.~\S{\ref{sec:intro}}).

\begin{figure}[t]
\centering
\includegraphics[width=1.0\textwidth]{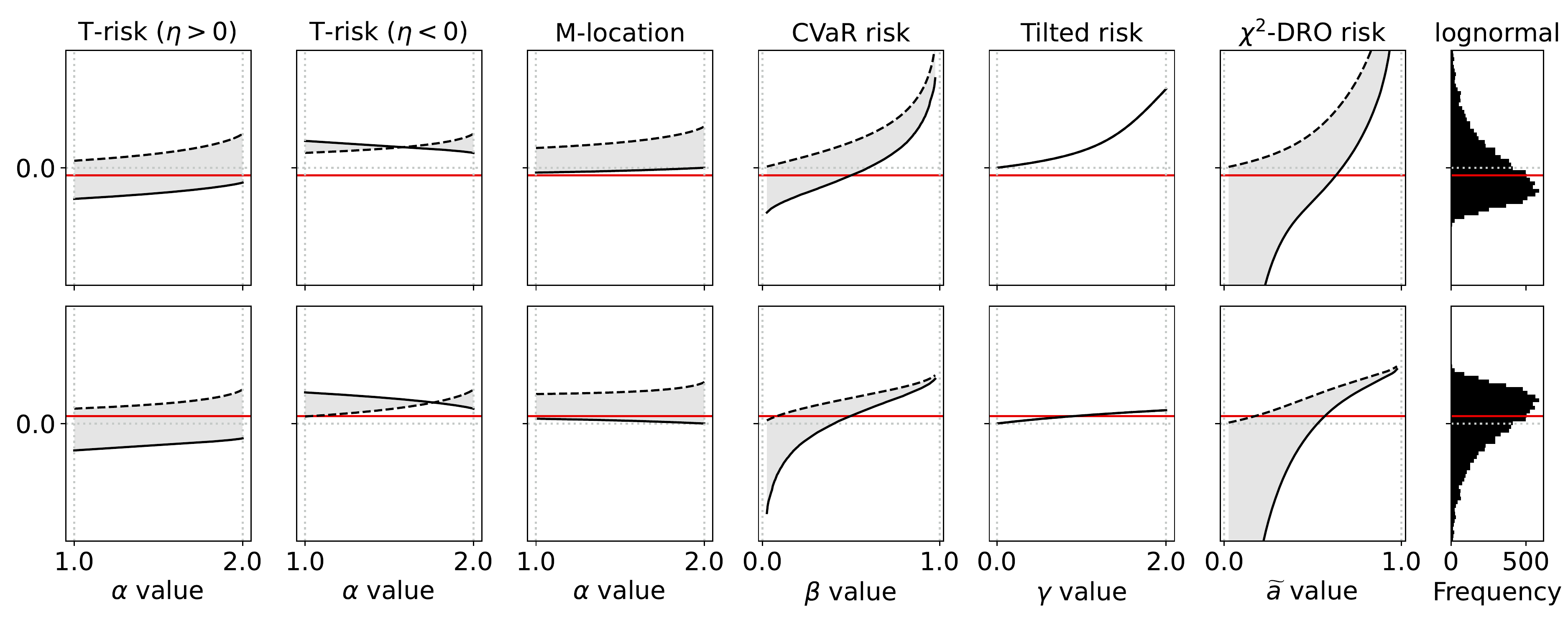}
\caption{Evaluation of risk class behavior after flipping an asymmetric distribution.}
\label{fig:fliptest_overshape_lognormal}
\end{figure}

\section{Learning algorithm analysis}\label{sec:learning_theory}

We now proceed to consider the learning problem, in which the goal is ultimately to select a candidate $h$ such that the distribution of $\loss(h)$ is ``optimal'' in the sense of achieving the smallest possible value of T-risk $\risk_{\rho}(h;\theta,\eta)$ in (\ref{eqn:trisk_general}), with $\rho$ taken from the Barron class (\ref{eqn:barron}), and $h$ taken from some set $\HH$. An obvious take-away of the integrability conditions (Lemma \ref{lem:dsp_finite}) is that even when the base loss is heavy-tailed in the sense of having infinite higher-order moments, we can always adjust the dispersion function $\rho(\cdot;\alpha)$ in such a way that transforming the base loss to obtain new feedback
\begin{align}\label{eqn:transformed_loss}
\loss_{\rho}(h;\theta,\eta) \defeq \eta\theta + \rho_{\sigma}(\loss(h)-\theta;\alpha)
\end{align}
gives us an unbiased estimator of the finite T-risk, i.e., $\exx_{\ddist}\loss_{\rho}(h;\theta,\eta) = \risk_{\rho}(h;\theta,\eta) \in \RR$. Intuitively, one expects that a similar property can be leveraged to control heavy-tailed stochastic gradients used in an iterative learning algorithm. We explore this point in detail in \S{\ref{sec:learning_theory_gradients}}. We will then complement this analysis by considering in \S{\ref{sec:learning_theory_mintrisk}} the basic properties of T-risk at the minimal threshold $\underline{\risk}_{\rho}(h;\eta)$ given in (\ref{eqn:trisk_minimal}), viewed from the perspectives of axiomatic risk design and empirical risk minimization.

\subsection{T-risk and stochastic gradients}\label{sec:learning_theory_gradients}

For the time being let us fix an arbitrary threshold $\theta \in \RR$, and assuming the gradient $\dv{\loss}(h)$ is $\ddist$-almost surely finite, denote the partial derivative of the transformed losses (\ref{eqn:transformed_loss}) with respect to $h$ by
\begin{align}\label{eqn:transformed_grad}
\partial_{h}\loss_{\rho}(h;\theta,\eta) \defeq \dv{\rho}_{\sigma}(\loss(h)-\theta) \dv{\loss}(h).
\end{align}
Writing $\partial_{h}\risk_{\rho}$ for the gradient of $h \mapsto \risk_{\rho}(h;\theta,\eta)$, an analogue of Lemma \ref{lem:dsp_finite} for gradients holds.
\begin{lem}[Unbiased gradients]\label{lem:unbiased_new}
Let $\UU$ be an open subset of any metric space such that $\HH \subseteq \UU$. Let the base loss map $h \mapsto \loss(h)$ be Fr\'{e}chet differentiable on $\UU$ ($\ddist$-almost surely), with gradient denoted by $\dv{\loss}(h)$ for each $h \in \UU$. Fixing any choice of $-\infty \leq \alpha \leq 2$, we have that
\begin{align}\label{eqn:unbiased_new}
\exx_{\ddist} \left[ \sup_{h \in \HH}\Abs{\dv{\rho}_{\sigma}(\loss(h))\dv{\loss}(h)} \right] < \infty \implies \exx_{\ddist}\left[\partial_{h}\loss_{\rho}\right] = \partial_{h}\risk_{\rho}
\end{align}
with the implied equality valid on all of $\HH \times \RR^{2}$.
\end{lem}
Consider a setting in which the gradients are heavy-tailed, i.e., where $\exx_{\ddist}\Abs{\dv{\loss}(h)}^{p} = \infty$ for $p > 2$. If the ultimate goal of learning is minimization of $h \mapsto \exx_{\ddist}\loss(h)$, then in order to obtain high-probability guarantees of finding a nearly-stationary point with rates matching the in-expectation case, one cannot naively use the raw gradients $\dv{\loss}(h)$, but must rather carry out a delicate truncation which accounts for the bias incurred \citep{cutkosky2021a,gorbunov2020aNIPS,gorbunov2021a,nazin2019a}. On the other hand, if the ultimate objective is $h \mapsto \risk_{\rho}(h;\theta,\eta)$, then using (\ref{eqn:transformed_grad}) there is zero bias by design (Lemma \ref{lem:unbiased_new}), and when we take the shape parameter of our dispersion function such that $\alpha \leq 0$, we have
\begin{align}\label{eqn:bounded_gradients}
\prr\left\{ \AbsLR{\partial_{h}\loss_{\rho}(h;\theta,\eta)} \leq \Gamma \right\} = 1
\end{align}
for an appropriate choice of $0 < \Gamma < \infty$ under standard loss functions such as quadratic and logistic losses, even when the random losses and gradients are heavy-tailed (see Corollary \ref{cor:sgd_convergence}).

To see how this plays out for the analysis of learning algorithms, let us consider plugging the raw stochastic gradients $\partial_{h}\loss_{\rho}$ into a simple update procedure. Given an independent sequence of random losses $(\loss_{1},\loss_{2},\ldots)$, let us denote by $(\loss_{\rho,1},\loss_{\rho,2},\ldots)$ the transformed losses computed via (\ref{eqn:transformed_loss}), and for a sequence $(h_{1},h_{2},\ldots)$ let $G_{t} \defeq \partial_{h}\loss_{\rho,t}(h_{t};\theta,\eta)$ denote the resulting stochastic gradients for any integer $t \geq 1$. Fixing $\theta \in \RR$ and letting $h_{1} \in \HH$ denote an arbitrary initial value, we consider a particular sequence generated using the following update rule:
\begin{align}\label{eqn:update_new_1}
h_{t+1} = h_{t} - a_{t}\widetilde{M}_{t},
\end{align}
where $a_{t}$ is a non-negative step-size we control, and the update direction satisfies
\begin{align}\label{eqn:update_new_2}
\Abs{\widetilde{M}_{t}} = 1, \langle \widetilde{M}_{t}, M_{t} \rangle = \Abs{M_{t}}, \text{ where } M_{t} \defeq b M_{t-1} + (1-b) G_{t}
\end{align}
with $0 < b < 1$ also being a controllable parameter. This is an unconstrained, normalized stochastic gradient descent routine using momentum; it modifies the procedure of \citet{cutkosky2021a} in that we do not truncate $G_{t}$. Note that if $\HH$ is a Banach space and $\HH^{\ast}$ its dual, in general we have that $G_{t}$ and $M_{t}$ are elements of $\HH^{\ast}$. When $\HH$ is reflexive (e.g., any Hilbert space), it is always possible to construct $\widetilde{M}_{t}$ from $M_{t}$.\footnote{Given any $h^{\ast} \in \HH^{\ast}$, we can always find $h \in \HH$ such that $\langle h^{\ast}, h \rangle = \Abs{h}\Abs{h^{\ast}}$ \citep[\S{5.6}]{luenberger1969Book}.} The following theorem shows how the gradient norms incurred by this algorithm can be bounded with high probability.
\begin{thm}[Stationary points with high probability]\label{thm:sgd_convergence}
Let $\HH$ be a reflexive Banach space, with a Fr\'{e}chet differentiable norm satisfying $\Abs{h_{1}+h_{2}}^{2} \leq \Abs{h_{1}}^{2} + \langle \nabla\Abs{h_{1}}^{2}, h_{2} \rangle + \Abs{h_{2}}^{2}$ for any $h_{1},h_{2} \in \HH$. In addition, assume the losses are such that (\ref{eqn:bounded_gradients}) holds on $\HH$, $\exx_{\ddist}\Abs{\dv{\loss}(h_{1})-\dv{\loss}(h_{2})} \leq \smooth_{1}\Abs{h_{1}-h_{2}}$ for any $h_{1},h_{2} \in \HH$, and $\exx_{\ddist}\Abs{\dv{\loss}}_{\HH}^{2} < \infty$, where $\Abs{\dv{\loss}}_{\HH} \defeq \sup_{h \in \HH}\Abs{\dv{\loss}(h)}$. Run the learning algorithm in (\ref{eqn:update_new_1})--(\ref{eqn:update_new_2}) for $T$ iterations, with $b = 1-1/\sqrt{T}$ and $a_{t} = (1/T)^{3/4}$ for all steps $t = 1, 2, \ldots, T$, assuming each $h_{t} \in \HH$. Taking any $0 < \delta < 1$, it then follows that
\begin{align*}
\frac{1}{T} \sum_{t=1}^{T}\Abs{\partial_{h}\risk_{\rho}(h_{t};\theta,\eta)} \leq  \frac{c_{1}}{T^{1/4}} + \frac{c_{2}}{\sqrt{T}} + \frac{\smooth}{2T^{3/4}}
\end{align*}
with probability no less than $1-\delta$, using coefficients defined as
\begin{align*}
c_{1} & \defeq \risk_{\rho}(h_{1};\theta,\eta)-\risk_{\rho}(h_{T+1};\theta,\eta) + 16\Gamma\sqrt{\log(3T\delta^{-1})} + 2\smooth\\
c_{2} & \defeq 20\Gamma\log(3T\delta^{-1}) + 2\Gamma\\
\smooth & \defeq \frac{\smooth_{3} + \smooth_{4}}{\sigma} + \frac{\smooth_{2}}{\sigma^{2}}\exx_{\ddist}\Abs{\dv{\loss}}_{\HH}, \quad \smooth_{2} \defeq \Abs{\ddv{\rho}}_{\infty}, \quad \smooth_{3} \defeq \left(\frac{2\smooth_{2}}{\sigma}\right)\exx_{\ddist}\Abs{\dv{\loss}}_{\HH}^{2}, \quad \smooth_{4} \defeq \smooth_{1}\Abs{\dv{\rho}}_{\infty}
\end{align*}
for any choice of $\sigma > 0$.
\end{thm}
These high-probability $\mathcal{O}(T^{-1/4})$ rates match standard guarantees in the stochastic optimization literature under non-convex objectives in expectation \citep{davis2019b,ghadimi2016a}. The main take-away of Theorem \ref{thm:sgd_convergence} is that even if the random losses and gradients are unbounded and heavy-tailed, as long as the dispersion function $\rho$ is chosen to modulate extreme values (such that (\ref{eqn:bounded_gradients}) holds), we can obtain confidence intervals for the T-risk gradient norms incurred by stochastic gradient updates. Distribution control is implied by the risk design, and thus there is no additional need for truncation or bias control. The following corollary illustrates how the bounded gradient condition in Theorem \ref{thm:sgd_convergence} is satisfied under very weak assumptions on the data.
\begin{cor}\label{cor:sgd_convergence}
Assume the random losses $\loss(h)$ are driven by random data $(\rdv{X},\rdv{Y})$, where $\rdv{X}$ takes values in a Banach space $\XX$, and $\HH \subset \XX^{\ast}$ has finite diameter. Consider the following losses:
\begin{itemize}
\item[\namedlabel{ex:linreg}{\textup{E1}}.] \underline{Quadratic loss:} $\loss(h) = (h(\rdv{X})-\rdv{Y})^{2}/2$, with $h(\rdv{X}) = \langle h, \rdv{X} \rangle$ and $\rdv{Y} = \langle h^{\ast}, \rdv{X} \rangle + \varepsilon$, where $\varepsilon$ is zero-mean, has finite variance, and is independent of $\rdv{X}$.
\item[\namedlabel{ex:logreg}{\textup{E2}}.] \underline{Logistic loss:} $\loss(h) = \log(\sum_{j=1}^{k}\exp(\langle h_{j}, \rdv{X} \rangle)) - \sum_{j=1}^{k}\widetilde{\rdv{Y}}_{j}\langle h_{j}, \rdv{X} \rangle$, where we have $k \geq 2$ classes, $h=(h_{1},\ldots,h_{k})$ with each $h_{j} \in \HH$, and $(\widetilde{\rdv{Y}}_{1},\ldots,\widetilde{\rdv{Y}}_{k})$ is a one-hot representation of the class label $\rdv{Y}$ assigned to $\rdv{X}$.
\end{itemize}
If we set $\rho(\cdot) = \rho(\cdot;\alpha)$ with $\alpha \leq 0$, then under the examples \ref{ex:linreg}--\ref{ex:logreg}, as long as $\exx_{\ddist}\Abs{\rdv{X}}^{2} < \infty$, we have that the bounds assumed by Theorem \ref{thm:sgd_convergence}, including (\ref{eqn:bounded_gradients}), are satisfied on $\HH$.
\end{cor}
In practice, the threshold $\theta$ will not typically be fixed arbitrarily, but rather selected in a data-dependent fashion, potentially optimized alongside $h$; the impact of such algorithmic choices will be evaluated in our empirical tests in \S{\ref{sec:applications}}. In the following sub-section, we consider some key properties of the special case in which threshold $\theta$ is always taken to yield the smallest overall T-risk value.

\paragraph{Limitations}
An obvious limitation of our analysis is that we have assumed that each $h_{t} \in \HH$; this matches the setup of \citep{cutkosky2021a}, but the condition that $\exx_{\ddist}\Abs{\dv{\loss}}_{\HH}^{2} < \infty$ will sometimes require $\HH$ to have a finite diameter. Modifying the procedure to allow for projection of the iterates to $\HH$ is a point of technical interest, but is out of this paper's scope. Another limitation is that our current approach using smoothness properties (Lemma \ref{lem:smoothness_limited}) necessitates an assumption of gradients with finite second-order moments. In the special case where $\alpha=1$, arguments based on smoothness can be replaced with arguments based on weak convexity \citep{davis2019b,holland2022c}; but this fails for more general $\alpha$ since $\rho(\cdot;\alpha)$ is not convex for $\alpha < 1$, and not Lipschitz for $\alpha > 1$. Another potential option is to split the sample, leverage stronger guarantees available in expectation for gradient descent run on the subsets, and robustly choose the best candidate based on a validation set of data \citep{holland2021b}.

\subsection{T-risk with minimizing thresholds}\label{sec:learning_theory_mintrisk}

Recall the minimal T-risk $\underline{\risk}_{\rho}(h;\eta)$ and optimal thresholds $\mt_{\rho}(h;\eta)$ defined earlier in (\ref{eqn:trisk_minimal}), here restricted to scaled $\rho(\cdot) = \rho_{\sigma}(\cdot;\alpha)$ from the Barron class (\ref{eqn:barron}). This is arguably the most natural subset of T-risks, with a functional form aligned with OCE/DRO risks discussed in \S{\ref{sec:background_risk_review}}. For readability, we denote the dispersion of $\loss(h)$ measured about $\theta \in \RR$ using $\rho$ as
\begin{align}\label{eqn:dsp_defn}
\dsp_{\rho}(h;\theta) \defeq \exx_{\ddist}\rho\left(\loss(h)-\theta\right).
\end{align}
Here we will overload our notation, writing $\underline{\risk}_{\rho}(\loss;\eta)$, $\mt_{\rho}(\loss;\eta)$, and $\dsp_{\rho}(\loss;\theta)$ when we want to leave $h$ abstract and focus on random losses in a set $\loss \in \LL$. For this risk to be finite, we require $\alpha \geq 1$, and in addition need $\abs{\eta} < 1/\sigma$ for the special case of $\alpha = 1$; otherwise, the dispersion term grows too slowly and we have $\underline{\risk}_{\rho}(\loss;\eta) = -\infty$ (Lemma \ref{lem:optimal_threshold_existence}). When finite, the optimal threshold is unique, and overloading our notation once more we use $\mt_{\rho}(\loss;\eta)$ to denote it. The following lemma summarizes some basic properties of the minimal T-risk.
\begin{lem}\label{lem:axiom_check}
Let $\LL$ be such that for each $\loss \in \LL$ we have $\underline{\risk}_{\rho}(\loss;\eta) < \infty$, and let $\alpha$, $\sigma$, and $\eta$ be such that $\underline{\risk}_{\rho}(\loss;\eta) > -\infty$. Under these assumptions, the dispersion part of the minimal T-risk is translation invariant, i.e., for any $c \in \RR$, we have $\dsp_{\rho}(\loss+c;\mt_{\rho}(\loss+c;\eta)) = \dsp_{\rho}(\loss;\mt_{\rho}(\loss;\eta))$. This dispersion term is always non-negative, and if $\eta \neq 0$, then we have $\dsp_{\rho}(\loss;\mt_{\rho}(\loss;\eta)) > 0$ for any $\loss \in \LL$, even if $\loss$ is constant. If $\LL$ is convex, then so is the map $\loss \mapsto \underline{\risk}_{\rho}(\loss;\eta)$. The optimal threshold is translation equivariant in that we have $\mt_{\rho}(\loss+c;\eta) = \mt_{\rho}(\loss;\eta) + c$ for any choice of $c \in \RR$, and it is monotonic in that whenever $\loss_{1}, \loss_{2} \in \LL$ and $\loss_{1} \leq \loss_{2}$ almost surely, we have $\mt_{\rho}(\loss_{1};\eta) \leq \mt_{\rho}(\loss_{2};\eta)$.
\end{lem}
Let us briefly discuss the properties described in Lemma \ref{lem:axiom_check} with a bit more context. One of the best-known classes of risk functions is that of \term{coherent} risks \citep{artzner1999a}, typically characterized by properties of convexity, monotonicity, translation equivariance, and positive homogeneity \citep{ruszczynski2006a}. Our general notion of ``dispersion'' is often referred to as ``deviation'' in the risk literature, and the properties of translation invariance, sub-linearity (implying convexity), non-negativity, and definiteness (i.e., zero only for constants) allow one to establish links between deviations and coherent risks \citep{rockafellar2006a}. In general, the risk $\loss \mapsto \underline{\risk}_{\rho}(\loss;\eta)$ takes us outside of this traditional class, while still maintaining lucid connections as summarized in the preceding lemma.
\begin{rmk}[Risk quadrangle]
It should also be noted that our T-risks are not what would typically be called ``risks'' in the context of the ``expectation quadrangle'' in the framework developed by \citet{rockafellar2013a}. With their Example 7 as a clear reference for comparison, our function $\rho(\cdot)$ corresponds to their ``error integrand'' $e(\cdot)$, and the risk derived from their quadrangle would be
\begin{align*}
\loss \mapsto \exx_{\ddist}\loss + \exx_{\ddist}\rho\left(\loss - \mloc_{\rho}(\loss)\right) = \inf_{\theta \in \RR} \left[\theta + \exx_{\ddist}v(\loss-\theta)\right]
\end{align*}
where $v(x) \defeq \rho(x)+x$ and $\mloc_{\rho}(\loss)$ denotes the M-location (\ref{eqn:mloc_general}) under $\rho$. Under appropriate choice of $\rho$, this is an OCE-type risk, and evidently our optimal thresholds $\mt_{\rho}(\loss;\eta)$ with $\eta \neq 0$ do not make an appearance in risks of this form, highlighting the distinct nature of $\underline{\risk}_{\rho}(\loss;\eta)$.
\end{rmk}

Next, we briefly consider the question of how algorithms designed to minimize $\underline{\risk}_{\rho}(h;\eta)$ perform in terms of the classical risk, namely the expected loss $\risk(h) \defeq \exx_{\ddist}\loss(h)$. This is a big topic, but as an initial look, we consider how the expected loss incurred by minimizers of the empirical T-risk can be controlled given sufficiently good concentration of the M-estimators induced by $\rho$, and the empirical mean. Let us denote any empirical T-risk minimizer by
\begin{align}\label{eqn:trisk_empirical_minimizer}
\widehat{h}_{\rho} \in \argmin_{h \in \HH}\left[ \inf_{\theta \in \RR} \left( \eta\theta + \frac{1}{n}\sum_{i=1}^{n}\rho_{\sigma}\left(\loss_{i}(h)-\theta\right) \right) \right]
\end{align}
where for simplicity $(\loss_{1},\ldots,\loss_{n})$ is an iid sample of random losses.
\begin{prop}\label{prop:risk_relations}
Take any T-risk parameters $\alpha \geq 1$, $\eta > 0$, and $\sigma > 0$, denoting the resulting empirical (minimal) T-risk minimizer $\widehat{h}_{\rho}$ in (\ref{eqn:trisk_empirical_minimizer}). Let $\widehat{\risk}(h)$, $\widehat{\mloc}_{\rho}(h)$, $\widehat{\theta}_{\rho}(h;\eta)$, and $\widehat{\dsp}_{\rho}(h;\theta)$ denote the empirical analogues of $\risk(h)$, $\mloc_{\rho}(h)$, $\mt_{\rho}(h;\eta)$, and $\dsp_{\rho}(h;\theta)$ for each $h \in \HH$. Then, with $\Abs{\cdot}_{\HH}$ as in Theorem \ref{thm:sgd_convergence}, we have
\begin{align*}
\risk(\widehat{h}_{\rho}) \leq \risk(h^{\ast}) + \Abs{\widehat{\mloc}_{\rho}-\widehat{\theta}_{\rho}}_{\HH} + 2\Abs{\widehat{\mloc}_{\rho}-\risk}_{\HH} + \frac{1}{\eta}\widehat{\dsp}_{\rho}(h^{\ast};\widehat{\mloc}_{\rho}(h^{\ast}))+ 4\Abs{\risk-\widehat{\risk}}_{\HH}
\end{align*}
where we can freely choose $h^{\ast} \in \argmin_{h \in \HH} \risk(h)$ to be optimal in the expected loss.
\end{prop}
We have left the upper bound in Proposition \ref{prop:risk_relations} rather abstract to emphasize the key factors that can be used to control expected loss bounds for empirical T-risk minimizers and keep the overall narrative clear. Note that in the special case of $\alpha = 2$ one has $\widehat{\mloc}_{\rho}(h)-\widehat{\theta}_{\rho}(h;\eta) = \eta$, and more generally taking $\eta \to 0$ sends this difference to zero. There exists tension due to the $1/\eta$ coefficient on the dispersion term. All other terms, including the dispersion itself, are free of $\eta$. Note also that the remaining term $\Abs{\widehat{\mloc}_{\rho}-\risk}_{\HH}$ is the (uniform) difference between the M-estimator induced by $\rho$ and the classical risk; this can be modulated directly using the scale parameter $\sigma$, and sharp bounds for a broad classes of M-estimators have been established recently \citep{minsker2019a}. Our Proposition \ref{prop:risk_relations} is analogous to Theorem 7 of \citet{lee2020a}, with the key difference being that we study (minimal) T-risks instead of OCE risks, without assuming that losses are bounded (below or above).

\section{Learning applications}\label{sec:applications}

To complement the ``static'' empirical analysis in \S{\ref{sec:trisk_expressive}} and the formal insights for learning algorithms in \S{\ref{sec:learning_theory}}, here we empirically investigate how risk function design and data properties impact the behaviour of stochastic gradient-based learners.

\subsection{Classification with noisy and unbalanced labels}\label{sec:applications_2dclass}

\paragraph{Experiment setup}
As an initial example using simulated data, we design a binary classification problem by generating 500 labeled data points as shown in Figure \ref{fig:unbounded_below_data} (left-most plot). The majority class comprises 95\% of the sample, and we flip 5\% of the labels uniformly at random. Aside from the flipped labels, the data is linearly separable. We consider two candidate classifiers to initialize an iterative learning algorithm: one candidate that does well on the majority class (dotted purple), and one candidate that does well on the minority class (dashed pink). Due to class imbalance, the loss distributions incurred by these two candidates are highly asymmetric and differ in their long tail direction (Figure \ref{fig:unbounded_below_data}, remaining plots). We run empirical risk minimization for each risk class described in \S{\ref{sec:background_risk_review}} (joint in $(h,\theta)$), implemented by 15,000 iterations of full-batch gradient descent, using a step size of $0.01$, and three different choices of base loss function (logistic, hinge, unhinged). We run this procedure on the same data for a wide range of risk parameters $\alpha$ (T-risk), $\beta$ (CVaR), $\gamma$ (tilted), and $a$ ($\chi^{2}$-DRO), and choose a representative setting for each risk class as the one achieving the best final training classification error.

\paragraph{Representative results}
The full-batch (training) classification error and norm trajectories for these representatives initialized at each of the two candidates just mentioned (``mostly correct'' and ``mostly incorrect'') and trained using the unhinged loss as a base loss are shown in Figure \ref{fig:unbounded_below_unhinged}. Regardless of the direction of the initial loss distribution tails, we see that using the Barron-type T-risk has the flexibility to achieve both a stable and superior long-term error rate, while at the same time penalizing exceedingly overconfident correct examples (via bi-directionality), and thereby keeping the norm of the linear classifier small. An almost identical trend is also observed under the (binary) logistic loss, whereas classification error rates under the hinge loss tend to be less stable (see Figures \ref{fig:unbounded_below_logistic}--\ref{fig:unbounded_below_hinge} in \S{\ref{sec:supplementary_info_figures}}).

\begin{figure}[t]
\centering
\includegraphics[width=0.225\textwidth]{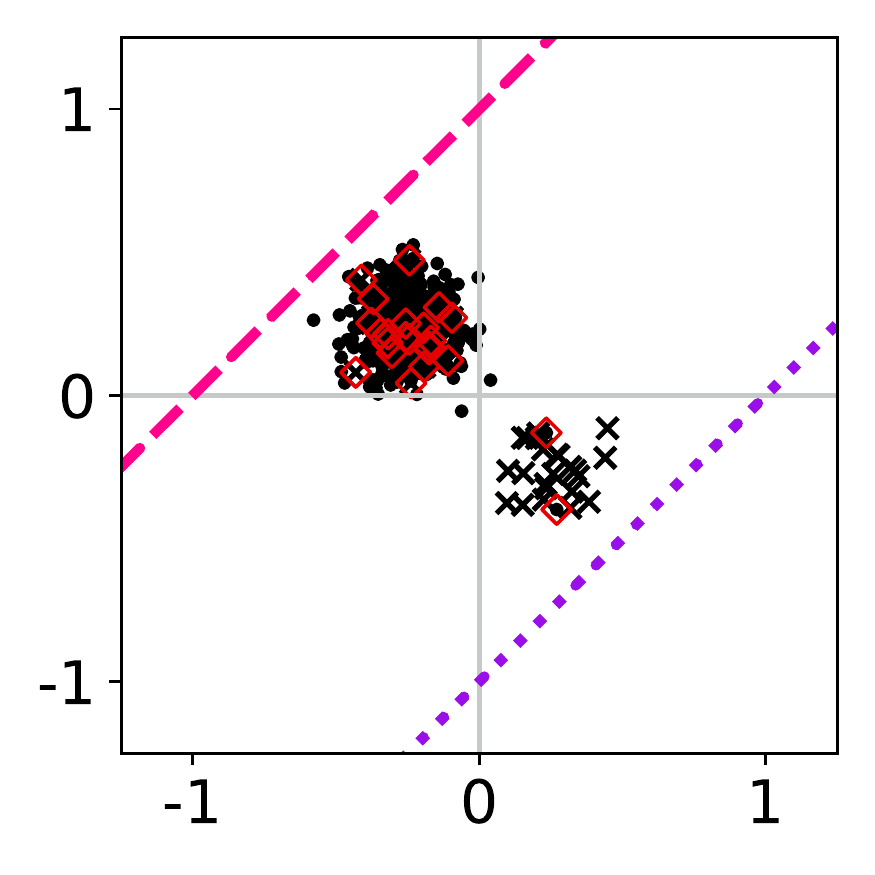}\includegraphics[width=0.775\textwidth]{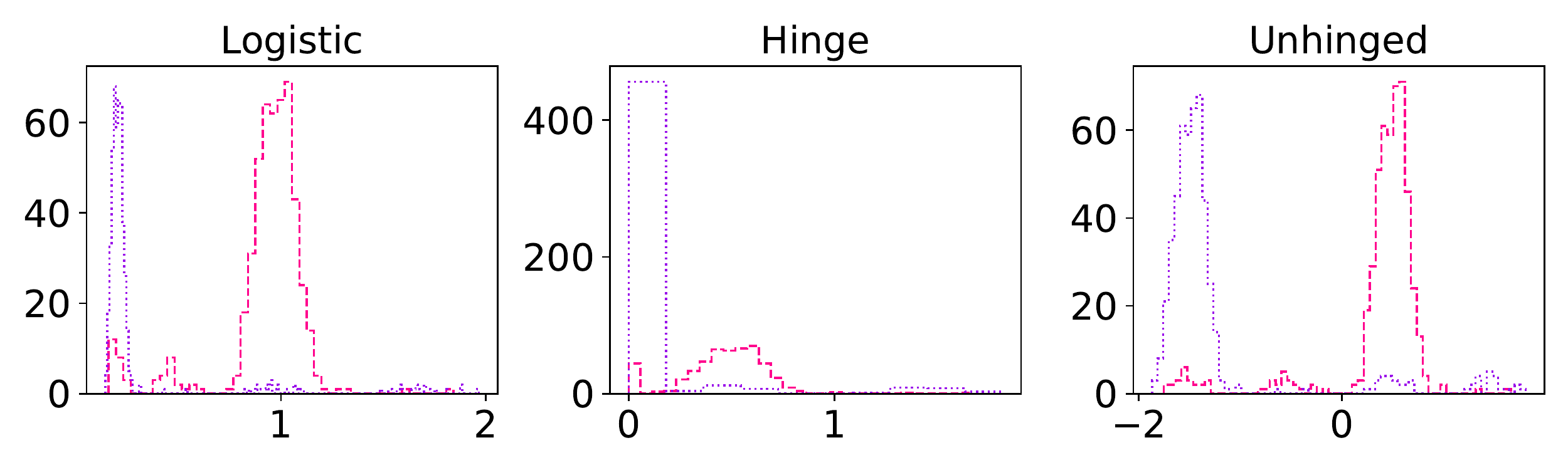}
\caption{Left-most plot: simulated binary classification data, with unbalanced classes and randomly flipped labels (indicated by red diamonds). Remaining plots: loss distributions for three types of loss functions, evaluated at two different candidates, corresponding to the dashed pink and dotted purple lines in the left-most plot.}
\label{fig:unbounded_below_data}
\end{figure}
\begin{figure}[t]
\centering
\includegraphics[width=0.5\textwidth]{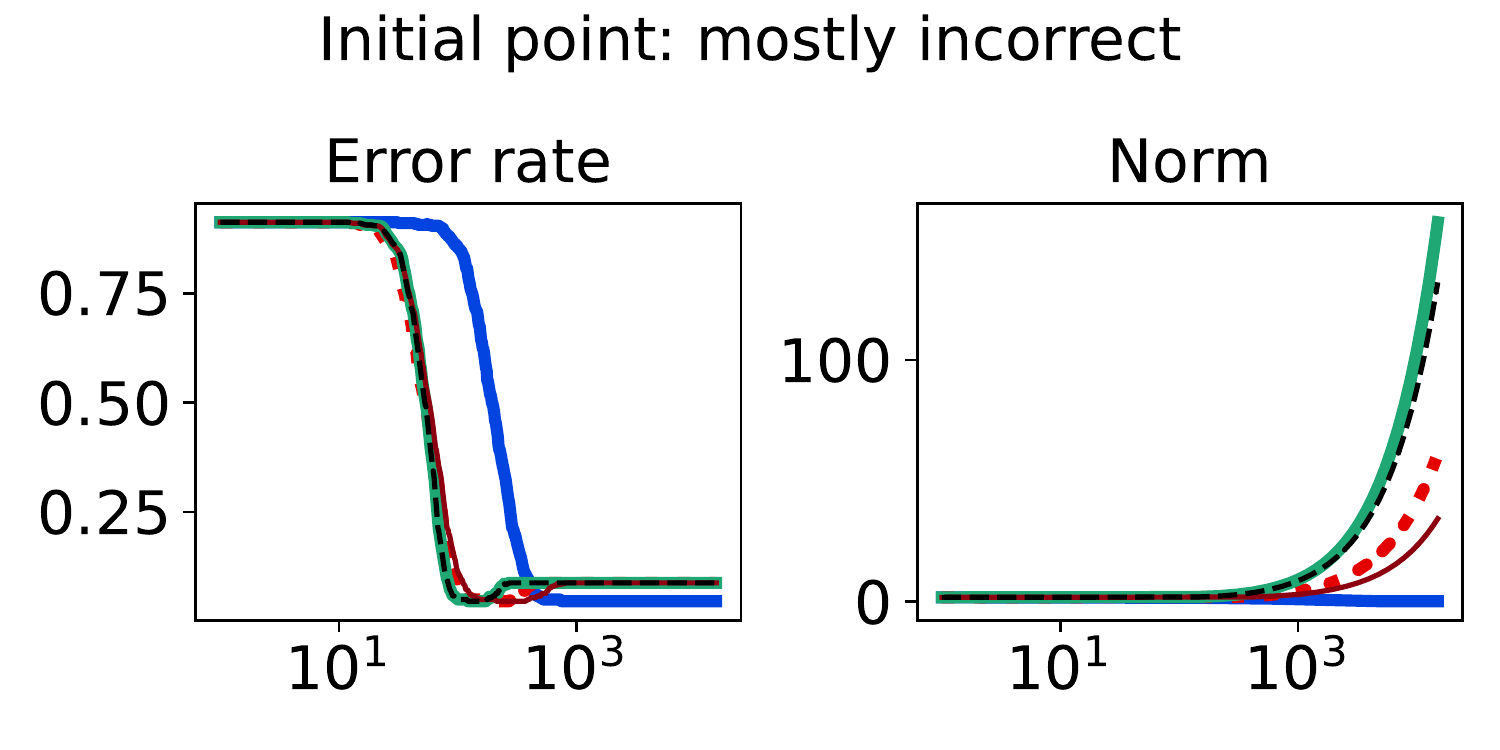}\includegraphics[width=0.5\textwidth]{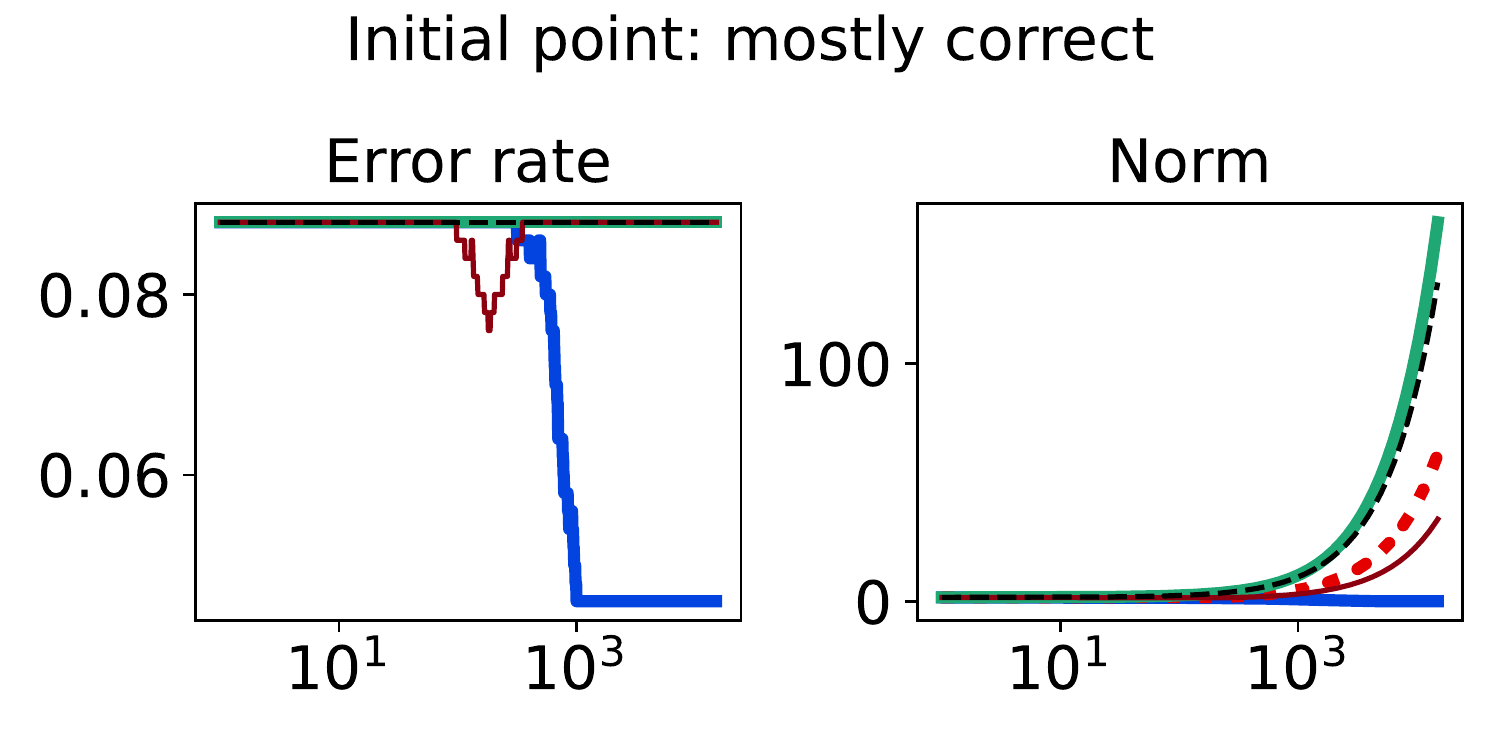}\\
\includegraphics[width=0.75\textwidth]{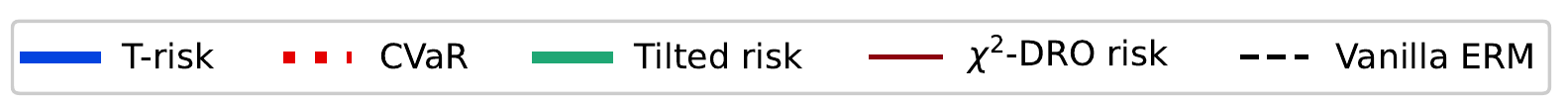}\hspace{0.1cm}
\caption{Classification error and norm trajectories (over iteration number, $\log_{10}$ scale) for gradient descent implementations of each risk class, using unhinged base loss and data given in Figure \ref{fig:unbounded_below_data}.}
\label{fig:unbounded_below_unhinged}
\end{figure}

\subsection{Control of outlier sensitivity}\label{sec:applications_1dlinreg}

\paragraph{Experiment setup}
For a crystal-clear example of real data that induces losses with heavy tails, we consider the famous Belgian phone call dataset \citep{rousseeuw1987RobReg,venables2002a} with input features normalized to the unit interval, and raw outputs used as-is. We conduct two regression tasks: the first uses the original data just stated, and the second has us modify such that it is an outlier on both the horizontal and vertical axes (we just multiply the original data point by 5); such points are said to have ``high leverage'' \citep[Ch.~7]{huber2009a}. For these two tasks, once again we run empirical risk minimization under each of the risk classes of interest, implemented using 15,000 iterations of full-batch gradient descent, with fixed step size 0.005. To illustrate the flexibility of the T-risk, here instead of minimizing (\ref{eqn:trisk_general}) jointly in $(h,\theta)$, we fix $\theta$ at the start of the learning process along with $\sigma$ and $\eta$, iteratively optimizing $h$ only (i.e., $\theta$ is not updated at any point). For simplicity, we set $\theta$ and $\sigma$ both to be the median of the losses incurred at initialization. All other risk classes are precisely as in the previous experiment.

\paragraph{Representative results}
The final regression lines obtained under each risk setting by running the learning procedure just described are plotted along with the data in Figure \ref{fig:tail_control}. Colors correspond to the individual risk function choices within each risk family, as denoted by color bars under each plot. The gray regression line denotes the common initial value used by all methods. Algorithm outputs using CVaR and $\chi^{2}$-DRO are always at least as sensitive to outliers as the ordinary least-squares (OLS) solution. While the tilted risk does let us interpolate between lower and upper quantiles, this transition is not smooth; even trying 20 values within the small window of $\gamma \in [-0.025,0.025]$, the algorithm outputs essentially jump between two extremes. This difference is particularly lucid in the bottom plots of Figure \ref{fig:tail_control}, where we have a high-leverage point. In contrast, with all other parameters of T-risk fixed, note that just tweaking $\alpha$ gives us a remarkable degree of flexibility to control the final output; while the base loss is fixed to the squared error, the regression lines range from those that ignore outliers to those that are very sensitive to outliers. In the high-leverage case, it is well-known that this cannot be achieved by simply changing to a different convex base loss (e.g., MAE instead of OLS), giving a concise illustration of the flexibility inherent in the T-risk class. Algorithm behavior can be sensitive to initial value; see \S{\ref{sec:supplementary_info_applications}} and Figure \ref{fig:tail_control_all_updown} in the appendix for an example where the naive median-based procedure described here causes gradient-based algorithms to stall out. 

\begin{figure}[t]
\centering
\includegraphics[width=1.0\textwidth]{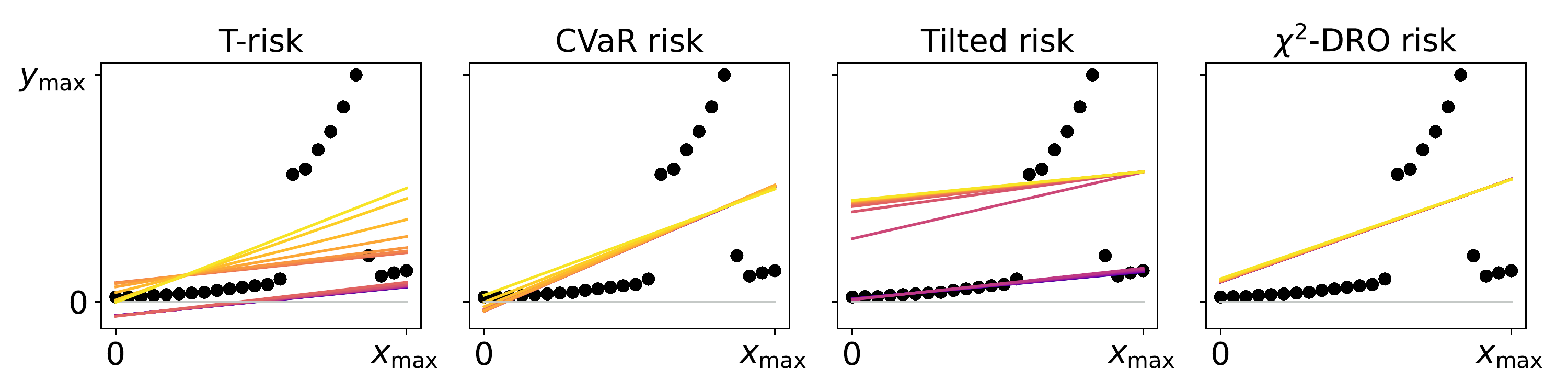}\\
\includegraphics[width=1.0\textwidth]{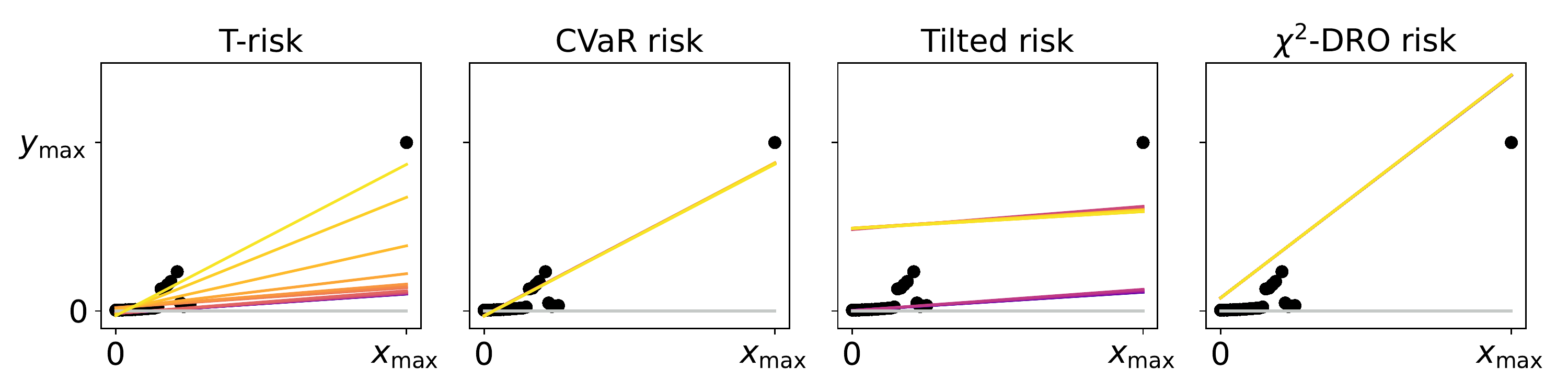}\\
\hfill\includegraphics[width=0.24\textwidth]{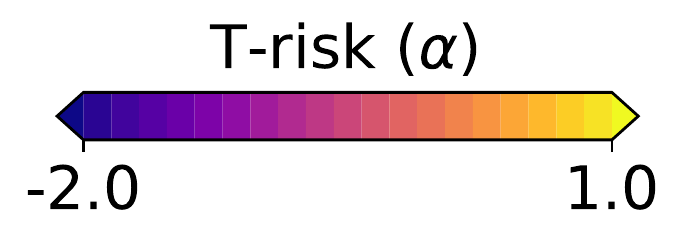}\includegraphics[width=0.24\textwidth]{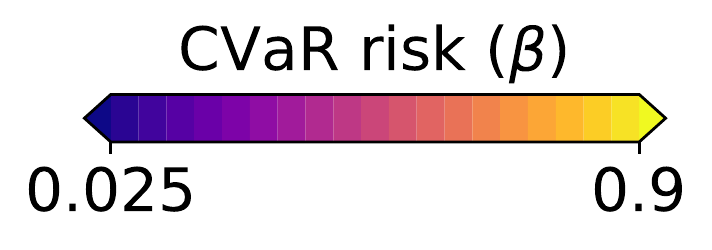}\includegraphics[width=0.24\textwidth]{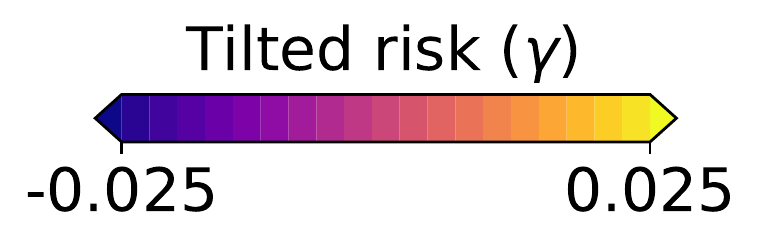}\includegraphics[width=0.24\textwidth]{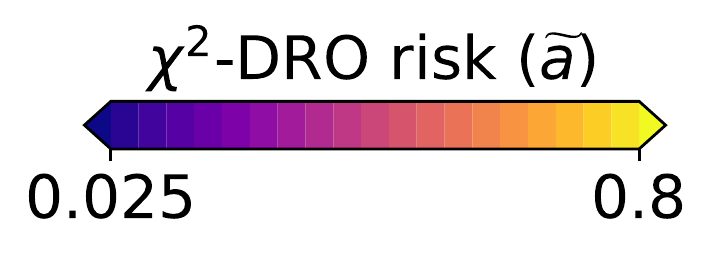}
\caption{Using T-risks with thresholds set to median initialization error, we can achieve a smooth transition between outlier-robust and outlier-sensitive solutions, even under high-leverage points. Top: original ``phones'' dataset. Bottom: modified data including a single high-leverage point.}
\label{fig:tail_control}
\end{figure}

\subsection{Distribution control under SGD on benchmark datasets}\label{sec:applications_real}

\paragraph{Experiment setup}
Finally we consider tests using datasets and models which are orders of magnitude larger than the previous two experimental setups. We use several well-known benchmark datasets for multi-class classification; details are given in \S{\ref{sec:supplementary_info_applications}}. All features are normalized to the unit interval (categorical variables are one-hot), and scores for each class are computed by a linear combination of features. As a base loss, we use the multi-class logistic regression loss. Here we investigate how a stochastic gradient descent implementation (with averaging) of the empirical T-risk minimization (joint in $(h,\theta)$) behaves as we control the shape parameter $\alpha$ of the Barron class, with $\sigma = 0.99$ and $\eta=1$ fixed throughout. We record the base loss (logistic loss) and zero-one loss (misclassification error) at the start and end of each epoch. We also record the full base loss distribution after the last epoch concludes. We run 5 independent trials, in which the data is shuffled and initial parameters are determined randomly. In each trial, for all datasets and methods, we use a mini-batch size of 32, and we run 30 epochs. As reference algorithms, we consider ``vanilla'' ERM, using the traditional risk $\exx_{\ddist}\loss(h)$, and mean-variance implemented as a special case of T-risk (with $\rho(\cdot)=(\cdot)^{2}/2$ and $\eta=1$). 80\% of the data is used for training, 10\% for validation, and 10\% for testing. For each risk class, we try five different step sizes. Validation data is used to evaluate different step sizes and choose the best one for each risk setting. All the results we present here are based on loss values computed on the test set: solid lines represent averages taken over trials, and shaded areas denote $\pm$ standard deviation over trials.

\paragraph{Representative results}
In Figures \ref{fig:real_emnist_balanced}--\ref{fig:real_covtype}, we give results based on the following two datasets: ``extended MNIST'' (47 classes, balanced) \citep{cohen2017a}, and ``cover type'' (7 classes, imbalanced) \citep{blackard1999a}. We plot the average and standard deviation of the base and zero-one losses as a function of epoch number, plus give a histogram of test (base) losses for a single trial, compared with the loss distribution incurred by a random initialization of the same model (left-most plot; gray is test, black is training). Colors are analogous to previous experiments, here evenly spaced over the allowable range ($1 \leq \alpha \leq 2$). It is clear how modifying the Barron dispersion function shape across this range lets us flexibly interpolate between the test (base) loss distribution achieved by a traditional ERM solution and a mean-variance solution, in terms of both the mean and standard deviation. This monotonicity (as a function of $\alpha$) is salient in the base loss, but this does not always appear for the zero-one loss. Since these datasets have normalized features with negligible label noise, egregious outliers are rare, and thus the trends observed for the mean and standard deviation here also hold for outlier-resistant location-dispersion pairs such as the median and the median-absolute-deviations about the median. Similar results for several other datasets are provided in \S{\ref{sec:supplementary_info_figures}}, and we remark that the key trends hold across all datasets tested. We have seen in our previous experiments how under heavy-tailed losses/gradients the T-risk solution can differ greatly from that of the vanilla ERM solution, so it is interesting to observe how under large, normalized, clean classification datasets, the T-risk allows us to very smoothly control a tradeoff between average test loss and variance on the test set.

\begin{figure}[t]
\centering
\includegraphics[width=0.5\textwidth]{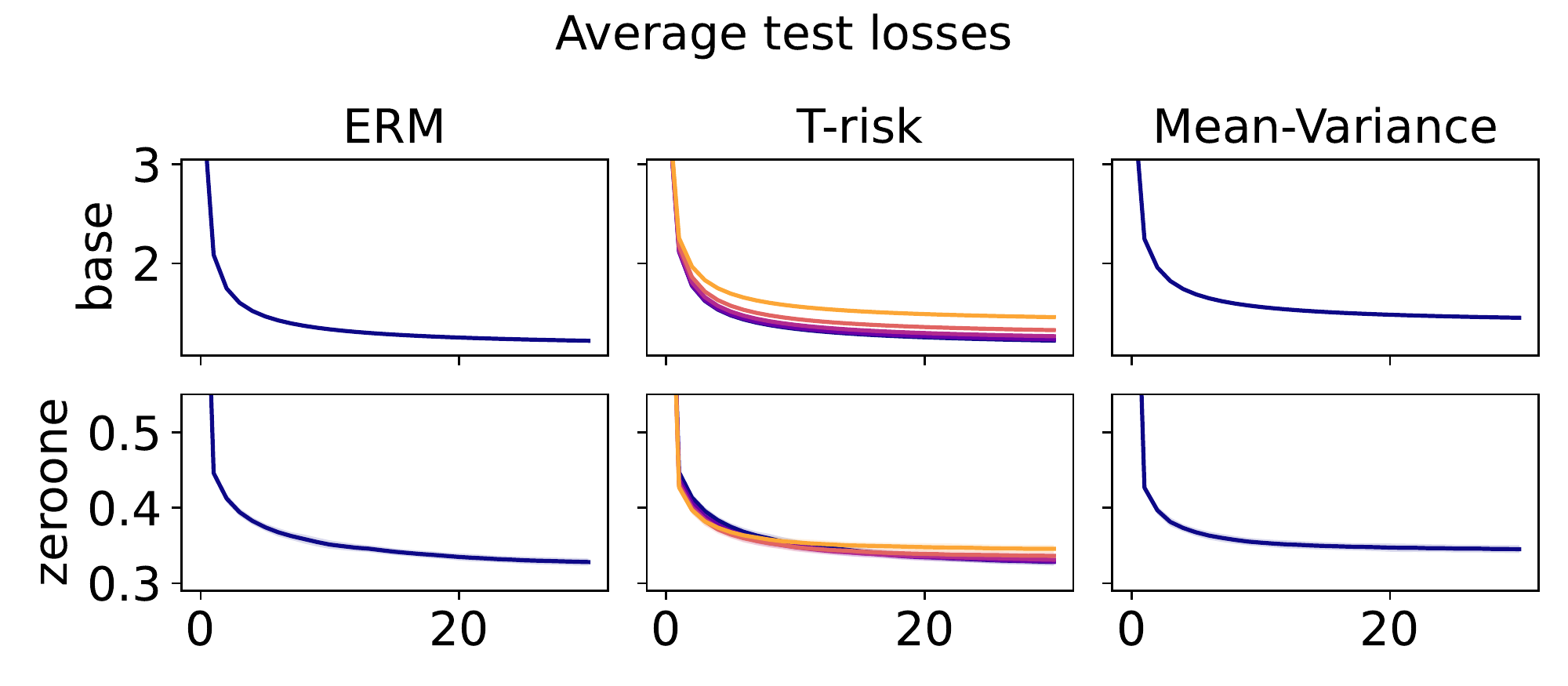}\includegraphics[width=0.5\textwidth]{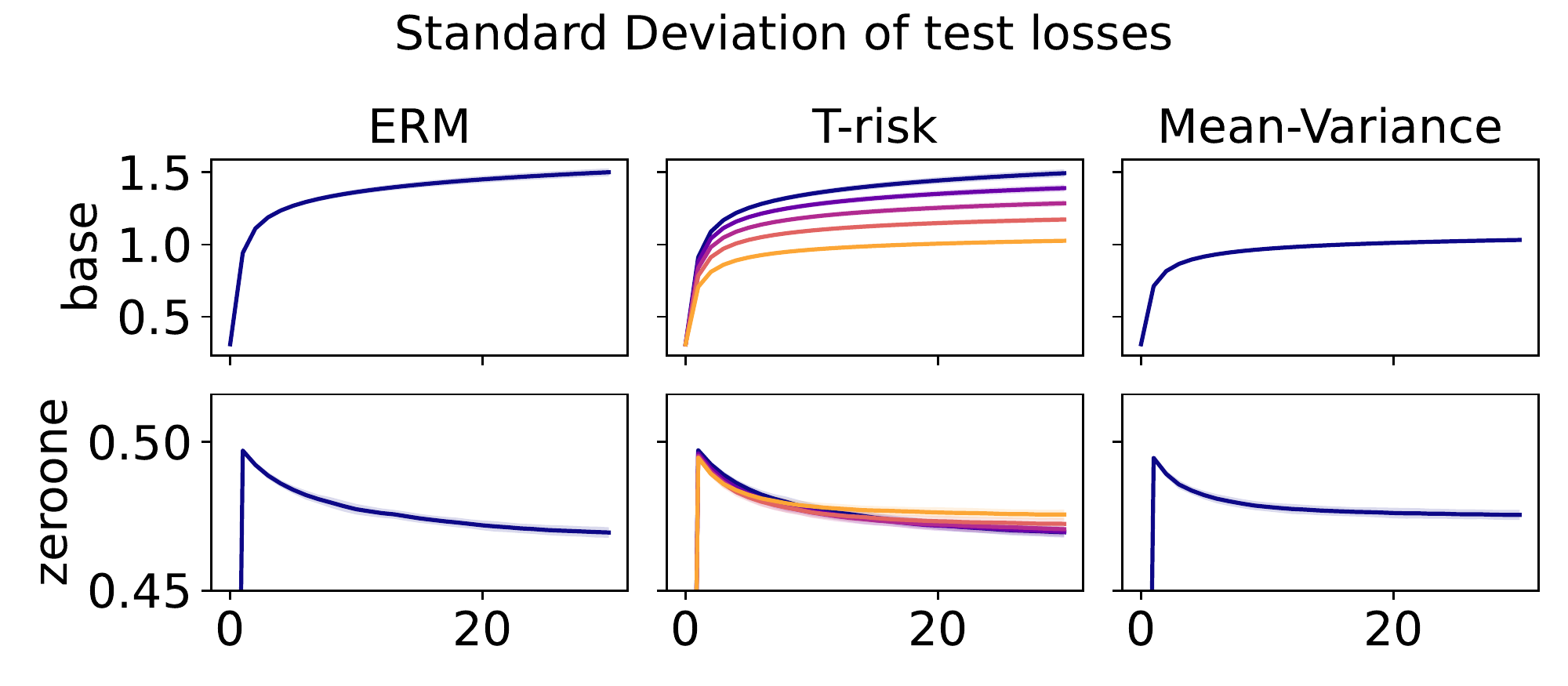}\\
\includegraphics[width=1.0\textwidth]{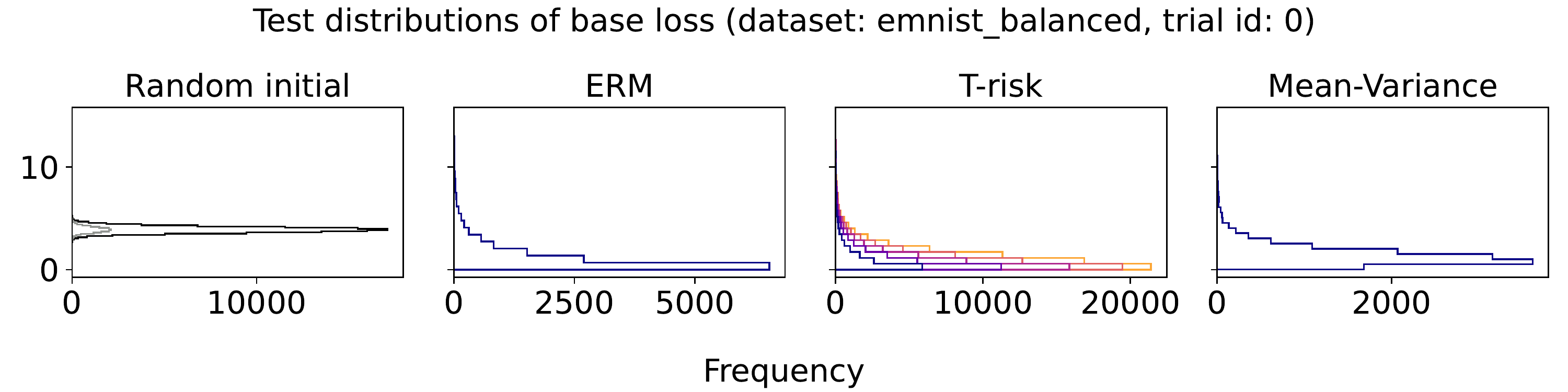}
\caption{Using T-risk to interpolate between test loss distributions (dataset: \texttt{emnist\_balanced}).}
\label{fig:real_emnist_balanced}
\end{figure}

\begin{figure}[t]
\centering
\includegraphics[width=0.5\textwidth]{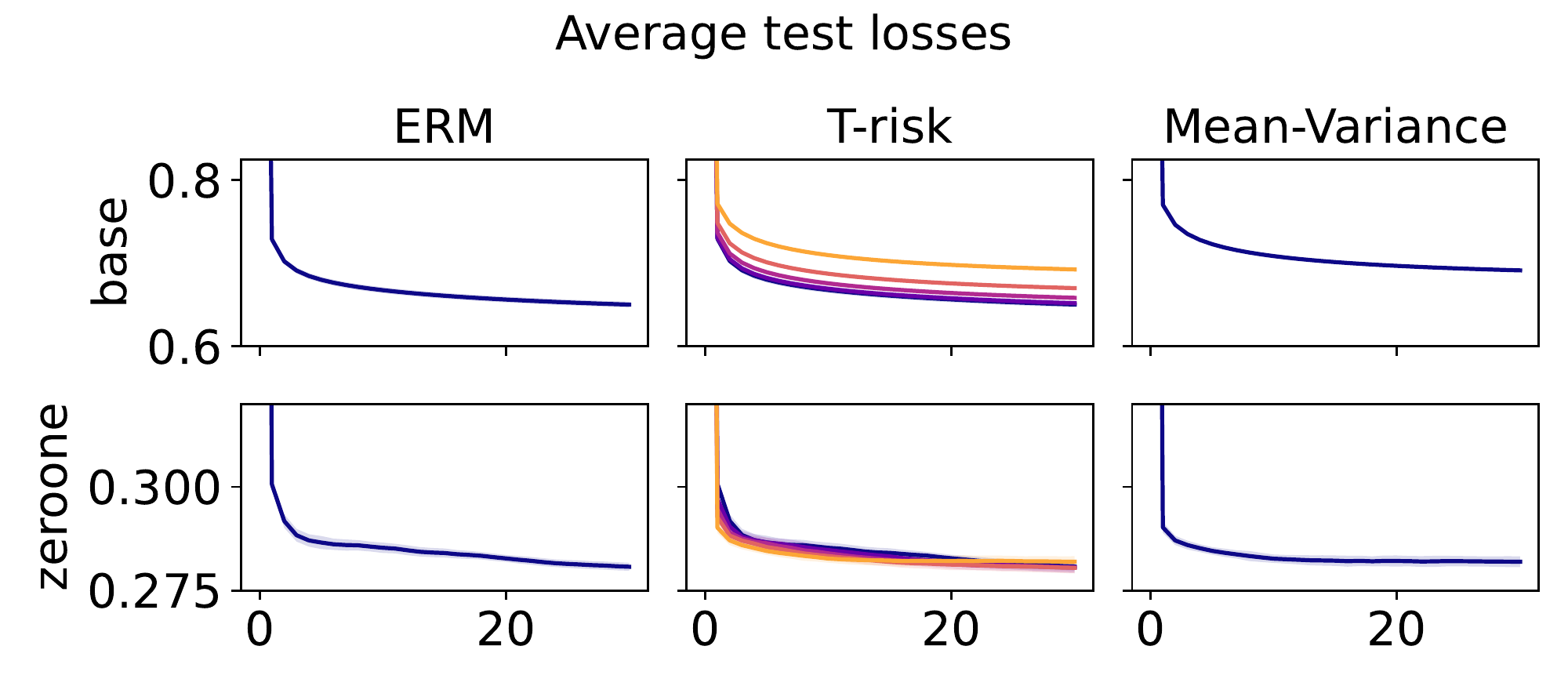}\includegraphics[width=0.5\textwidth]{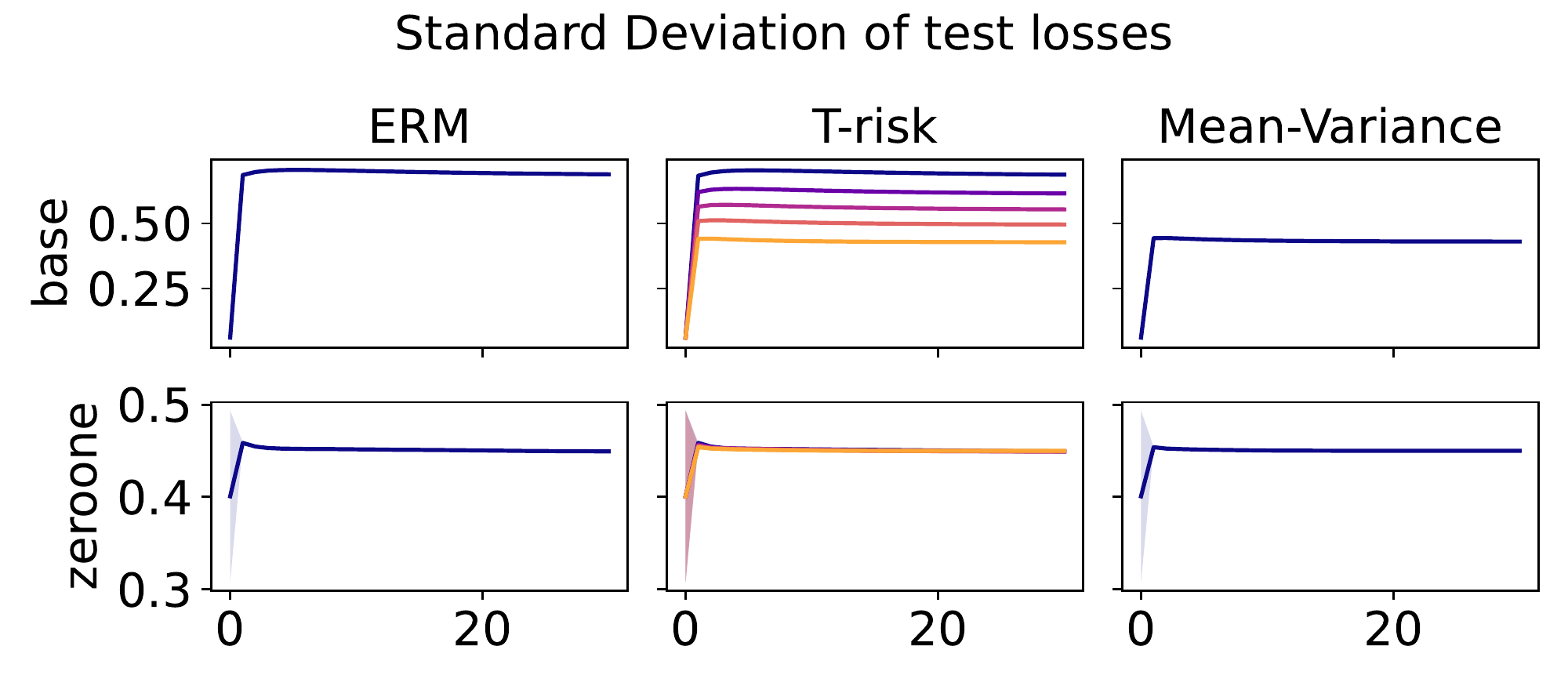}\\
\includegraphics[width=1.0\textwidth]{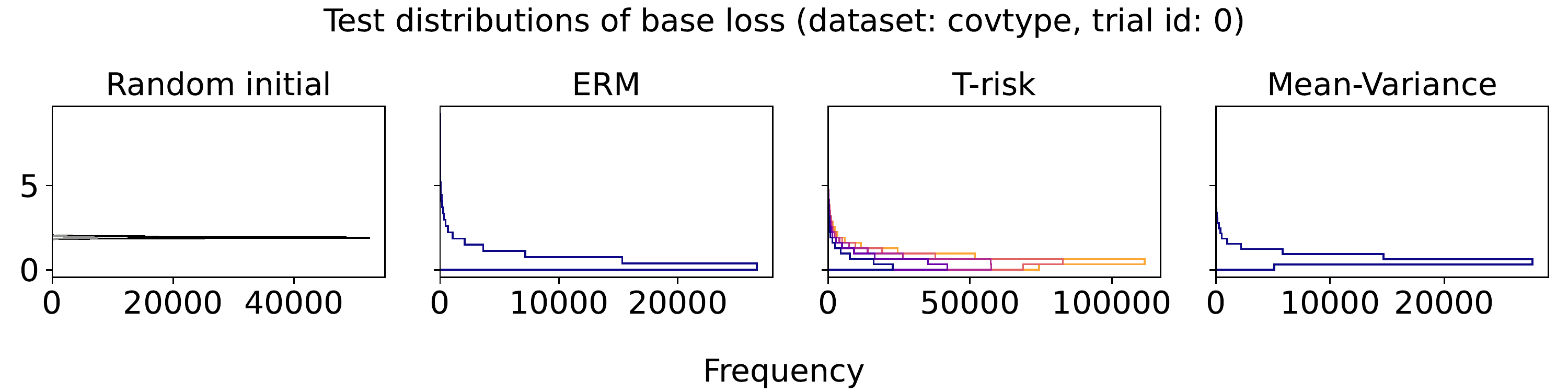}
\caption{Analogous to Figure \ref{fig:real_emnist_balanced} (dataset: \texttt{covtype}).}
\label{fig:real_covtype}
\end{figure}

\section*{Acknowledgments}
This work was supported by JST ACT-X Grant Number JPMJAX200O and JST PRESTO Grant Number JPMJPR21C6.

\bibliographystyle{./refs/apalike}
\bibliography{./refs/refs}

\clearpage

\appendix

\section{Appendix summary}\label{sec:appendix_summary}

For ease of reference, we start the appendix with a concise tables of contents.
\begin{itemize}
\item \S{\ref{sec:supplementary_info}}: Supplementary information
\begin{itemize}
\item \S{\ref{sec:proofs_tilted}}: Details for tilted risk
\item \S{\ref{sec:proofs_barronlimits}}: Details for Barron class limits
\item \S{\ref{sec:supplement_barron}}: Additional lemmas for Barron class and T-risk
\item \S{\ref{sec:learning_smoothness}}: Smoothness of the T-risk
\item \S{\ref{sec:supplementary_info_applications}}: Experimental details
\item \S{\ref{sec:supplementary_info_figures}}: Additional figures
\end{itemize}
\item \S{\ref{sec:proofs}}: Detailed proofs
\begin{itemize}
\item \S{\ref{sec:proofs_main}}: Proofs of results in the main text
\item \S{\ref{sec:proofs_smoothness_limited}}: Smoothness computations (proof of Lemma \ref{lem:smoothness_limited})
\end{itemize}
\item \S{\ref{sec:additional_facts}}: Additional technical facts
\begin{itemize}
\item \S{\ref{sec:additional_facts_lipschitz}}: Lipschitz properties
\item \S{\ref{sec:additional_facts_convex}}: Convexity
\item \S{\ref{sec:proofs_barron_derivs}}: Derivatives for the Barron class
\item \S{\ref{sec:additional_facts_ineqs}}: Elementary inequalities
\item \S{\ref{sec:additional_facts_coercive}}: Expected dispersion is coercive
\end{itemize}
\end{itemize}

\section{Supplementary information}\label{sec:supplementary_info}

Here we provide additional details that did not fit into the main body of the paper.

\subsection{Details for tilted risk}\label{sec:proofs_tilted}

Let $\rdv{X} \sim \ddist$ be an arbitrary random variable. Assuming the distribution is such that we can differentiate through the integral, we have
\begin{align}\label{eqn:proofs_tilted_1}
\frac{\dif}{\dif\theta}\left[ \theta + \frac{1}{\gamma}\left(\exx_{\ddist}\mathrm{e}^{\gamma(\rdv{X}-\theta)}-1\right) \right] = 0 \iff \exx_{\ddist}\mathrm{e}^{\gamma(\rdv{X}-\theta)} = 1.
\end{align}
Let $\theta^{\ast}$ be any value that satisfies the first-order optimality condition in (\ref{eqn:proofs_tilted_1}). It follows that
\begin{align*}
\theta^{\ast} + \frac{1}{\gamma}\left(\exx_{\ddist}\mathrm{e}^{\gamma(\rdv{X}-\theta^{\ast})}-1\right) = \theta^{\ast}.
\end{align*}
It is easy to confirm that setting $\theta^{\ast} = (1/\gamma)\log(\exx_{\ddist}\mathrm{e}^{\gamma \rdv{X}})$ gives us a valid solution. For more background, see the recent works of \citet{li2021a,li2021b} and the references therein. Note also that this log-exponential criterion appears (with $\gamma=1$) in \citet[Ex.~8]{rockafellar2013a}.

\subsection{Details for Barron class limits}\label{sec:proofs_barronlimits}

For the limit as $\alpha \to 0$, use the fact that for any $a > 0$, we have
\begin{align}
\lim\limits_{x \to 0_{+}} \frac{(a^{x}-1)}{x} = \log(a).
\end{align}
This equality is sometimes known as Halley's formula. For the limit as $\alpha \to -\infty$, first note that for any $\alpha < 0$ we can write $\abs{2-\alpha} = 2 + \abs{\alpha}$, and thus $\abs{\alpha}/2 = (\abs{2-\alpha}/2) - 1$. With this in mind, for any $\alpha < 0$ we can observe
\begin{align*}
\left(1 + \frac{x^{2}}{\abs{\alpha-2}}\right)^{\alpha/2} = \left(1 + \frac{x^{2}}{\abs{\alpha-2}}\right)^{-\abs{\alpha}/2} & = \left(1 + \frac{x^{2}}{\abs{\alpha-2}}\right)^{1-(\abs{2-\alpha}/2)}\\
& = \frac{\left(1 + \frac{x^{2}}{\abs{\alpha-2}}\right)}{\sqrt{\left(1 + \frac{x^{2}}{\abs{\alpha-2}}\right)^{\abs{\alpha-2}}}}\\
& \to \frac{1}{\sqrt{\exp(x^{2})}} = \exp(-x^{2}/2),
\end{align*}
where the limit is taken as $\alpha \to -\infty$, and follows from the classical limit characterization of the exponential function. For the limit as $\alpha \to 2_{-}$, first note that
\begin{align*}
\abs{\alpha-2}\left(1 + \frac{x^{2}}{\abs{\alpha-2}}\right)^{\alpha/2} = \left(\abs{\alpha-2}^{2/\alpha} + \abs{\alpha-2}^{(2/\alpha)-1}x^{2}\right)^{\alpha/2}
\end{align*}
and that as long as $\alpha < 2$, we can write
\begin{align*}
\abs{\alpha-2}^{(2/\alpha)-1} = (2-\alpha)^{(2-\alpha)/\alpha} = \left(u^{u}\right)^{1/(2-u)}
\end{align*}
where we have introduced $u \defeq 2-\alpha$. Taking $\alpha \to 2_{-}$ amounts to $u \to 0_{+}$, and thus using the fact that $u^{u} \to 1$ as $u \to 0_{+}$, the desired result follows from straightforward analysis.

\subsection{Additional lemmas for Barron class and T-risk}\label{sec:supplement_barron}

\begin{lem}[Dispersion function convexity and smoothness]\label{lem:basic_barron}
Consider $\rho_{\sigma}(x) \defeq \rho(x/\sigma;\alpha)$ with $\rho(\cdot;\alpha)$ from the Barron class (\ref{eqn:barron}), with parameter $-\infty \leq \alpha \leq 2$. The following properties hold for any choice of $\sigma > 0$.
\begin{itemize}
\item \underline{Case of $\alpha = 2$:}\\$\rho_{\sigma}$ is convex and $1/\sigma^{2}$-smooth on $\RR$.
\item \underline{Case of $\alpha = 0$:}\\$\rho_{\sigma}$ is convex on $[-\sqrt{2}\sigma,\sqrt{2}\sigma]$, and is $1/(\sqrt{2}\sigma)$-Lipschitz and $1/\sigma^{2}$-smooth on $\RR$.
\item \underline{Case of $\alpha = -\infty$:}\\$\rho_{\sigma}$ is convex on $[-\sigma,\sigma]$, and is $(1/\sigma)\exp(-1/2)$-Lipschitz and $1/\sigma^{2}$-smooth on $\RR$.
\item \underline{Otherwise:}\\
$\rho_{\sigma}$ is $1/\sigma^{2}$-smooth on $\RR$. When $\alpha \geq 1$, $\rho_{\sigma}$ is convex on $\RR$. When $\alpha = 1$, $\rho_{\sigma}$ is also $1/\sigma$-Lipschitz on $\RR$. Else, when $\alpha < 1$, we have that $\rho_{\sigma}$ is convex between $\pm \sigma\sqrt{\abs{\alpha-2}/(1-\alpha)}$, and is $(1/\sigma)(\sqrt{(1-\alpha)/\abs{\alpha-2}})^{1-\alpha}$-Lipschitz on $\RR$.
\end{itemize}
Furthermore, all these coefficients are tight (see also Figure \ref{fig:barron_bounds}).
\end{lem}
\begin{figure}[t]
\centering
\includegraphics[width=1.0\textwidth]{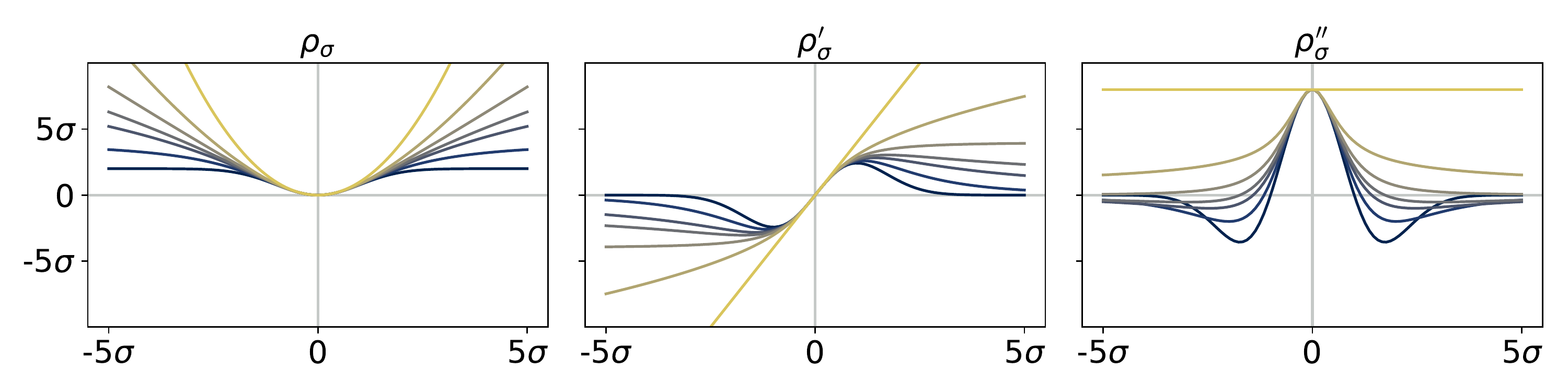}\\
\hfill\includegraphics[width=0.95\textwidth]{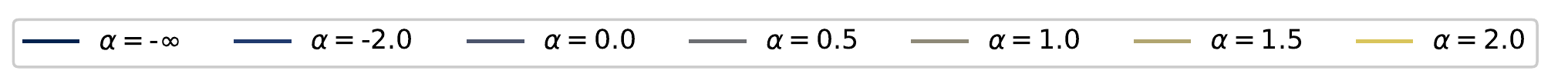}
\caption{Here we consider the function $x \mapsto \rho(x/\sigma;\alpha)$, with $\rho$ from the Barron class (\ref{eqn:barron}), for a variety of shape parameter values $\alpha$, with scale parameter $\sigma = 0.5$ fixed. We plot graphs for this function (left), its first derivative (center), and its second derivative (right), each computed over $\pm 5\sigma$.}
\label{fig:dispersion_barron}
\end{figure}
\begin{figure}[t]
\centering
\includegraphics[width=0.5\textwidth]{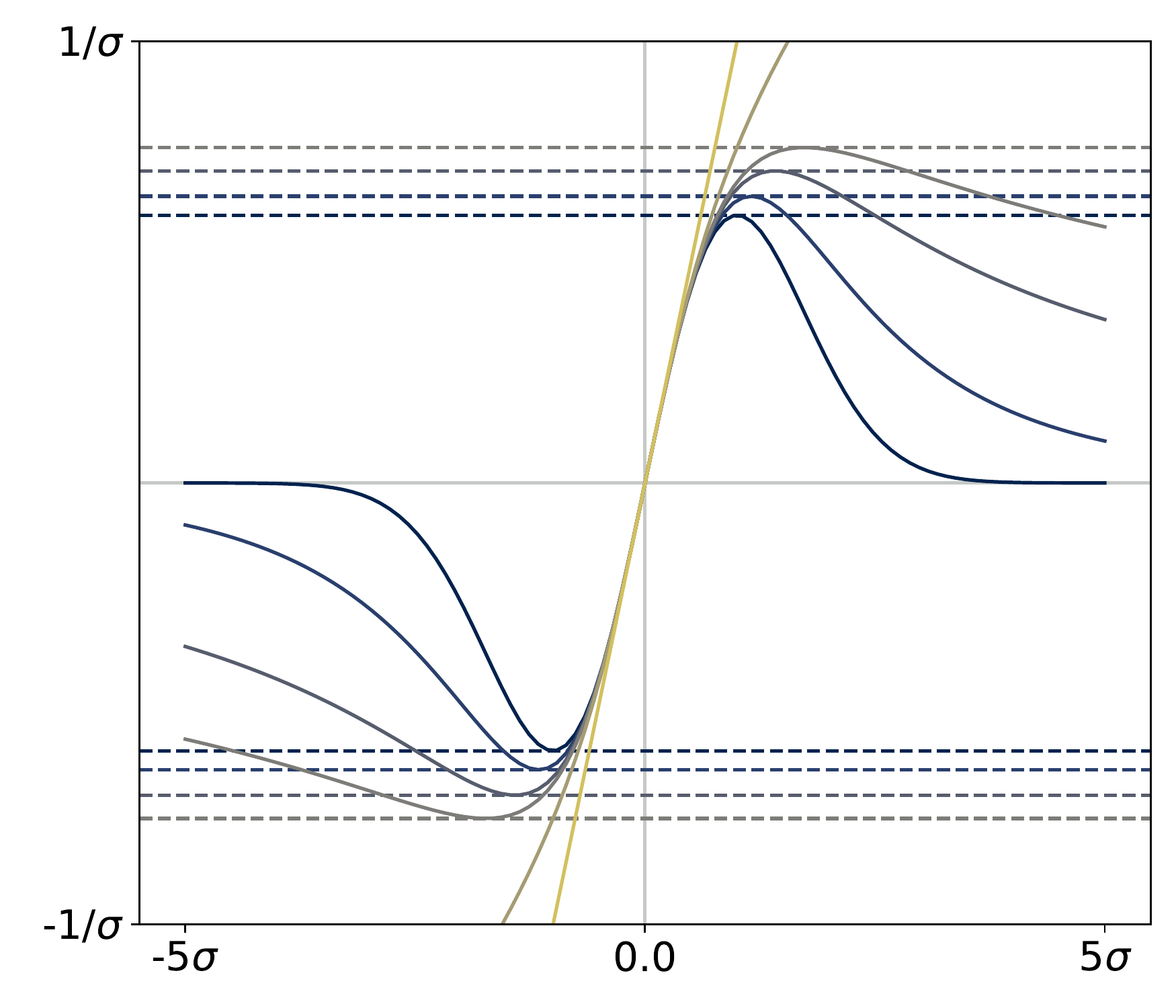}\includegraphics[width=0.5\textwidth]{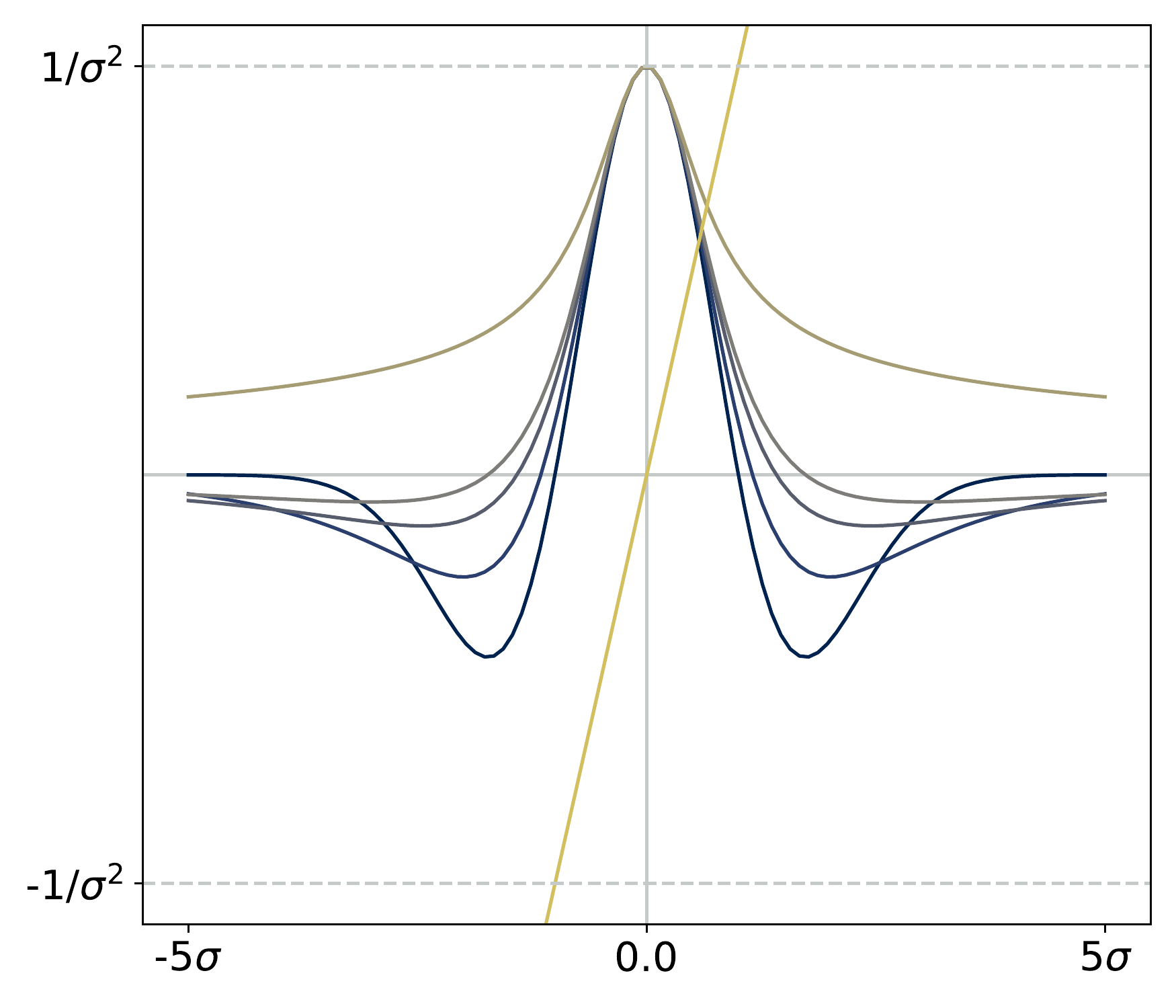}
\caption{Here we show zoomed-in versions of the two right-most plots in Figure \ref{fig:dispersion_barron}, now with dashed horizontal lines denoting the Lipschitz (left) and smoothness (right) coefficients stated in Lemma \ref{lem:basic_barron} (both positive and negative values). The Lipschitz coefficients depend on $\alpha$ and $\sigma$, and the colors in the left plot reflect the $\alpha$ value. The smoothness coefficient $1/\sigma^{2}$ is independent of $\alpha$, and is drawn in light gray. Numerically, we can see that the coefficients are tight, though tighter \emph{lower} bounds for the second derivatives are possible.}
\label{fig:barron_bounds}
\end{figure}
The dispersion $\dsp_{\rho}(h;\theta) \defeq \exx_{\ddist}\rho(\loss(h)-\theta)$ plays a prominent role in the risk definitions considered in this paper, and one is naturally interested in the properties of the map $\theta \mapsto \dsp_{\rho}(h;\theta)$. The following lemma shows that using $\rho = \rho_{\sigma}$ from the Barron class, we can differentiate under the integral without needing any additional conditions beyond those required for finiteness.
\begin{lem}\label{lem:dsp_diffable}
Let $\rho_{\sigma}$ be as in Lemma \ref{lem:basic_barron}. Assume that the random loss $\loss(h)$ is $\sigmafield$-measurable in general, and that $\exx_{\ddist}\lvert \loss(h) \rvert < \infty$ holds whenever $1 < \alpha \leq 2$. It follows that the first two derivatives are $\ddist$-integrable, namely that
\begin{align}
\lvert \exx_{\ddist}\dv{\rho}_{\sigma}(\loss(h)-\theta) \rvert < \infty, \qquad \lvert \exx_{\ddist}\ddv{\rho}_{\sigma}(\loss(h)-\theta) \rvert < \infty
\end{align}
for any $\theta \in \RR$. Furthermore, the function $\theta \mapsto \dsp_{\rho}(h;\theta)$ is twice-differentiable on $\RR$, and satisfies the Leibniz integration property for both derivatives, that is
\begin{align}
\left.\frac{\dif}{\dif\theta} \dsp_{\rho}(h;\theta)\right\vert_{\theta = u} = -\exx_{\ddist}\dv{\rho}_{\sigma}(\loss(h)-u), \qquad \left.\frac{\dif^{2}}{\dif\theta^{2}} \dsp_{\rho}(h;\theta)\right\vert_{\theta = u} = \exx_{\ddist}\ddv{\rho}_{\sigma}(\loss(h)-u)
\end{align}
for any $u \in \RR$.\footnote{Let us emphasize that $\dv{\rho}_{\sigma}$ and $\ddv{\rho}_{\sigma}$ denote the first and second derivatives of $x \mapsto \rho(x/\sigma;\alpha)$, which differ from $\dv{\rho}(x/\sigma;\alpha)$ and $\ddv{\rho}(x/\sigma;\alpha)$ by a $\sigma$-dependent factor; see \S{\ref{sec:proofs_barron_derivs}} for details.}
\end{lem}
With first-order information about the expected dispersion in hand, one can readily obtain conditions under which the special case of T-risk $\underline{\risk}_{\rho}(h;\eta) \defeq \inf_{\theta \in \RR} \risk_{\rho}(h;\theta,\eta)$ is finite and determined by a meaningful ``optimal threshold.''
\begin{lem}\label{lem:optimal_threshold_existence}
Following the setup of Lemma \ref{lem:dsp_finite}, let $\dsp_{\rho}(h;\theta) < \infty$ for all $\theta \in \RR$. If $\alpha > 1$, then for any choice of $\sigma > 0$ and $\eta \in \RR$, there exists a finite optimal threshold $\mt_{\rho}(h;\eta) \in \RR$ such that
\begin{align}\label{eqn:optimal_threshold_existence_trisk}
\underline{\risk}_{\rho}(h;\eta) = \eta\mt_{\rho}(h;\eta) + \dsp_{\rho}(h;\mt_{\rho}(h;\eta)).
\end{align}
In the case of $\alpha = 1$, we have that (\ref{eqn:optimal_threshold_existence_trisk}) holds if and only if $\abs{\eta} < 1/\sigma$. In both cases, the optimal threshold is unique. In the case of $\alpha < 1$, then there is no minimum and thus $\underline{\risk}_{\rho}(h;\eta) = -\infty$.
\end{lem}
\begin{rmk}
We have overloaded our notation $\mt_{\rho}(h;\eta)$ in Lemma \ref{lem:optimal_threshold_existence}, recalling that we have used the same notation to denote the \emph{set} of optimal thresholds for the T-risk in (\ref{eqn:trisk_general}). This saves us from having to introduce additional symbols, and should not lead to any confusion since we only do this overloading when the solution set contains a single unique solution.
\end{rmk}

\subsection{Smoothness of the T-risk}\label{sec:learning_smoothness}

When the objective $\risk_{\rho}(h;\theta,\eta)$ in (\ref{eqn:trisk_general}) is sufficiently smooth, we can apply well-established analytical techniques to control the gradient norms of stochastic gradient-based learning algorithms. Assuming we have unbiased first-order stochastic feedback as in (\ref{eqn:transformed_grad})--(\ref{eqn:unbiased_new}), we will always have to deal with terms of the form $\exx_{\ddist}\left[\dv{\rho}_{\sigma}(\loss(h)-\theta)\dv{\loss}(h)\right]$. Defining $f(h,\theta) \defeq \dv{\rho}_{\sigma}(\loss(h)-\theta)\dv{\loss}(h)$ for readability, and considering the function difference at two arbitrary points $(h_{1},\theta_{1})$ and $(h_{2},\theta_{2})$, first note that
\begin{align}
\nonumber
f&(h_{1},\theta_{1})-f(h_{2},\theta_{2})\\
\label{eqn:difference_tricky}
& = \underbrace{\dv{\rho}_{\sigma}(\loss(h_{1})-\theta_{1})\left[ \dv{\loss}(h_{1}) - \dv{\loss}(h_{2}) \right]}_{A} + \underbrace{\left[\dv{\rho}_{\sigma}(\loss(h_{1})-\theta_{1})-\dv{\rho}_{\sigma}(\loss(h_{2})-\theta_{2})\right]\dv{\loss}(h_{2})}_{B}.
\end{align}
In the case of $\rho(\cdot) = \rho(\cdot;\alpha)$ from the Barron class (\ref{eqn:barron}), when $-\infty \leq \alpha \leq 1$, we have that $\dv{\rho}$ is both bounded ($\Abs{\dv{\rho}}_{\infty} < \infty$) and Lipschitz continuous on $\RR$ (see Lemma \ref{lem:basic_barron}). This means that all we need in order to control $\exx_{\ddist}A + \exx_{\ddist}B$ is for $\loss$ to be smooth (for control of $A$) and for $\dv{\loss}(\cdot)$ to have a norm bounded over $\HH$ (for control of $B$); see \S{\ref{sec:proofs_smoothness_limited}} for more details. Things are more difficult in the case of $1 < \alpha \leq 2$, since the dispersion function $\rho$ is not (globally) Lipschitz, meaning that $\Abs{\dv{\rho}}_{\infty} = \infty$. Even if $\loss$ is smooth, when the threshold parameter is left unconstrained, it will always be possible for $\Abs{\exx_{\ddist}A} \to \infty$ as $\abs{\theta_{1}} \to \infty$, impeding smoothness guarantees for $\risk_{\rho}$ in this setting.

Let us proceed by distilling the preceding discussion into a set of concrete conditions that are sufficient to make $(h,\theta) \mapsto \risk_{\rho}(h;\theta,\eta)$ amenable to standard analysis techniques for stochastic gradient-based algorithms. For readability, we write $\Abs{\dv{\loss}}_{\HH} \defeq \sup\{\Abs{\dv{\loss}(h)}: h \in \HH\}$.
\begin{itemize}
\item[\namedlabel{asmp:grad_moments}{\textup{A1}}.] \textbf{Moment bound for loss gradient.} For any choice of $h_{1},h_{2} \in \HH$, $0 < c < 1$, and $k \in \{1,2\}$, the loss $\loss$ is differentiable at $ch_{1} + (1-c)h_{2}$, and satisfies
\begin{align}\label{eqn:grad_moments}
\exx_{\ddist}\left(\sup_{0 < c < 1}\Abs{\dv{\loss}(ch_{1} + (1-c)h_{2})}\right)^{k} \leq \exx_{\ddist}\Abs{\dv{\loss}}_{\HH}^{k} < \infty.
\end{align}

\item[\namedlabel{asmp:loss_smooth}{\textup{A2}}.] \textbf{Loss is smooth in expectation.} There exists $0 < \smooth_{1} < \infty$ such that for any choice of $h_{1},h_{2} \in \HH$, we have $\exx_{\ddist}\Abs{\dv{\loss}(h_{1})-\dv{\loss}(h_{2})} \leq \smooth_{1}\Abs{h_{1} - h_{2}}$.

\item[\namedlabel{asmp:dispersion}{\textup{A3}}.] \textbf{Dispersion is Lipschitz and smooth.} The function $\rho$ is such that $\Abs{\dv{\rho}}_{\infty} < \infty$, and there exists $0 < \smooth_{2} < \infty$ such that $\abs{\dv{\rho}(x_{1})-\dv{\rho}(x_{2})} \leq \smooth_{2}\abs{x_{1} - x_{2}}$ for all $x_{1},x_{2} \in \RR$.
\end{itemize}
If $\HH$ is a convex set, then the first inequality in \ref{asmp:grad_moments} holds trivially. Note that under \ref{asmp:loss_smooth}, the right-hand side of (\ref{eqn:grad_moments}) will be finite for $k=1$ whenever $\HH$ has bounded diameter and $\exx_{\ddist}\Abs{\dv{\loss}(h)} < \infty$ for some $h \in \HH$. As for \ref{asmp:dispersion}, all the requirements are clearly met by the Barron class with $-\infty \leq \alpha \leq 1$. These conditions imply a Lipschitz property for the gradients, as summarized in the following lemma.
\begin{lem}\label{lem:smoothness_limited}
Let the conditions \ref{asmp:grad_moments}, \ref{asmp:loss_smooth}, and \ref{asmp:dispersion} hold. Then, the T-risk map $(h,\theta) \mapsto \risk_{\rho}(h;\theta,\eta)$ defined in (\ref{eqn:trisk_general}) is smooth on $\HH \times \RR$ in the sense that
\begin{align*}
\Abs{\dv{\risk}_{\rho}(h_{1};\theta_{1},\eta) - \dv{\risk}_{\rho}(h_{2};\theta_{2},\eta)} & \leq \left(\frac{\smooth_{5}}{\sigma} + \frac{\smooth_{2}}{\sigma^{2}}\exx_{\ddist}\Abs{\dv{\loss}}_{\HH}\right)\left( \Abs{h_{1} - h_{2}} + \abs{\theta_{1} - \theta_{2}} \right)
\end{align*}
for any choice of $h_{1}, h_{2} \in \HH$ and $\theta_{1}, \theta_{2} \in \RR$. Here the factor $\smooth_{5}$ is defined $\smooth_{5} \defeq \smooth_{3} + \smooth_{4}$, where
\begin{align*}
\smooth_{3} \defeq \left(\frac{\smooth_{2}}{\sigma}\right)\left[\exx_{\ddist}\Abs{\dv{\loss}}_{\HH}^{2} + \sup_{h \in \HH}\exx_{\ddist}\Abs{\dv{\loss}(h)} \right], \qquad \smooth_{4} \defeq \smooth_{1}\Abs{\dv{\rho}}_{\infty}.
\end{align*}
\end{lem}
\begin{rmk}[Norm on product spaces]\label{rmk:product_norm}
In Lemma \ref{lem:smoothness_limited} we have to deal with norms on product spaces, and in all cases we just use the traditional choice of summing the norms of the constituent elements, i.e., on $\HH \times \RR$, we have $\Abs{(h,\theta)} \defeq \Abs{h} + \abs{\theta}$. Similarly, we have that the gradient $\dv{\risk}_{\rho}(h;\theta,\eta) = (\partial_{h}\risk_{\rho}(h;\theta,\eta), \partial_{\theta}\risk_{\rho}(h;\theta,\eta))$, a pair of linear functionals. As such, the norm of $\dv{\risk}_{\rho}(h;\theta,\eta)$ defined as the sum of the norms of these two constituent functionals.
\end{rmk}
Proving Lemma \ref{lem:smoothness_limited} is straightforward but somewhat tedious. Detailed computations as well as a direct proof are organized in \S{\ref{sec:proofs_smoothness_limited}} for easy reference.

\subsection{Experimental details}\label{sec:supplementary_info_applications}

Here we provide some additional details for the empirical analysis carried out in \S{\ref{sec:trisk_expressive}} and \S{\ref{sec:applications}}. Detailed hyperparameter settings and seeds for exact re-creation of all the results in this paper are available at the GitHub repository cited in \S{\ref{sec:intro}}, and thus here we will not explicitly write all the step sizes, shape settings, etc., but rather focus on concise exposition of points on which we expect readers to desire clarification.

\paragraph{Static risk analysis}
For our experiments in \S{\ref{sec:trisk_expressive}}, we gave just one plot using a log-Normal distribution, but analogous tests were run for a wide variety of parametric distributions. In total, we have run tests for Bernoulli, Beta, $\chi^{2}$, Exponential, Gamma, log-Normal, Normal, Pareto, Uniform, Wald, and Weibull distributions. The settings of each distribution to be sampled from has not been tweaked at all; we set the parameters rather arbitrarily before running any tests. As for the fixed value of $\sigma = 0.5$ in the T-risk across all tests, we tested several values of $\sigma$ ranging from $0.1$ to $10$, and based on the rough position of $\mt_{\rho}(\loss;\eta)$ in the Normal case, we determined $0.5$ as a reasonable representative value; indeed, this settings performs quite well across a very wide range of distributions. Regarding the optimization involved in solving for the optimal threshold $\theta$ (for T-risk, OCE risks, and DRO risk), we use \texttt{minimize\_scalar} from SciPy, with bounded solver type, and valid brackets set manually.

\paragraph{Noisy linear classification}
In the tests described in \S{\ref{sec:applications_2dclass}}, we only give error/norm trajectories for ``representative'' settings of each risk class (Figure \ref{fig:unbounded_below_unhinged}). For T-risk we consider different choices of $1 \leq \alpha \leq 2$, for CVaR we consider $0.025 \leq \beta \leq 0.75$, for tilted risk we consider $\gamma$ between $\pm 0.05$, and for $\chi^{2}$-DRO we consider $0.025 \leq \widetilde{a} \leq 0.35$. For each of these ranges, we evaluate five evenly-spaced candidates (via \texttt{np.linspace}), and representative settings were selected as those which achieved the best classification error (average zero-one loss) after the final iteration. In the event of ties, the smaller hyperparameter value was always selected (via \texttt{np.argmin}).

\paragraph{Regression under outliers}
For the tests introduced in \S{\ref{sec:applications_1dlinreg}}, we have given results for learning algorithms started at a point which is quite accurate for the majority of the data points, but incurs extremely large errors on the outlying minority. This choice of initial value naturally has a strong impact on the behavior of learning algorithms under each risk. For some perspective, in Figure \ref{fig:tail_control_all_updown} we provide results using a different initial value (again, in gray), which complement Figure \ref{fig:tail_control}. Since the naive strategy fixing $\theta$ at the initial median sets the scale extremely large and close to the loss incurred at \emph{most} points, even a large number of gradient-based updates result in minimal change. The basic reason for this is quite straightforward. Since the T-risk gradient is modulated by $\rho_{\sigma}^{\prime}(\loss(h)-\theta)$, and all points are such that either $\loss_{i}(h) \ll \theta$ (the minority) or $\loss_{i}(h) \approx \theta$ (the majority), both cases shrink the norm of the update direction vector $\alpha \leq 1$. When implementing T-risk in the more traditional way (jointly in $(h,\theta)$) and choosing $\alpha \in [1,2]$, we see results that are very similar to CVaR. Finally, we remark that here we have set the hyperparameter ranges with upper bounds low enough that the learning procedure described here does not run into overflow errors.

\paragraph{Classification under larger benchmark datasets}
In \S{\ref{sec:applications_real}} we make use of several well-known benchmark datasets, and in our figures we identify them respectively by the following keywords: \texttt{adult},\footnote{\url{https://archive.ics.uci.edu/ml/datasets/Adult}} \texttt{australian},\footnote{\url{https://archive.ics.uci.edu/ml/datasets/statlog+(australian+credit+approval)}} \texttt{cifar10},\footnote{\url{https://www.cs.toronto.edu/~kriz/cifar.html}} \texttt{covtype},\footnote{\url{https://archive.ics.uci.edu/ml/datasets/covertype}} \texttt{emnist\_balanced},\footnote{\url{https://www.nist.gov/itl/products-and-services/emnist-dataset}} \texttt{fashion\_mnist},\footnote{\url{https://github.com/zalandoresearch/fashion-mnist}} and \texttt{protein}.\footnote{\url{https://www.kdd.org/kdd-cup/view/kdd-cup-2004/Data}} For further background on all of these datasets, please access the URLs provided in the footnotes. As mentioned in the main text, for a $k$-class problem with $d$ features, the predictor is characterized by $k$ weighting vectors $(w_{1},\ldots,w_{k})$, each of which is $w_{j} \in \RR^{d}$ and computes scores as $\tran{w}_{j} x$ for $x \in \RR^{d}$. These weight vectors are penalized using the usual multi-class logistic loss, namely the negative log-likelihood of the $k$-Categorical distribution that arises after passing these scores through the (logistic) softmax function. Regarding step sizes, as mentioned in the main text, we consider choices of factor $c \in \{0.1, 0.5, 1.0, 1.5, 2.0\}$, and set step size to $c / \sqrt{kd}$, where $d$ and $k$ are as just stated. In Figures \ref{fig:real_adult}--\ref{fig:real_protein} of \S{\ref{sec:supplementary_info_figures}}, we provide additional results for the datasets not covered in Figures \ref{fig:real_emnist_balanced}--\ref{fig:real_covtype} from the main text.

\subsection{Additional figures}\label{sec:supplementary_info_figures}

Here we include a number of figures that complement those provided in the main text. A brief summary is given below.
\begin{itemize}
\item Figures \ref{fig:unbounded_below_logistic}--\ref{fig:unbounded_below_hinge} are additional results from the experiments described in \S{\ref{sec:applications_2dclass}}, here using logistic and hinge losses instead of the unhinged loss.
\item Figures \ref{fig:outliers_phones_loss_tails}--\ref{fig:tail_control_all_updown} are related to the regression under outliers task described in \S{\ref{sec:applications_1dlinreg}}. The first figure shows how different regression lines incur very different loss distributions under different convex base loss functions. The second figure illustrates how a different initial value impacts the learned regression lines under each risk class.
\item Figures \ref{fig:real_adult}--\ref{fig:real_protein} are all completely analogous to Figure \ref{fig:real_emnist_balanced} given in the main text of \S{\ref{sec:applications_real}}, but for additional benchmark datasets.
\end{itemize}

\clearpage

\begin{figure}[t]
\centering
\includegraphics[width=0.5\textwidth]{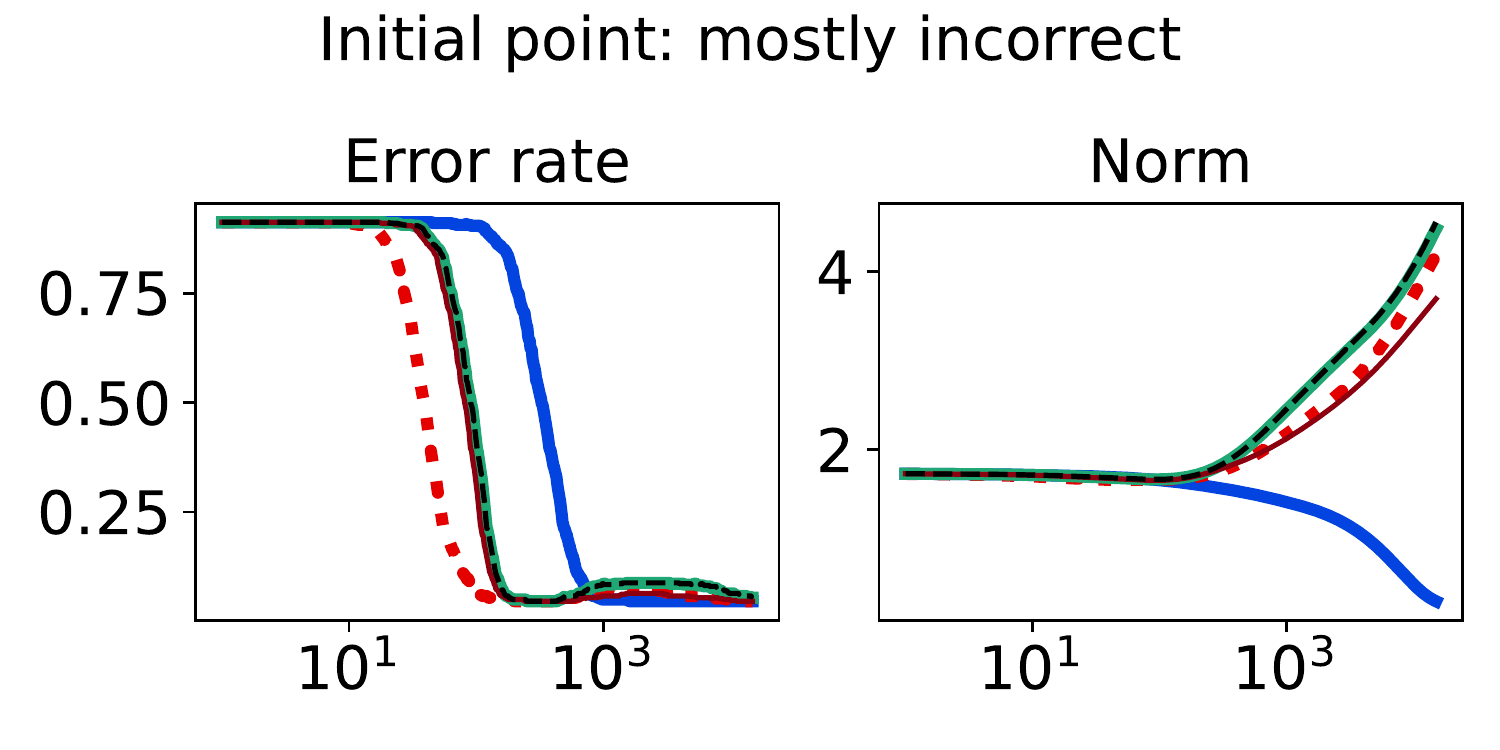}\includegraphics[width=0.5\textwidth]{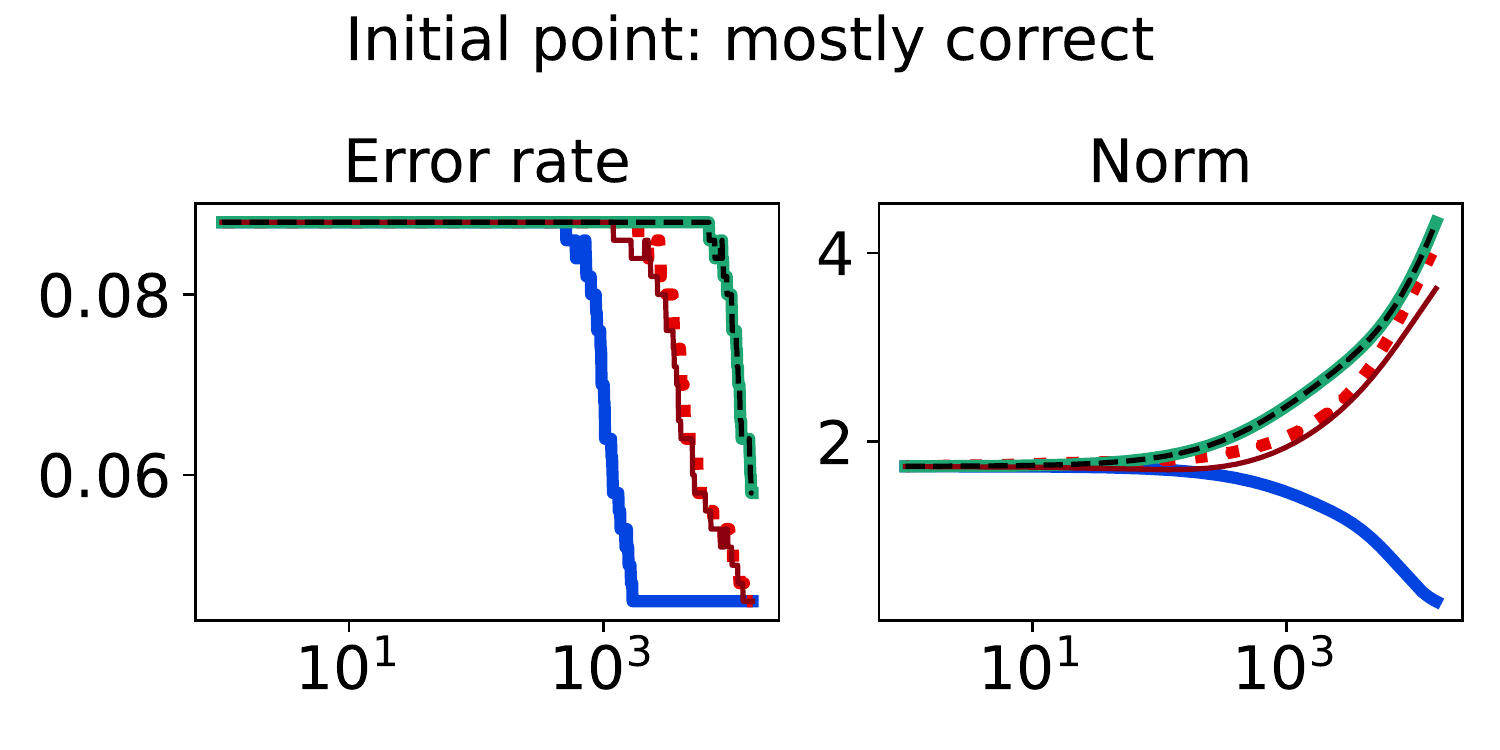}\\
\includegraphics[width=0.75\textwidth]{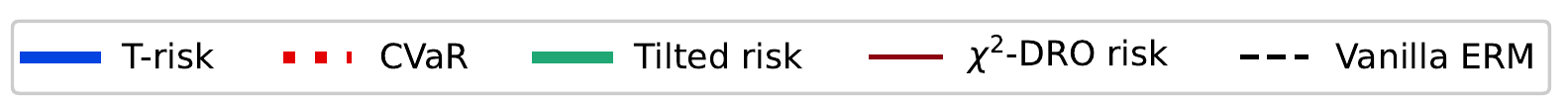}\hspace{0.1cm}
\caption{Analogous to Figure \ref{fig:unbounded_below_unhinged}, this time using logistic loss.}
\label{fig:unbounded_below_logistic}
\end{figure}

\begin{figure}[t]
\centering
\includegraphics[width=0.5\textwidth]{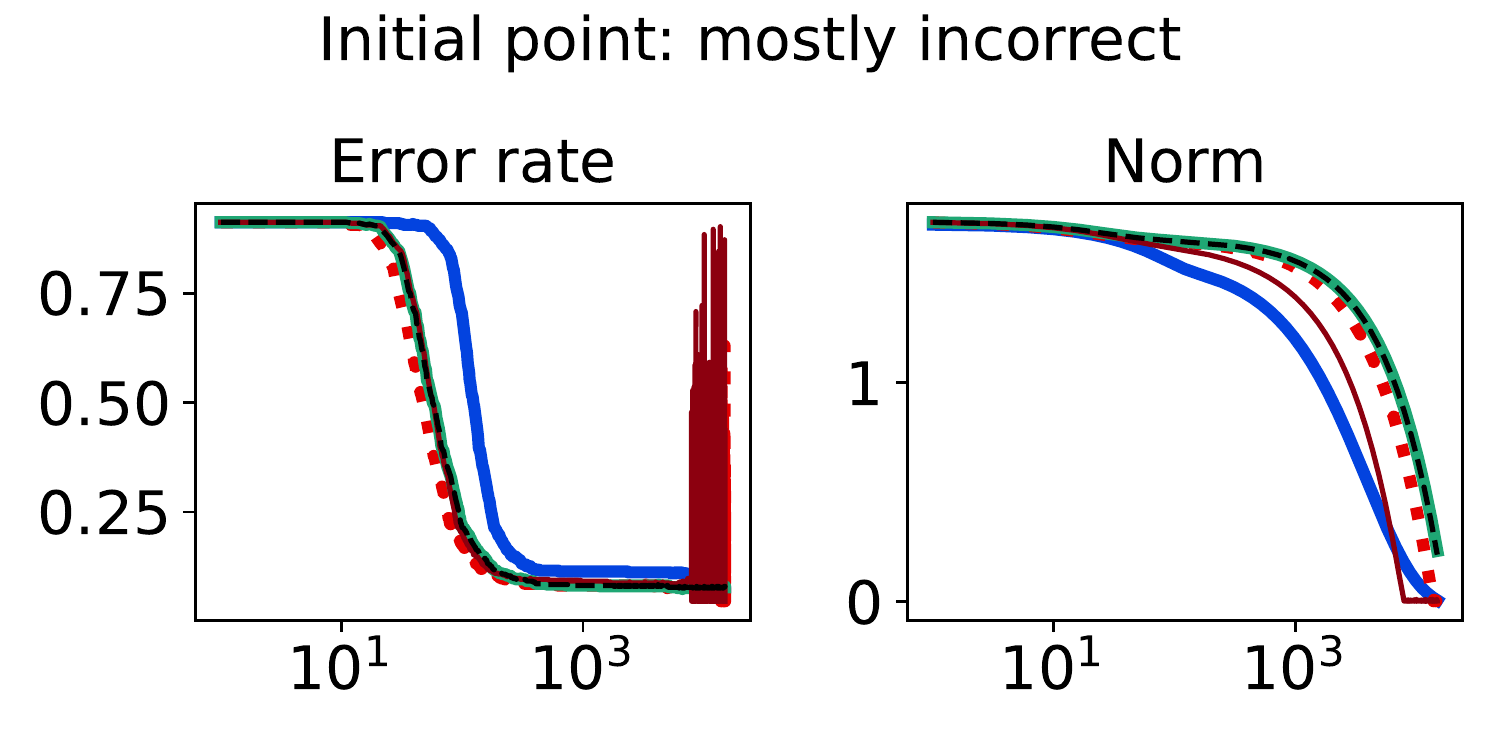}\includegraphics[width=0.5\textwidth]{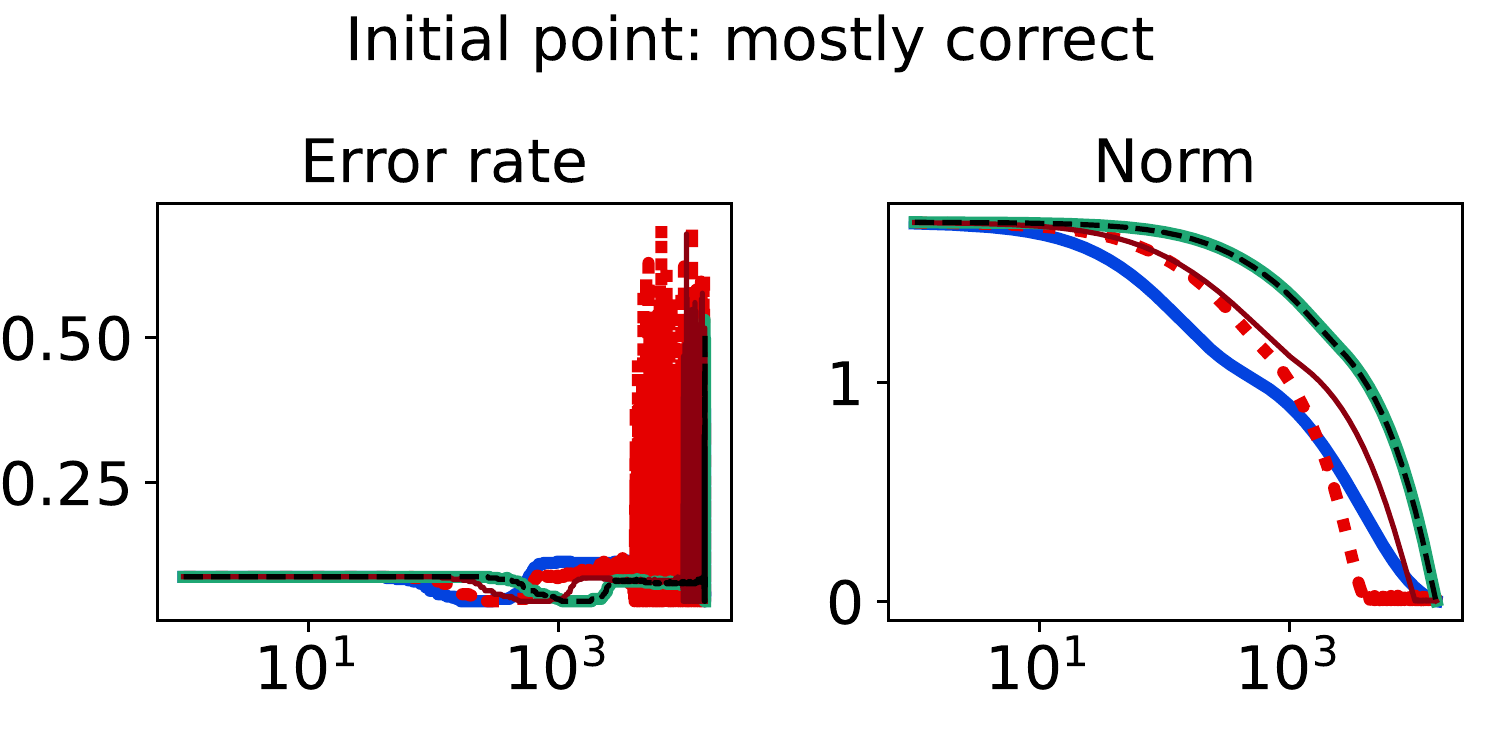}\\
\includegraphics[width=0.75\textwidth]{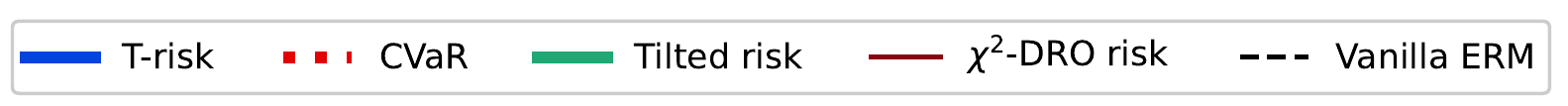}\hspace{0.1cm}
\caption{Analogous to Figure \ref{fig:unbounded_below_unhinged}, this time using the hinge loss.}
\label{fig:unbounded_below_hinge}
\end{figure}

\clearpage

\begin{figure}[t]
\centering
\includegraphics[width=0.28\textwidth]{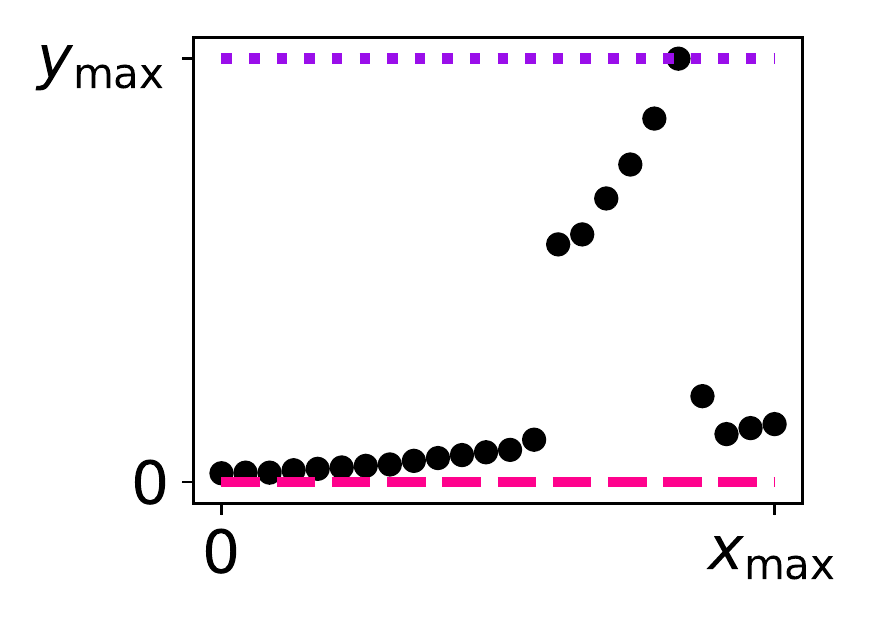}\includegraphics[width=0.72\textwidth]{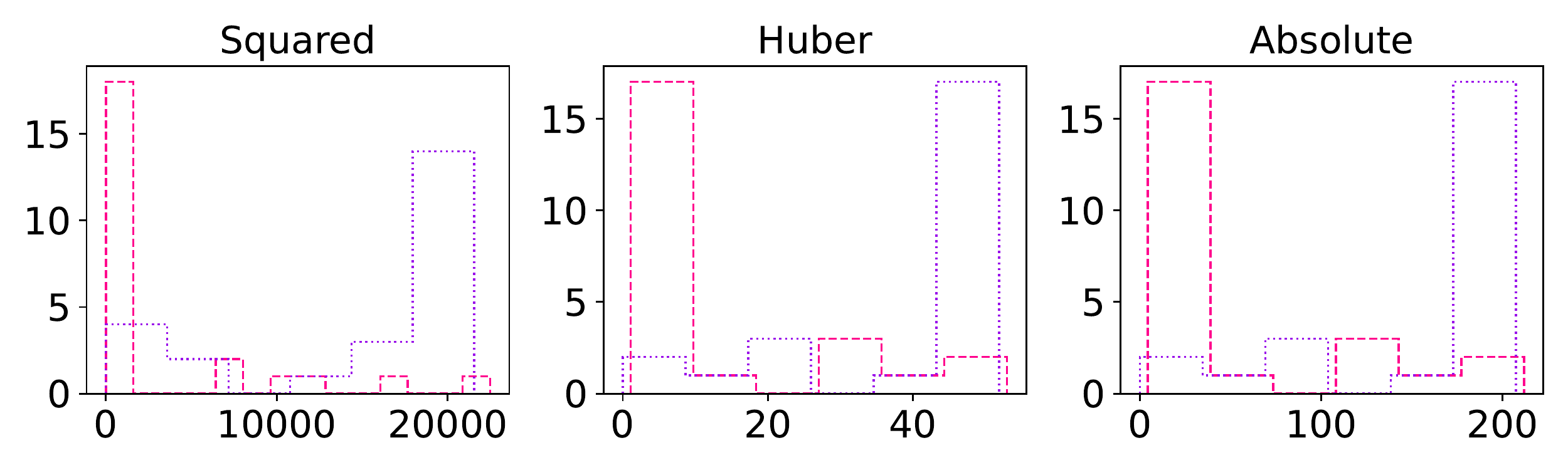}
\caption{Here the left-most plot is the Belgian phone call dataset for use in a one-dimensional regression task, with two linear regression candidates. The remaining three plots show histograms of the loss distributions incurred by each of these candidates using three loss functions commonly used in regression.}
\label{fig:outliers_phones_loss_tails}
\end{figure}

\begin{figure}[t]
\centering
\includegraphics[width=1.0\textwidth]{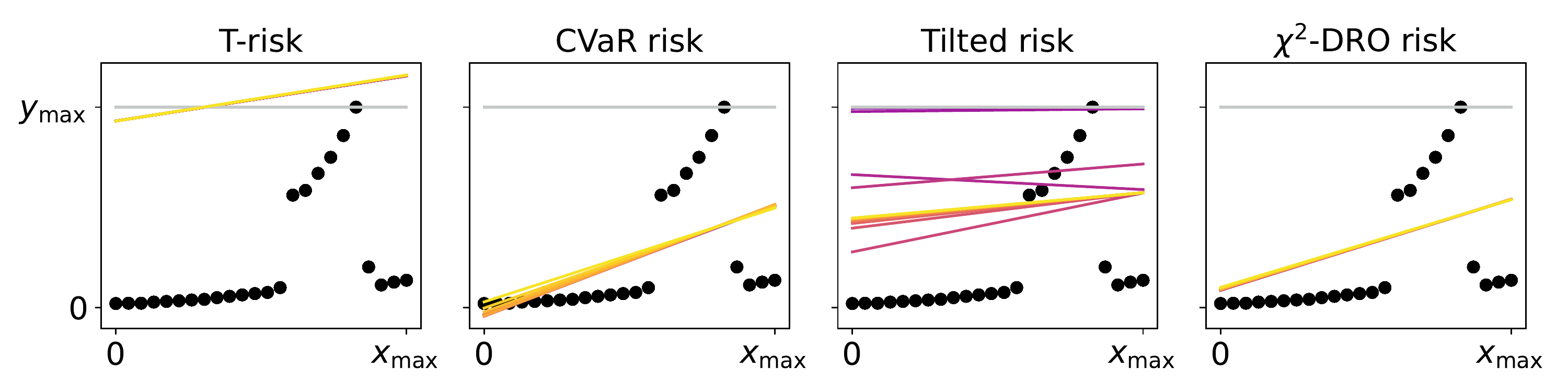}\\
\hfill\includegraphics[width=0.24\textwidth]{bddrob/outliers_phones_traj_all_cbar_triskGeneral}\includegraphics[width=0.24\textwidth]{bddrob/outliers_phones_traj_all_cbar_cvar}\includegraphics[width=0.24\textwidth]{bddrob/outliers_phones_traj_all_cbar_entropic}\includegraphics[width=0.24\textwidth]{bddrob/outliers_phones_traj_all_cbar_dro}
\caption{Same procedure as in Figure \ref{fig:tail_control}, but this time started from a different initial value.}
\label{fig:tail_control_all_updown}
\end{figure}

\clearpage

\begin{figure}[t]
\centering
\includegraphics[width=0.5\textwidth]{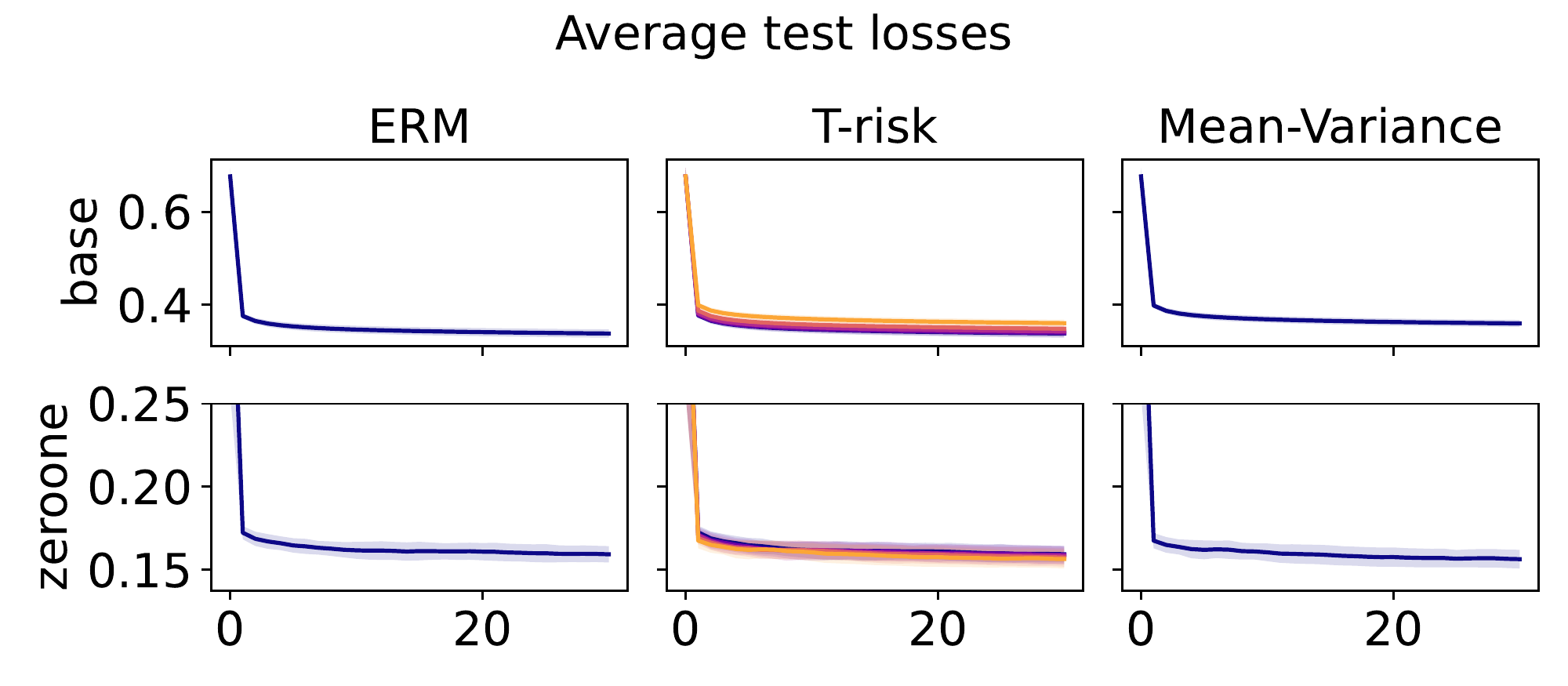}\includegraphics[width=0.5\textwidth]{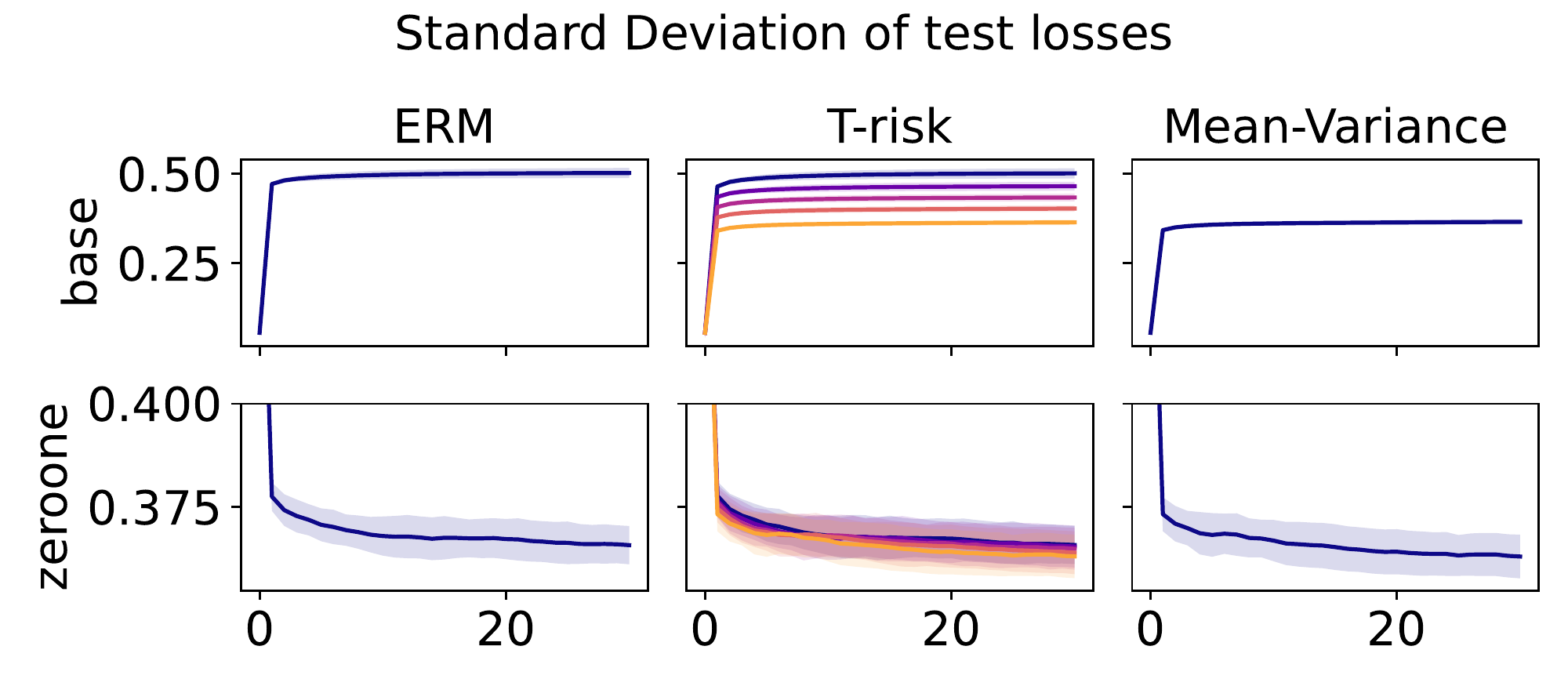}\\
\includegraphics[width=1.0\textwidth]{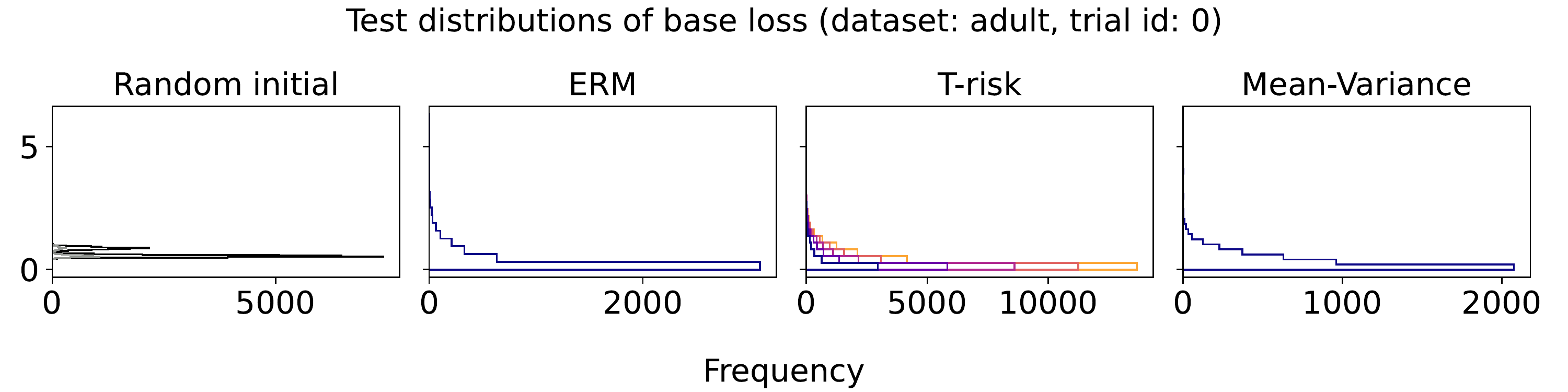}
\caption{Using T-risk to interpolate between test loss distributions (dataset: \texttt{adult}).}
\label{fig:real_adult}
\end{figure}

\begin{figure}[t]
\centering
\includegraphics[width=0.5\textwidth]{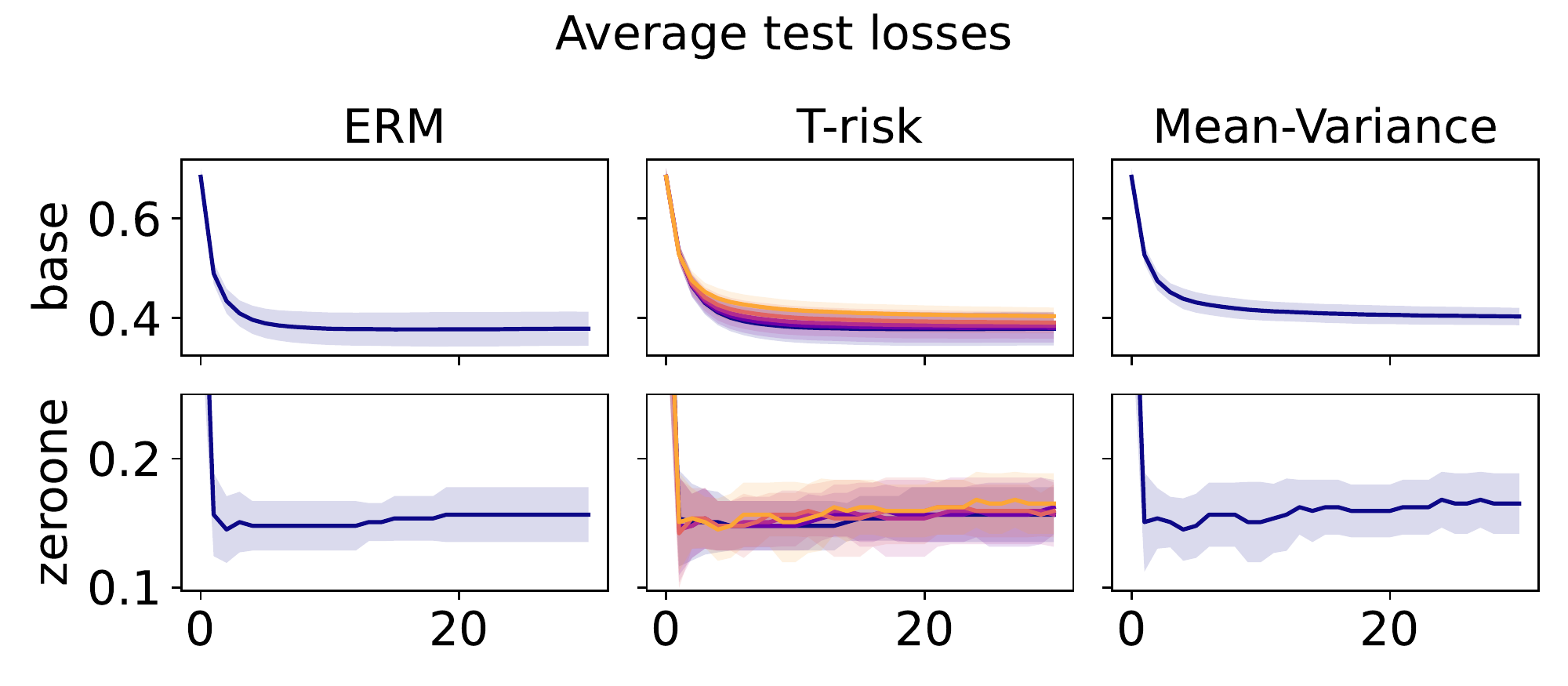}\includegraphics[width=0.5\textwidth]{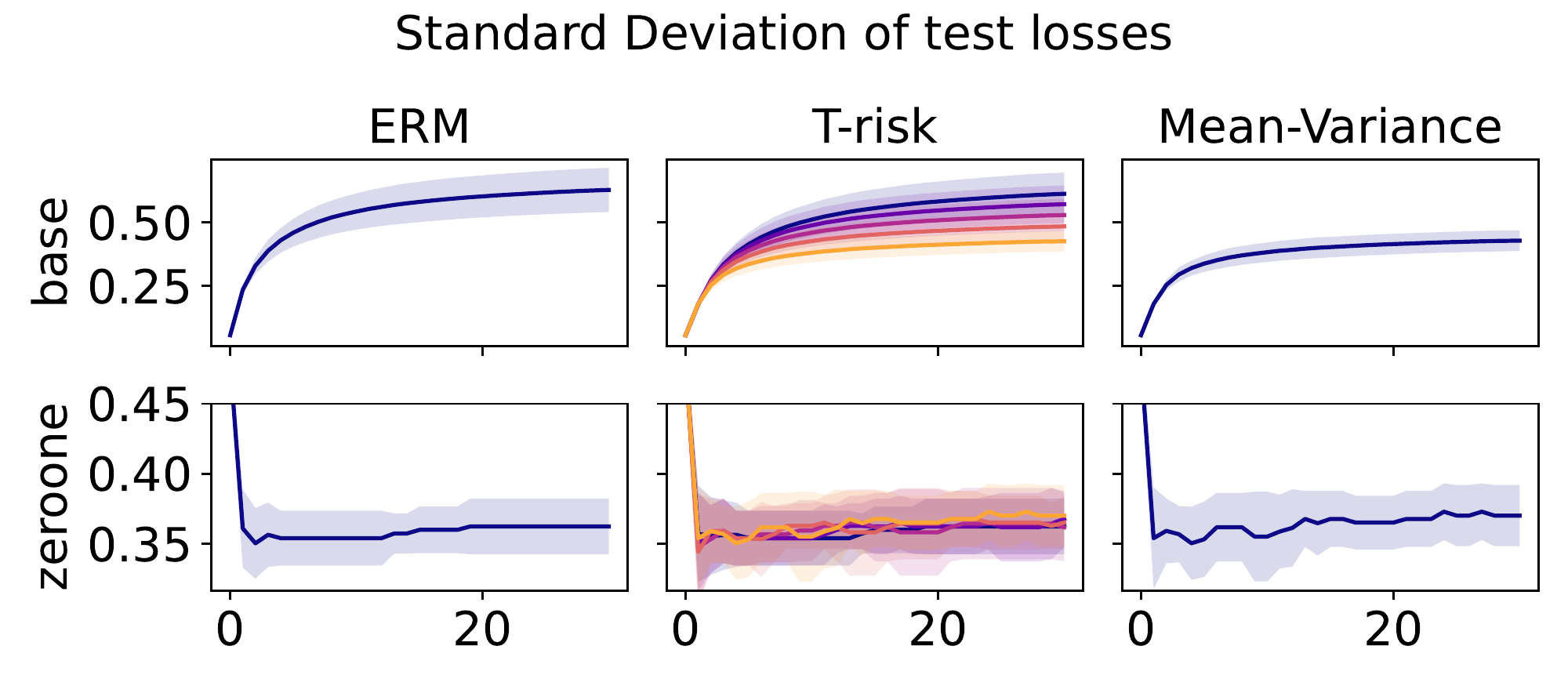}\\
\includegraphics[width=1.0\textwidth]{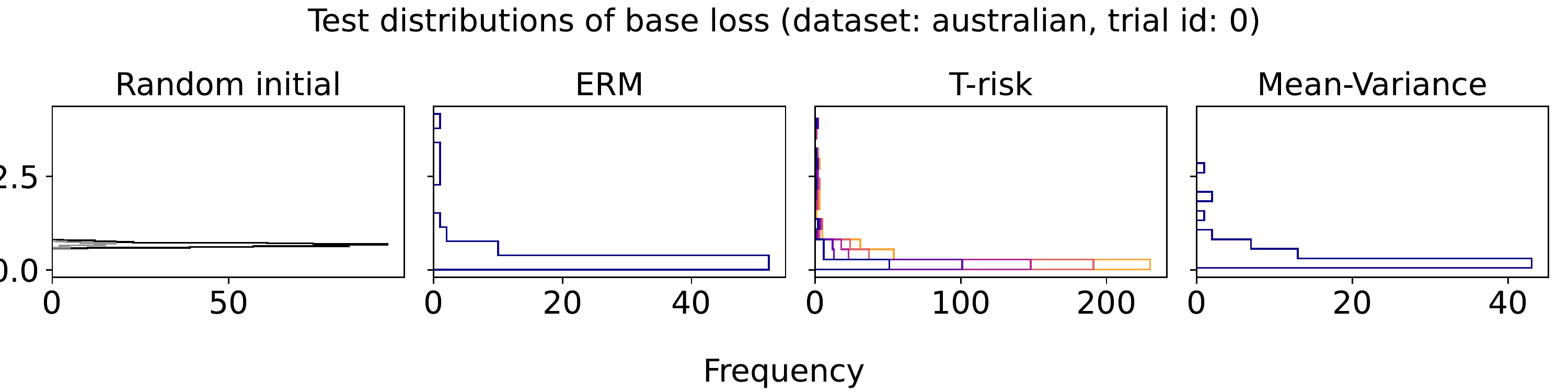}
\caption{Using T-risk to interpolate between test loss distributions (dataset: \texttt{australian}).}
\label{fig:real_australian}
\end{figure}

\begin{figure}[t]
\centering
\includegraphics[width=0.5\textwidth]{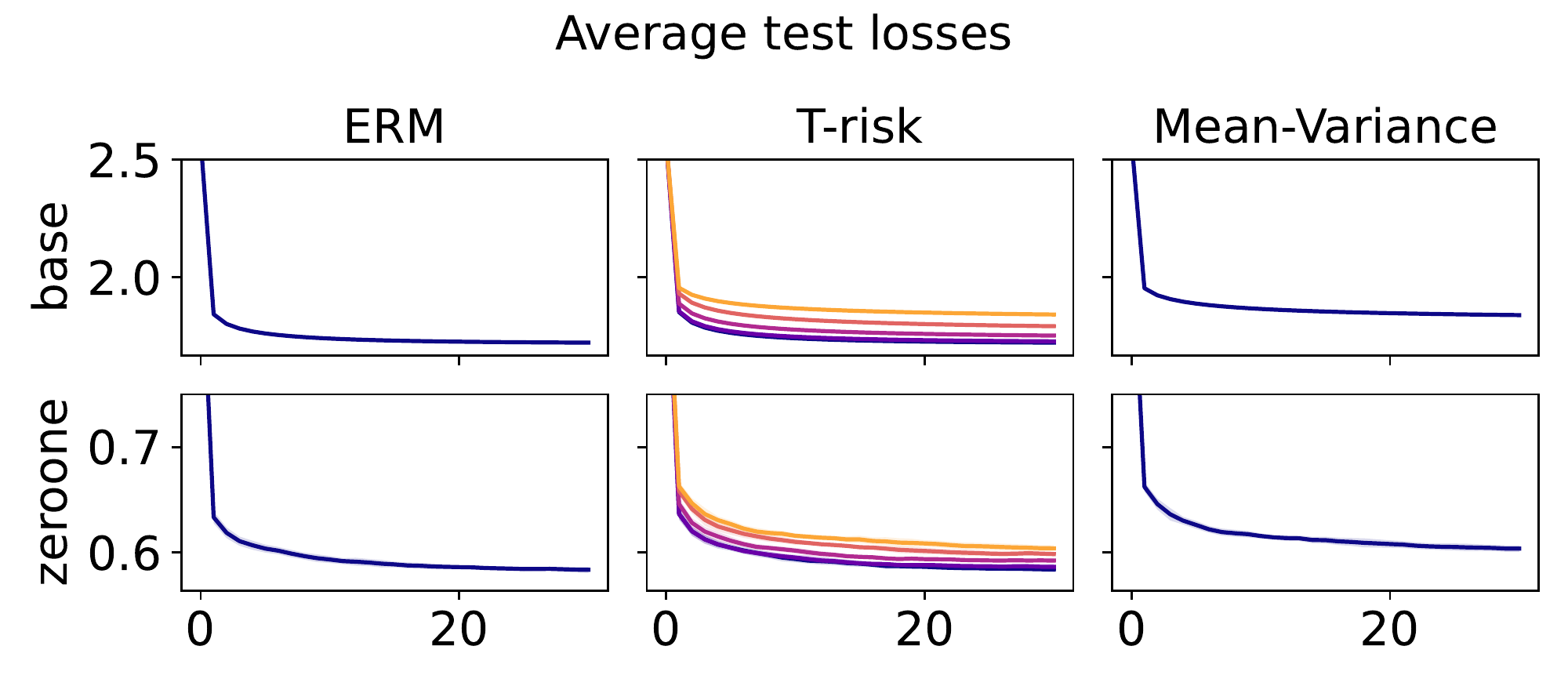}\includegraphics[width=0.5\textwidth]{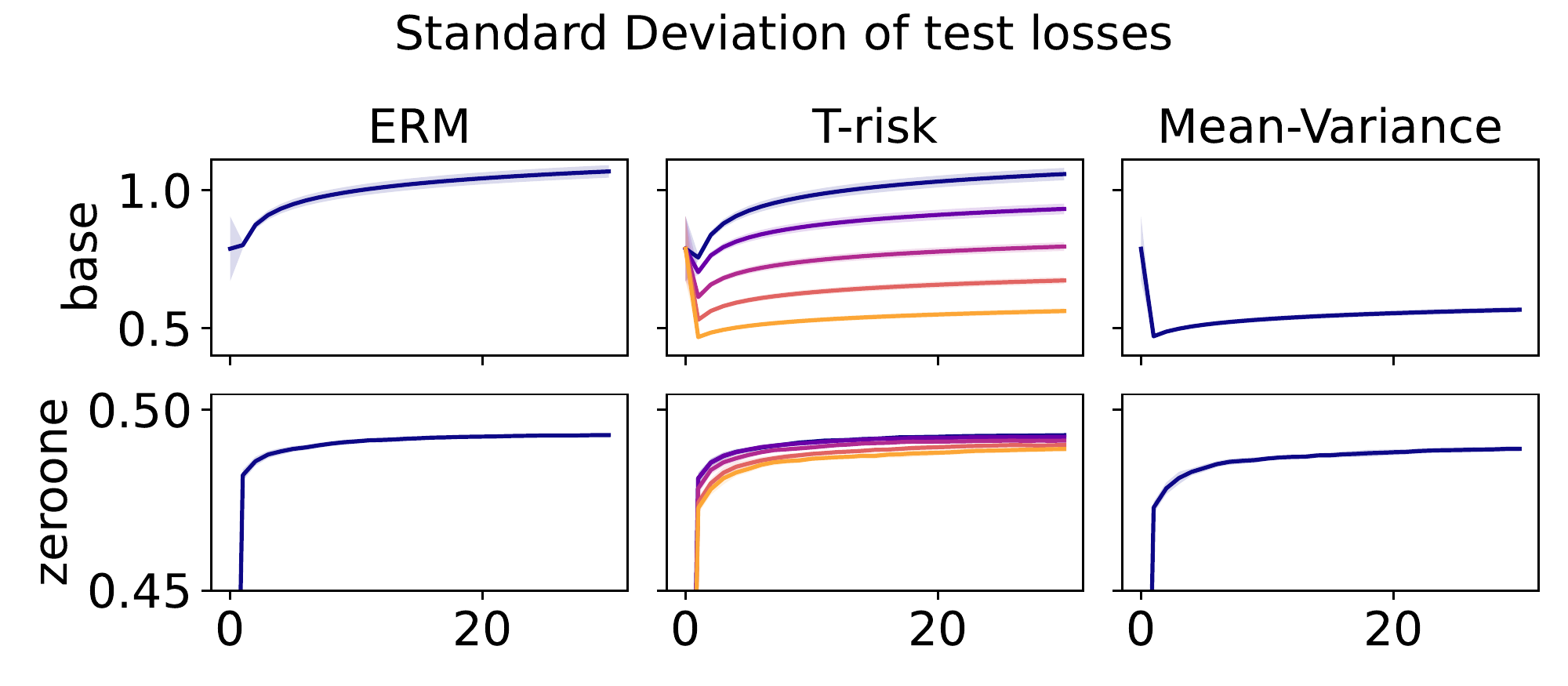}
\caption{Using T-risk to interpolate between test loss distributions (dataset: \texttt{cifar10}).}
\label{fig:real_cifar10}
\end{figure}

\begin{figure}[t]
\centering
\includegraphics[width=0.5\textwidth]{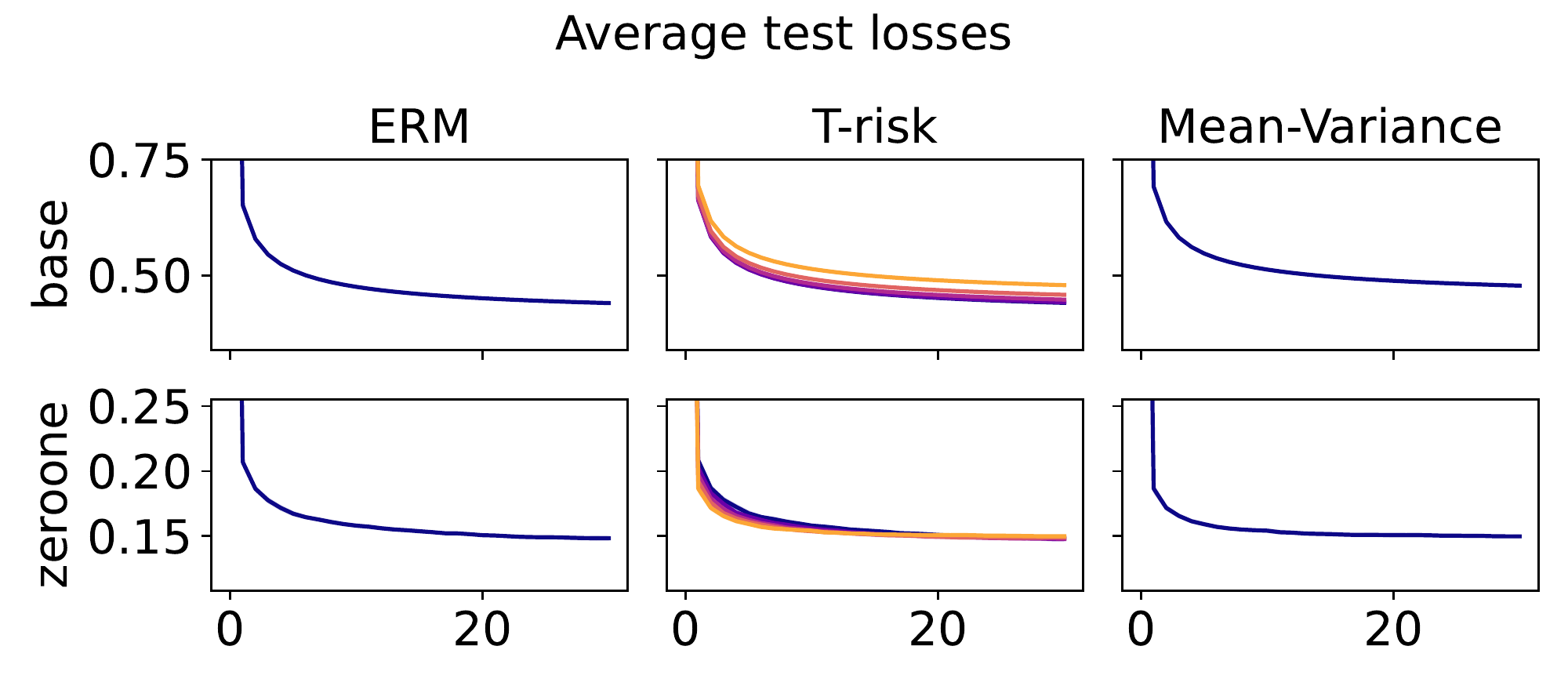}\includegraphics[width=0.5\textwidth]{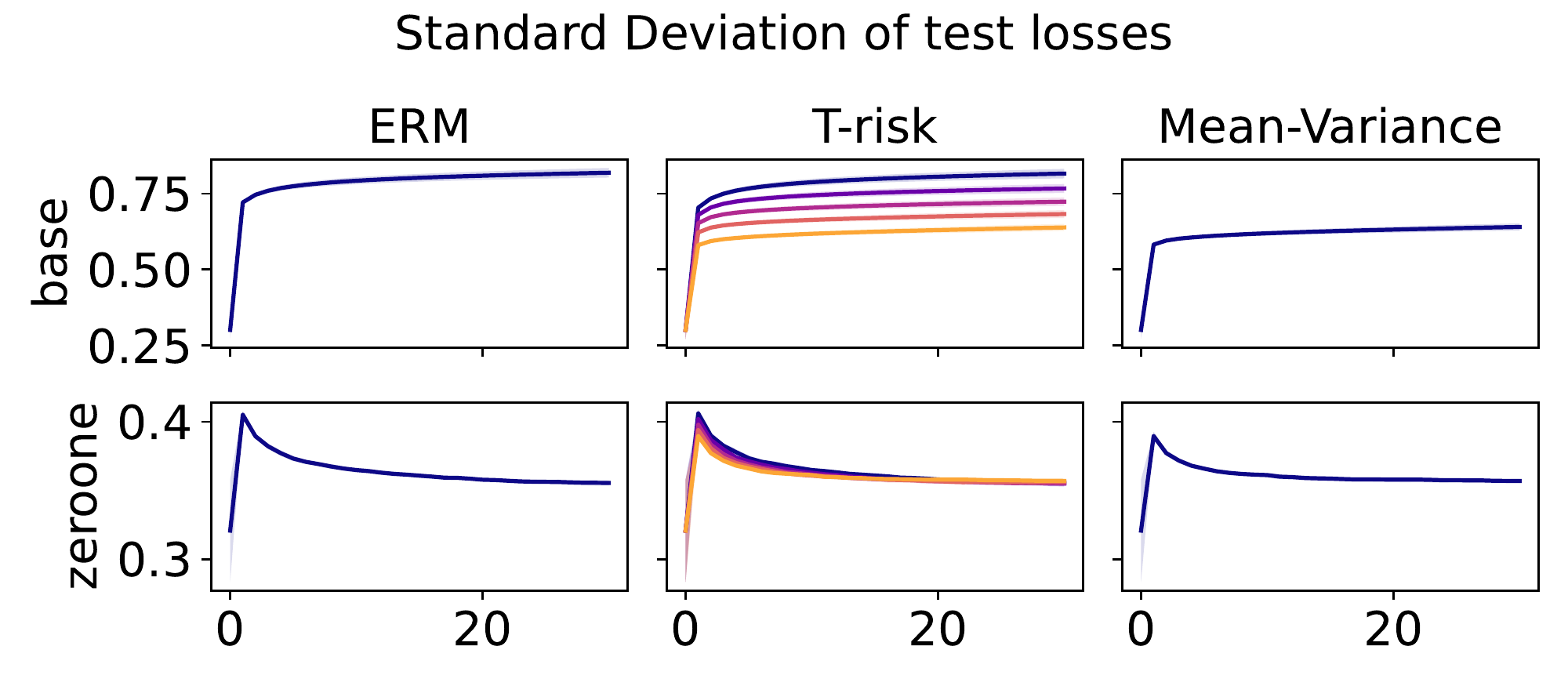}\\
\includegraphics[width=1.0\textwidth]{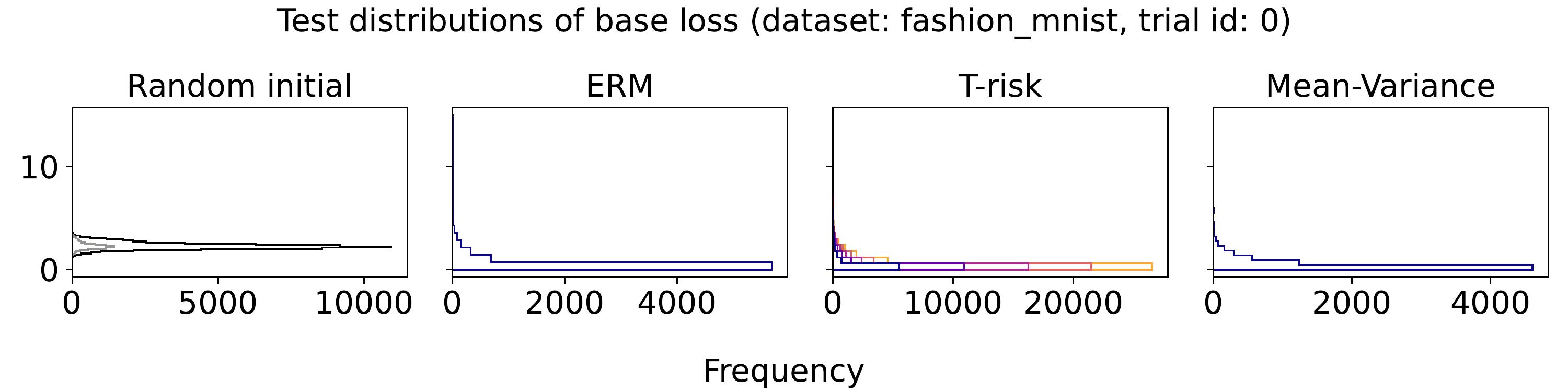}
\caption{Using T-risk to interpolate between test loss distributions (dataset: \texttt{fashion\_mnist}).}
\label{fig:real_fashion_mnist}
\end{figure}

\begin{figure}[t]
\centering
\includegraphics[width=0.5\textwidth]{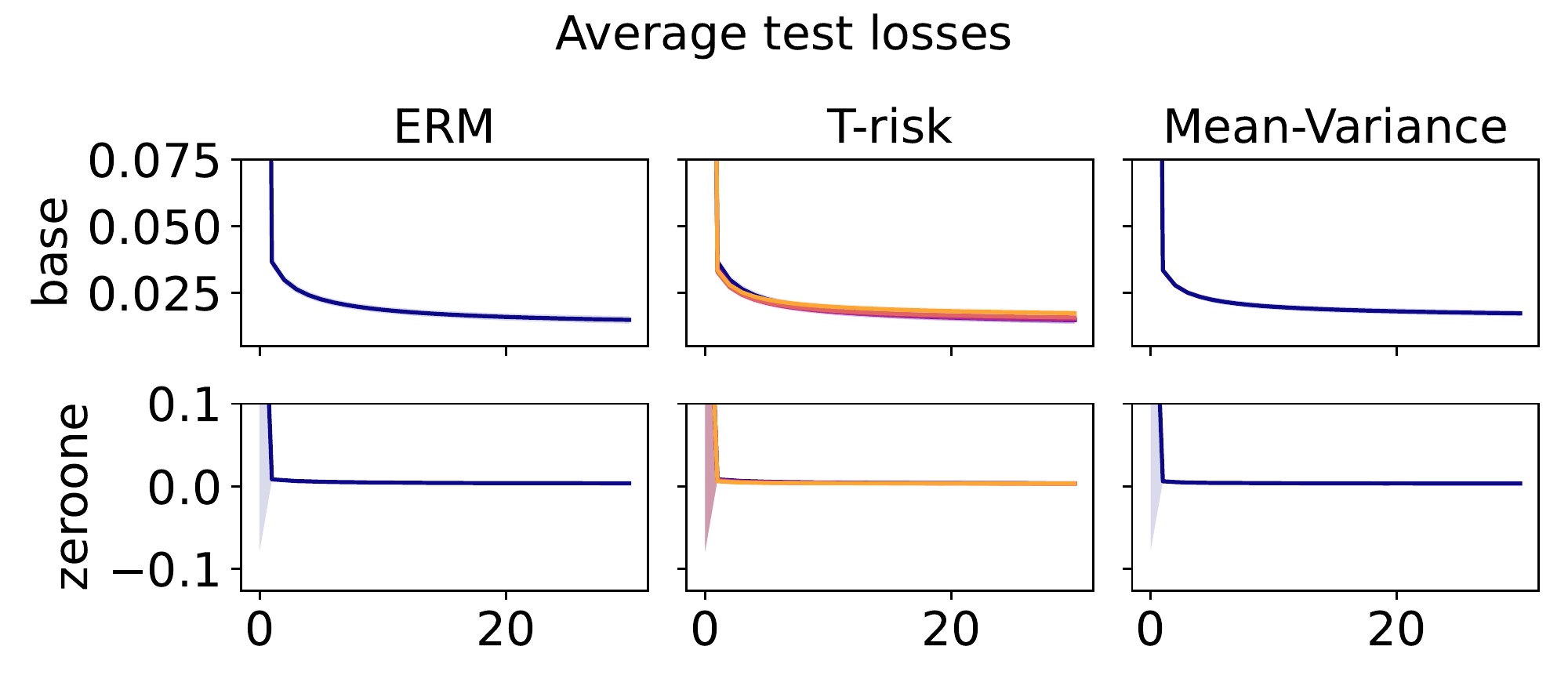}\includegraphics[width=0.5\textwidth]{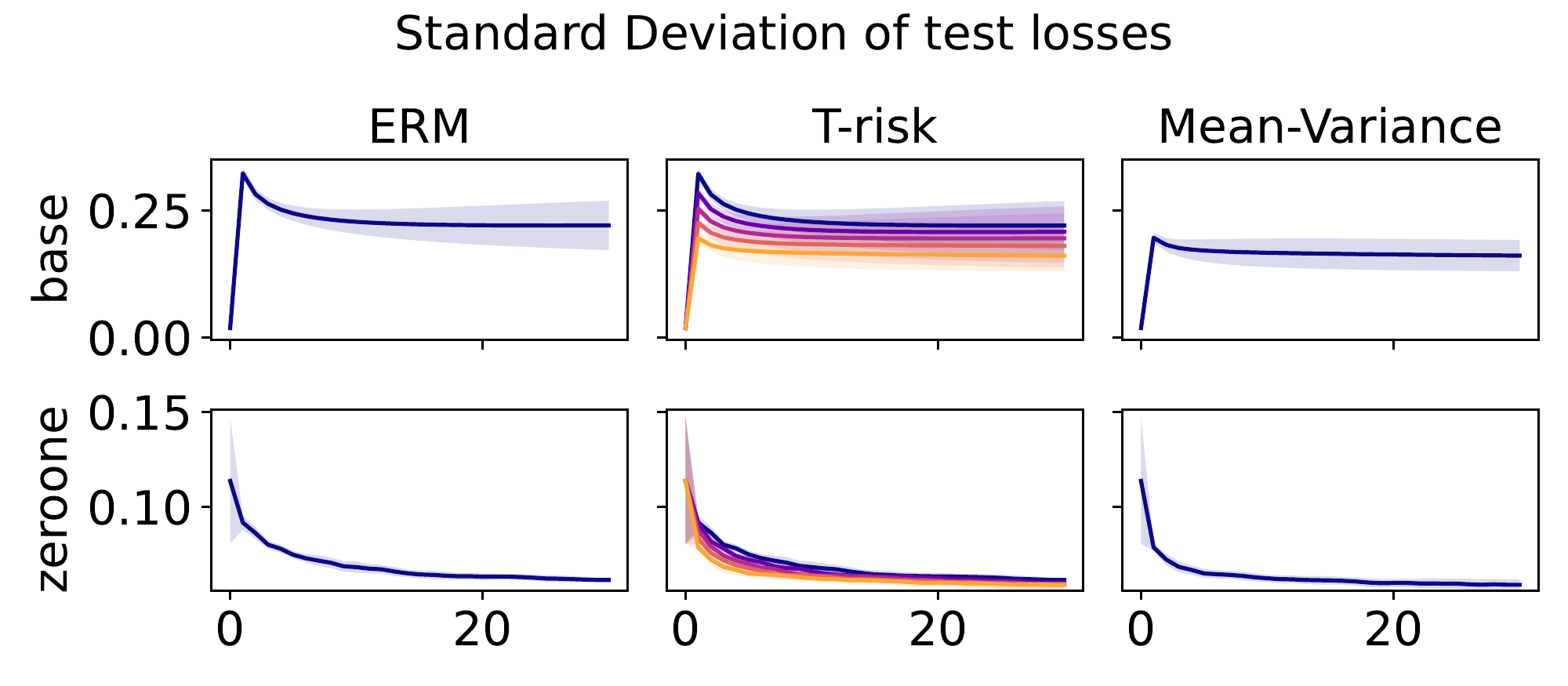}\\
\includegraphics[width=1.0\textwidth]{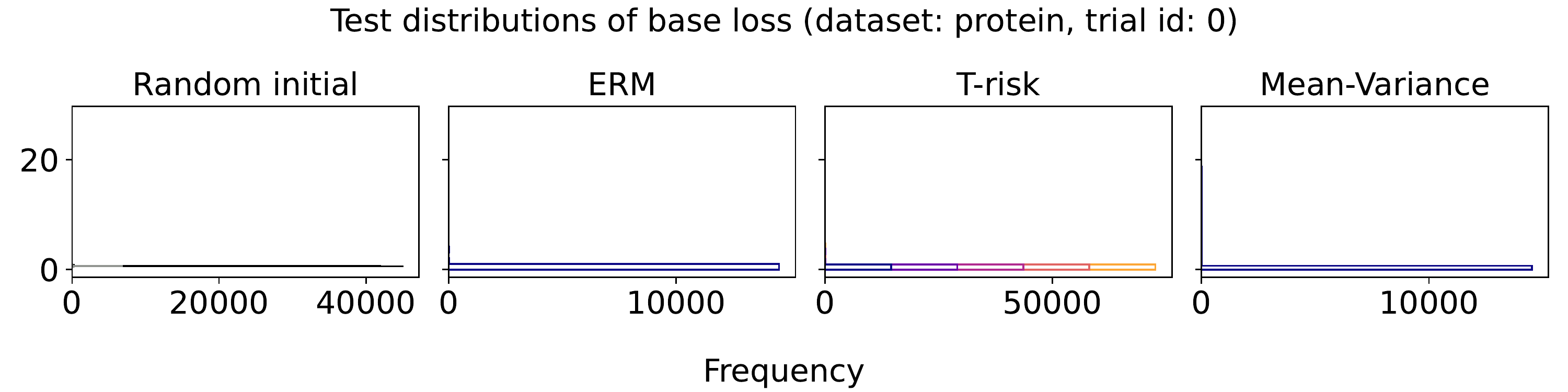}
\caption{Using T-risk to interpolate between test loss distributions (dataset: \texttt{protein}).}
\label{fig:real_protein}
\end{figure}

\clearpage

\section{Detailed proofs}\label{sec:proofs}

\subsection{Proofs of results in the main text}\label{sec:proofs_main}

\begin{proof}[Proof of Lemma \ref{lem:basic_barron}]
In this proof, without further mention, we will make regular use of the following two helper results: Lemma \ref{lem:meanvalue_general} (bounded gradient implies Lipschitz continuity) and Lemma \ref{lem:convexity_hessian} (positive definite Hessian implies convexity). For reference, the first and second derivatives of $\rho_{\sigma}$ are given in \S{\ref{sec:proofs_barron_derivs}}. We take up each $\alpha$ setting one at a time.

First, the case of $\alpha = 2$. For this case, clearly $\dv{\rho}_{\sigma}$ is unbounded, and thus $\rho_{\sigma}$ is not (globally) Lipschitz on $\RR$. On the other hand, since $\ddv{\rho}_{\sigma}(x) = 1/\sigma^{2}$, we have that $\dv{\rho}_{\sigma}$ is $\smooth$-Lipschitz with $\smooth = (1/\sigma^{2})$.

Next, the case of $\alpha = 0$. For any fixed $\sigma > 0$, in both the limits $x \to 0$ and $\lvert x \rvert \to \infty$, we have $\dv{\rho}_{\sigma}(x) \to 0$. Maximum and minimum values are achieved when $\ddv{\rho}_{\sigma}(x) = 0$, and this occurs if and only if $x^{2} = 2\sigma^{2}$. It follows from direct computation that $\dv{\rho}_{\sigma}(\pm \sqrt{2}\sigma) = \pm 1/(\sqrt{2}\sigma)$, and thus $\rho_{\sigma}$ is $\smooth$-Lipschitz with $\smooth = 1/(\sqrt{2}\sigma)$. Next, recalling that $\ddv{\rho}_{\sigma}$ takes the form
\begin{align*}
\ddv{\rho}_{\sigma}(x) = \frac{2}{x^{2}+2\sigma^{2}}\left(1 - \frac{2x^{2}}{x^{2}+2\sigma^{2}}\right),
\end{align*}
we see that this is a product of two factors, one taking values in $(0,1/\sigma^{2})$, and one taking values in $(-1,1)$. The absolute value of both of these factors is maximized when $x=0$, and so $\lvert \ddv{\rho}_{\sigma}(x) \rvert \leq \ddv{\rho}_{\sigma}(0) = 1/\sigma^{2}$, meaning that $\dv{\rho}_{\sigma}$ is $\smooth$-Lipschitz with $\smooth = 1/\sigma^{2}$. Finally, regarding convexity, we have that $\ddv{\rho}_{\sigma}(x) \geq 0$ if and only if $\lvert x \rvert \leq \sqrt{2}\sigma$.

Next, the case of $\alpha = -\infty$. For any fixed $\sigma > 0$, we have $\dv{\rho}_{\sigma}(x) \to 0$ in both the limits $x \to 0$ and $\lvert x \rvert \to \infty$. Furthermore, it is immediate that $\ddv{\rho}_{\sigma}(x) = 0$ at the points $x = \pm \sigma$. Evaluating $\dv{\rho}_{\sigma}$ at these stationary points we have $\dv{\rho}_{\sigma}(\pm \sigma) = \pm(1/\sigma)\exp(-1/2)$, and thus $\rho_{\sigma}$ is $\smooth$-Lipschitz with $\smooth = (1/\sigma)\exp(-1/2)$. Regarding bounds on $\ddv{\rho}_{\sigma}$, first note that $\ddv{\rho}_{\sigma}(x) \to 0$ as $\lvert x \rvert \to \infty$, and $\ddv{\rho}_{\sigma}(0) = 1/\sigma^{2}$. Then to identify stationary points, note that
\begin{align*}
\rho_{\sigma}^{\prime\prime\prime}(x) = \frac{1}{\sigma^{2}} \exp\left(-\frac{1}{2}\left(\frac{x}{\sigma}\right)^{2}\right)\left[ \frac{x}{\sigma^{2}}\left(\frac{x^{2}}{\sigma^{2}}-1\right) - \frac{2x}{\sigma^{2}} \right]
\end{align*}
and thus $\rho_{\sigma}^{\prime\prime\prime}(x) = 0$ if and only if $(x/\sigma)^{2}-1 = 2$, i.e., the stationary points are $x = \pm \sqrt{3}\sigma$, both of which yield the same value, namely $\ddv{\rho}_{\sigma}(\pm \sqrt{3}\sigma) = -(2/\sigma^{2})\exp(-3/2)$. Since $2\exp(-3/2) \approx 0.45 < 1$, we conclude that $\dv{\rho}_{\sigma}$ is $\smooth$-Lipschitz with $\smooth = 1/\sigma^{2}$. Finally, since $\ddv{\rho}_{\sigma}(x) \geq 0$ if and only if $\lvert x \rvert \leq \sigma$, this specifies the region on which $\rho_{\sigma}$ is convex.

Finally, all that remains is the general case of $-\infty < \alpha < 2$ where $\alpha \neq 0$. Note that in order for $\ddv{\rho}_{\sigma}(x) = 0$ to hold, we require
\begin{align*}
\frac{2(x/\sigma)^{2}}{1 + (x/\sigma)^{2}/\abs{\alpha-2}} = \frac{\abs{\alpha-2}}{1-(\alpha/2)},
\end{align*}
which via some basic algebra is equivalent to
\begin{align*}
\left(\frac{x}{\sigma}\right)^{2} = \frac{\abs{\alpha-2}}{1-\alpha}.
\end{align*}
Clearly, this is only possible when $\alpha < 1$, so we consider this sub-case first. This implies stationary points $\pm x^{\ast} \defeq \pm \sigma\sqrt{\abs{\alpha-2}/(1-\alpha)}$, for which we have
\begin{align*}
\dv{\rho}_{\sigma}(\pm x^{\ast}) = \pm \frac{1}{\sigma} \sqrt{\frac{\abs{\alpha-2}}{1-\alpha}}\left(\frac{2-\alpha}{1-\alpha}\right)^{(\alpha/2)-1} = \pm \frac{1}{\sigma}\left(\sqrt{\frac{\abs{\alpha-2}}{1-\alpha}}\right)^{\alpha-1} = \pm \frac{1}{\sigma}\left(\sqrt{\frac{1-\alpha}{\abs{\alpha-2}}}\right)^{1-\alpha}.
\end{align*}
Since $\dv{\rho}_{\sigma}(x) \to 0$ in both the limits $x \to 0$ and $\lvert x \rvert \to \infty$, we have obtained a maximum value for $\dv{\rho}_{\sigma}$ at $x^{\ast}$, thus implying for the case of $\alpha < 1$ that $\rho_{\sigma}$ is $\smooth$-Lipschitz, with a coefficient of $\smooth = (1/\sigma)(\sqrt{(1-\alpha)/\abs{\alpha-2}})^{1-\alpha}$. For the case of $\alpha = 1$, direct inspection shows
\begin{align*}
\lvert \dv{\rho}_{\sigma}(x) \rvert = \frac{\abs{x/\sigma^{2}}}{\sqrt{1 + (x/\sigma)^{2}}} = \frac{1}{\sigma^{2}\sqrt{(1/x^{2})+(1/\sigma^{2})}},
\end{align*}
a value which is maximized in the limit $\lvert x \rvert \to \infty$. As such, for $\alpha = 1$, we have that $\rho_{\sigma}$ is $\smooth$-Lipschitz with $\smooth = 1/\sigma$. For the case of $1 < \alpha < 2$, $\dv{\rho}_{\sigma}$ is unbounded. To see this, note that for $x > 0$ we have
\begin{align*}
\dv{\rho}_{\sigma}(x) = \frac{1}{\sigma^{2}} \frac{\left(1 + (x/\sigma)^{2}/\abs{\alpha-2}\right)}{\left((1/x) + x/(\sigma^{2}\abs{\alpha-2})\right)} \geq \left(\frac{1}{\sigma^{2+\alpha}\sqrt{\abs{\alpha-2}^{\alpha}}}\right) \frac{x^{\alpha}}{\left((1/x) + x/(\sigma^{2}\abs{\alpha-2})\right)},
\end{align*}
and since $\alpha > 1$, sending $x \to \infty$ clearly implies $\dv{\rho}_{\sigma}(x) \to \infty$, and this means that $\rho_{\sigma}$ cannot be Lipschitz on $\RR$ when $\alpha > 1$. As for bounds on $\ddv{\rho}_{\sigma}$, recall that
\begin{align*}
\ddv{\rho}_{\sigma}(x) = \frac{1}{\sigma^{2}}\underbrace{\left(1 + \frac{(x/\sigma)^{2}}{\abs{\alpha-2}} \right)^{(\alpha/2)-1}}_{A(x)} \left(1 - \underbrace{\frac{1-(\alpha/2)}{\abs{\alpha-2}} \frac{2(x/\sigma)^{2}}{1 + (x/\sigma)^{2}/\abs{\alpha-2}}}_{B(x)}\right)
\end{align*}
where we have introduced the labels $A(x)$ and $B(x)$ just as convenient notation. Fixing any $\sigma > 0$, first note that since $\alpha < 2$, we have $(\alpha/2)-1 < 0$ and thus $0 \leq A(x) \leq 1$. Next, direct inspection shows $0 \leq B(x) \leq 2(1-(\alpha/2))$. These two facts immediately imply an upper bound $\ddv{\rho}_{\sigma}(x) \leq 1/\sigma^{2}$ and a lower bound $\ddv{\rho}_{\sigma}(x) \geq -(1-\alpha)/\sigma^{2}$, both of which hold for any $\alpha < 2$. Furthermore, for the case of $1 \leq \alpha < 2$, we thus have $0 \leq \ddv{\rho}_{\sigma}(x) \leq 1/\sigma^{2}$. When $\alpha < 1$ however, $\ddv{\rho}_{\sigma}$ can be negative. To get matching lower bounds requires $A(x)(1-B(x)) \geq -1$, or $A(x)(B(x)-1) \leq 1$. To study conditions under which this holds, first note that $B(x)$ can be re-written as
\begin{align*}
B(x) = \left(\frac{2-\alpha}{\abs{\alpha-2}}\right)\frac{(x/\sigma)^{2}}{1+\frac{(x/\sigma)^{2}}{\abs{\alpha-2}}} = \frac{(x/\sigma)^{2}}{1+\frac{(x/\sigma)^{2}}{\abs{\alpha-2}}},
\end{align*}
and thus we have
\begin{align}\label{eqn:basic_barron_0}
A(x)(B(x) - 1) = \frac{(x/\sigma)^{2}}{\left(1+\frac{(x/\sigma)^{2}}{\abs{\alpha-2}}\right)^{2-(\alpha/2)}} - \frac{1}{\left(1+\frac{(x/\sigma)^{2}}{\abs{\alpha-2}}\right)^{1-(\alpha/2)}}.
\end{align}
To get a more convenient upper bound on this, observe that $(1+x)^{1-(\alpha/2)} \leq (1+x)^{2-(\alpha/2)}$ for any $x \geq 0$ and $-\infty \leq \alpha \leq 2$. It follows immediately that
\begin{align}\label{eqn:basic_barron_1}
A(x)(B(x) - 1) \leq \frac{(x/\sigma)^{2}-1}{\left(1+\frac{(x/\sigma)^{2}}{\abs{\alpha-2}}\right)^{2-(\alpha/2)}}.
\end{align}
To get the right-hand side of (\ref{eqn:basic_barron_1}) to be no greater than $1$ is equivalent to
\begin{align}\label{eqn:basic_barron_2}
(x/\sigma)^{2}-1 \leq \left(1+\frac{(x/\sigma)^{2}}{\abs{\alpha-2}}\right)^{2-(\alpha/2)}.
\end{align}
For the case of $0 \leq \alpha < 1$, note that $1 \leq \abs{\alpha-2} = 2-\alpha < 2-(\alpha/2)$, and using the helper inequality (\ref{eqn:helper_exponential_1}), we have
\begin{align*}
\left(1+\frac{(x/\sigma)^{2}}{\abs{\alpha-2}}\right)^{2-(\alpha/2)} \geq \left(1+\frac{(x/\sigma)^{2}}{\abs{\alpha-2}}\right)^{\abs{\alpha-2}} \geq 1+(x/\sigma)^{2} > (x/\sigma)^{2}-1,
\end{align*}
which implies (\ref{eqn:basic_barron_2}) for $0 \leq \alpha < 1$. All that remains is the case of $-\infty < \alpha < 0$, which requires a bit more care. Returning to the exact form of $A(x)(B(x) - 1)$ given in (\ref{eqn:basic_barron_0}), note that the inequality
\begin{align}\label{eqn:basic_barron_3}
(x/\sigma)^{2} - \left(1 + \frac{(x/\sigma)^{2}}{\abs{\alpha-2}}\right) \leq \left(1+\frac{(x/\sigma)^{2}}{\abs{\alpha-2}}\right)^{2-(\alpha/2)}
\end{align}
is equivalent to the desired property, i.e., $A(x)(B(x) - 1) \leq 1 \iff \text{(\ref{eqn:basic_barron_3})}$. Using Bernoulli's inequality (\ref{eqn:helper_bernoulli}), we can bound the right-hand side of (\ref{eqn:basic_barron_3}) as
\begin{align*}
\left(1+\frac{(x/\sigma)^{2}}{\abs{\alpha-2}}\right)^{2-(\alpha/2)} \geq 1 + \left(\frac{2-(\alpha/2)}{\abs{\alpha-2}}\right)(x/\sigma)^{2}.
\end{align*}
Subtracting the left-hand side of (\ref{eqn:basic_barron_3}) from the right-hand side of the preceding inequality, we obtain
\begin{align}
\nonumber
\left(1+\frac{(x/\sigma)^{2}}{\abs{\alpha-2}}\right)^{2-(\alpha/2)} - \left[(x/\sigma)^{2} - 1 - \frac{(x/\sigma)^{2}}{\abs{\alpha-2}}\right] & \geq 2 + \left(\frac{2-(\alpha/2)}{\abs{\alpha-2}} - 1 + \frac{1}{\abs{\alpha-2}}\right)(x/\sigma)^{2}\\
\label{eqn:basic_barron_4}
& = 2 + \left(\frac{1-(\abs{\alpha}/2)}{2+\abs{\alpha}}\right)(x/\sigma)^{2},
\end{align}
where the second step uses the fact that for $\alpha < 0$, we can write $\abs{\alpha-2} = 2 + \abs{\alpha}$ and $2-(\alpha/2) = 2 + (\abs{\alpha}/2)$. Note that the right-hand side of (\ref{eqn:basic_barron_4}) is non-negative for all $x \in \RR$ whenever $-2 \leq \alpha < 0$, which via (\ref{eqn:basic_barron_3}) tells us that $A(x)(B(x) - 1) \leq 1$ indeed holds in this case as well. For the case of $-\infty < \alpha < -2$, note that showing (\ref{eqn:basic_barron_3}) holds is equivalent to showing $f_{\alpha}(x) \geq 0$ for all $x \geq 0$, where for convenience we define the polynomial
\begin{align*}
f_{\alpha}(x) \defeq 1 + \left(\frac{1}{2+\abs{\alpha}} - 1\right)x + \left(1 + \frac{x}{2+\abs{\alpha}}\right)^{2+(\abs{\alpha}/2)}.
\end{align*}
The first derivative is
\begin{align*}
\dv{f}_{\alpha}(x) = \frac{2+(\abs{\alpha}/2)}{2+\abs{\alpha}}\left(1 + \frac{x}{2+\abs{\alpha}}\right)^{1+(\abs{\alpha}/2)} + \frac{1}{2+\abs{\alpha}} - 1,
\end{align*}
and with this form in hand, solving for $f_{\alpha}(x)=0$, it is straightforward to confirm that $x_{\alpha}^{\ast}$ given below is a stationary point:
\begin{align*}
x_{\alpha}^{\ast} \defeq \left(2 + \abs{\alpha}\right)\left[ \left(\frac{1+\abs{\alpha}}{2+(\abs{\alpha}/2)}\right)^{\frac{1}{1+(\abs{\alpha}/2)}} - 1 \right].
\end{align*}
Furthermore, it is clear that $\ddv{f}_{\alpha} \geq 0$, implying that $f_{\alpha}$ is convex, and that $x_{\alpha}^{\ast}$ is a minimum. As such, the minimum value taken by $f_{\alpha}$ on $\RR_{+}$ is
\begin{align*}
f_{\alpha}(x_{\alpha}^{\ast}) & = \left(\frac{1+\abs{\alpha}}{2+(\abs{\alpha}/2)}\right)^{\frac{1}{1+(\abs{\alpha}/2)}} - \left(2 + \abs{\alpha}\right)\left[ \left(\frac{1+\abs{\alpha}}{2+(\abs{\alpha}/2)}\right)^{\frac{1}{1+(\abs{\alpha}/2)}} - 1 \right] + \left(\frac{1+\abs{\alpha}}{2+(\abs{\alpha}/2)}\right)^{\frac{2+(\abs{\alpha}/2)}{1+(\abs{\alpha}/2)}}\\
& = \left(2 + \abs{\alpha}\right) + \left(\frac{1+\abs{\alpha}}{2+(\abs{\alpha}/2)}\right)^{\frac{2+(\abs{\alpha}/2)}{1+(\abs{\alpha}/2)}} - \left(1+\abs{\alpha}\right)\left(\frac{1+\abs{\alpha}}{2+(\abs{\alpha}/2)}\right)^{\frac{1}{1+(\abs{\alpha}/2)}}\\
& = \left(2 + \abs{\alpha}\right) - \left(1 + \abs{\alpha}\right)\left(\frac{1+\abs{\alpha}}{2+(\abs{\alpha}/2)}\right)^{\frac{1}{1+(\abs{\alpha}/2)}}\left[ 1 - \frac{1}{2+(\abs{\alpha}/2)} \right]\\
& = 1 + \left(1 + \abs{\alpha}\right)\left[ 1 - \left(\frac{1+\abs{\alpha}}{2+(\abs{\alpha}/2)}\right)^{\frac{1}{1+(\abs{\alpha}/2)}}\left[ 1 - \frac{1}{2+(\abs{\alpha}/2)} \right] \right].
\end{align*}
We require $f_{\alpha}(x_{\alpha}^{\ast}) \geq 0$ for all $-\infty < \alpha < -2$. From the preceding equalities, note that a simple sufficient condition for $f_{\alpha}(x_{\alpha}^{\ast}) \geq 1$ is
\begin{align*}
\left(\frac{1+\abs{\alpha}}{2+(\abs{\alpha}/2)}\right)^{\frac{1}{1+(\abs{\alpha}/2)}}\left[ 1 - \frac{1}{2+(\abs{\alpha}/2)} \right] \leq 1
\end{align*}
or equivalently
\begin{align}\label{eqn:basic_barron_5}
\left(1 - \frac{1}{2+(\abs{\alpha}/2)}\right)^{1+(\abs{\alpha}/2)} \leq \left(\frac{2+(\abs{\alpha}/2)}{1+\abs{\alpha}}\right).
\end{align}
Applying the helper inequality (\ref{eqn:helper_exponential_2}) to the left-hand side of (\ref{eqn:basic_barron_5}), we have
\begin{align*}
\left(1 - \frac{1}{2+(\abs{\alpha}/2)}\right)^{1+(\abs{\alpha}/2)} \leq 1 + \frac{\left(\frac{-(1+(\abs{\alpha}/2))}{2+(\abs{\alpha}/2)}\right)}{\left(1 - \frac{-\abs{\alpha}/2}{2+(\abs{\alpha}/2)}\right)} = 1 - \left(\frac{1+(\abs{\alpha}/2)}{2+\abs{\alpha}}\right) & = \frac{1+(\abs{\alpha}/2)}{2+\abs{\alpha}}\\
& \leq \frac{2+(\abs{\alpha}/2)}{1+\abs{\alpha}}.
\end{align*}
This is precisely the desired inequality (\ref{eqn:basic_barron_5}), implying $f_{\alpha}(x_{\alpha}^{\ast}) \geq 1 > 0$ for all $-\infty < \alpha < -2$, and in fact all real $\alpha < 0$. To summarize, we have $A(x)(B(x)-1) \leq 1$ for all $x \in \RR$, and thus the desired $1/\sigma^{2}$-smoothness result follows, concluding the proof.
\end{proof}

\begin{proof}[Proof of Lemma \ref{lem:dsp_finite}]
Let $\rdv{X}$ denote any $\sigmafield$-measurable random variable. The continuity of $\rho$ implies that the integral $\exx_{\ddist}\rho_{\sigma}(\rdv{X}-\theta)$ exists for any $\sigma > 0$; we just need to prove it is finite.\footnote{This uses the fact that any composition of (Borel) measurable functions is itself measurable \citep[Lem.~1.5.7]{ash2000a}.} Since we are taking $\rho$ from the Barron class (\ref{eqn:barron}), we consider each $\alpha$ setting separately. Starting with $\alpha = 2$, note that
\begin{align*}
\rho_{\sigma}(\rdv{X}-\theta;2) = \frac{1}{2}\left(\frac{\rdv{X}-\theta}{\sigma}\right)^{2}
\end{align*}
and thus $\exx_{\ddist}\rdv{X}^{2} < \infty$ is sufficient and necessary. For $\alpha = 0$, first note that we have
\begin{align*}
\rho_{\sigma}(\rdv{X}-\theta;0) = \log\left(1 + \frac{1}{2}\left(\frac{\rdv{X}-\theta}{\sigma}\right)^{2}\right).
\end{align*}
Let $f_{1}(x) \defeq \log(1+x)$ and $f_{2}(x) \defeq x^{c} / c$, where $0 < c < 1$. Note that $f_{1}(0) = f_{2}(0) = 0$, and furthermore that for any $x > 0$,
\begin{align*}
\dv{f}_{1}(x) = \frac{1}{1+x} < \left(\frac{1}{1+x}\right)^{1-c} < \left(\frac{1}{x}\right)^{1-c} = \dv{f}_{2}(x).
\end{align*}
We may thus conclude that $f_{1}(x) \leq f_{2}(x)$ for all $x \geq 0$, and thus for any $0 < c < 1$ we have
\begin{align*}
\log\left(1 + \frac{1}{2}\left(\frac{\rdv{X}-\theta}{\sigma}\right)^{2}\right) \leq \frac{1}{c}\left(\frac{\rdv{X}-\theta}{\sqrt{2}\sigma}\right)^{2c}.
\end{align*}
It follows that to ensure $\exx_{\ddist}\rho_{\sigma}(\rdv{X}-\theta;0) < \infty$, it is sufficient if we assume $\exx_{\ddist} \abs{\rdv{X}}^{c} < \infty$ for some $c > 0$. Proceeding to the case of $\alpha = -\infty$, we have
\begin{align*}
\rho_{\sigma}(\rdv{X}-\theta;-\infty) = 1 - \exp\left(-\frac{1}{2}\left(\frac{\rdv{X}-\theta}{\sigma}\right)^{2}\right).
\end{align*}
Any composition of measurable functions is measurable, and since the right-hand side is bounded above by $1$ and below by $0$, we have that $\rho_{\sigma}(\rdv{X}-\theta;-\infty)$ is $\ddist$-integrable without requiring any extra assumptions on $\rdv{X}$ besides measurability. All that remains for the Barron class is the case of non-zero $-\infty < \alpha < 2$, where we have
\begin{align*}
\rho_{\sigma}(\rdv{X}-\theta;\alpha) = \frac{\abs{\alpha-2}}{\alpha} \left(\left(1 + \frac{1}{\abs{\alpha-2}}\left(\frac{\rdv{X}-\theta}{\sigma}\right)^{2} \right)^{\alpha/2} - 1\right).
\end{align*}
Let us break this into two cases: $-\infty < \alpha < 0$ and $0 < \alpha < 2$. Starting with the former case, this is easy since
\begin{align*}
\left(1 + x^{2}\right)^{\alpha/2} = \frac{1}{\left(\sqrt{1+x^{2}}\right)^{-\alpha}}
\end{align*}
which is bounded above by $1$ and below by $0$ for any $\alpha < 0$ and $x \in \RR$, which means the random variable $\rho_{\sigma}(\rdv{X}-\theta;\alpha)$ is $\ddist$-integrable without any extra assumptions on $\rdv{X}$. As for the latter case of $0 < \alpha < 2$, first note that the monotonicity of $(\cdot)^{\alpha/2}$ on $\RR_{+}$ implies
\begin{align*}
(1 + x^{2})^{\alpha/2} \geq \abs{x}^{\alpha}
\end{align*}
which means $\exx_{\ddist}\abs{\rdv{X}}^{\alpha} < \infty$ is necessary. That this condition is also sufficient is immediate from the form of $\rho_{\sigma}(\rdv{X}-\theta;\alpha)$ just given. This concludes the proof; the desired result stated in the lemma follows from setting $\rdv{X} = \loss(h)$ and the observation that the choice of $\theta \in \RR$ in the preceding discussion was arbitrary.
\end{proof}

\begin{proof}[Proof of Lemma \ref{lem:dsp_diffable}]
Referring to the derivatives (\ref{eqn:barron_d1})--(\ref{eqn:barron_d2}) in \S{\ref{sec:proofs_barron_derivs}}, we know that $\dv{\rho}_{\sigma}$ is measurable, and by the proof of Lemma \ref{lem:basic_barron}, we know that $\Abs{\dv{\rho}}_{\infty} < \infty$ for all $\alpha \leq 1$. Thus, as long as $\loss(h)$ is $\sigmafield$-measurable, we have that $\dv{\rho}_{\sigma}(\loss(h)-\theta)$ is $\ddist$-integrable. For the case of $1 < \alpha \leq 2$, note that $\abs{\dv{\rho}_{\sigma}(x)} \leq \abs{x}/\sigma^{2}$ holds, meaning that $\exx_{\ddist}\abs{\loss(h)} < \infty$ implies integrability. Similarly for the second derivatives, from the proof of Lemma \ref{lem:basic_barron}, we see that $\Abs{\ddv{\rho}}_{\infty} < \infty$ for all $-\infty \leq \alpha \leq 1$, implying the $\ddist$-integrability of $\ddv{\rho}_{\sigma}(\loss(h)-\theta)$.

The Leibniz integration property follows using a straightforward dominated convergence argument, which we give here for completeness. Letting $(a_k)$ be any non-zero real sequence such that $a_k \to 0$, we can write
\begin{align*}
\frac{\dif}{\dif\theta}\dsp_{\rho}(h;\theta) & = \lim\limits_{k \to \infty} \frac{\dsp_{\rho}(h;\theta+a_k)-\dsp_{\rho}(h;\theta)}{a_k}\\
& = \lim\limits_{k \to \infty} \exx_{\ddist}\left[ \frac{\rho_{\sigma}(\loss(h)-(\theta+a_k))-\rho_{\sigma}(\loss(h)-\theta)}{a_k} \right].
\end{align*}
For notational convenience, let us denote the key sequence of functions by
\begin{align*}
f_{k} \defeq \frac{\rho_{\sigma}(\loss(h)-(\theta+a_k))-\rho_{\sigma}(\loss(h)-\theta)}{a_k}
\end{align*}
and note that $f_{k} \to f \defeq -\dv{\rho}_{\sigma}(\loss(h)-\theta)$ pointwise as $k \to \infty$. We can then say the following: for all $k$, we have that
\begin{align*}
\abs{f_{k}} \leq \sup_{0 < c < 1}\abs{\dv{\rho}_{\sigma}(\loss(h)-(\theta + c a_k))} \leq g \defeq \abs{\dv{\rho}_{\sigma}(\loss(h)-\theta^{\prime})}
\end{align*}
for an appropriate choice of $\theta^{\prime} \in \RR$. The first inequality follows from the helper Lemma \ref{lem:meanvalue_general}. We can always find an appropriate $\theta^{\prime}$ because the sequence $(a_k)$ is bounded and $\dv{\rho}$ is eventually monotone, regardless of the choice of $\alpha$. With the fact $\abs{f_{k}} \leq g$ in hand, recall that we have already proved that $\exx_{\ddist}g < \infty$ under the assumptions we have made, and thus $\exx_{\ddist}f_{k} \to \exx_{\ddist}f$ by dominated convergence.\footnote{See for example \citet[Thm.~1.6.9]{ash2000a}.} As such, we have
\begin{align*}
\frac{\dif}{\dif\theta}\dsp_{\rho}(h;\theta) = \lim\limits_{k \to \infty} \exx_{\ddist}f_{k} = \exx_{\ddist}f = -\exx_{\ddist}\dv{\rho}_{\sigma}(\loss(h)-\theta),
\end{align*}
which is the desired Leibniz property for the first derivative. A completely analogous argument holds for the second derivative, yielding the desired result.
\end{proof}

\begin{proof}[Proof of Lemma \ref{lem:optimal_threshold_existence}]
From Lemma \ref{lem:dsp_diffable}, we know that the map $\theta \mapsto \dsp_{\rho}(h;\theta)$ is differentiable and thus continuous. Using continuity, taking any $a < b$ and constructing a closed interval $[a,b]$, the Weierstrass extreme value theorem tells us that $\dsp_{\rho}(h;\theta)$ achieves its maximum and minimum on $[a,b]$. Furthermore, note that $\rho$ taken from the Barron class (\ref{eqn:barron}) satisfies all the requirements of our helper Lemma \ref{lem:dsp_coercive}, and thus implies $\dsp_{\rho}(h;\theta) \to \sup\{\rho(x): x \in \RR\}$ as $\abs{\theta} \to \infty$. We can thus always take the interval $[a,b]$ wide enough that
\begin{align*}
\theta \notin [a,b] \implies \dsp_{\rho}(h;\theta) \geq \max_{a \leq x \leq b}\dsp_{\rho}(h;x) \geq \min_{a \leq x \leq b}\dsp_{\rho}(h;x).
\end{align*}
This proves the existence of a minimizer of $\theta \mapsto \dsp_{\rho}(h;\theta)$ on $\RR$.

Next, considering the T-risk $\risk_{\rho}(h;\theta,\eta)$ and minimization with respect to $\theta$, since we are doing unconstrained optimization, any solution $\mt_{\rho}(h;\eta)$ must satisfy $\eta + \dif_{\theta}\dsp_{\rho}(h;\mt_{\rho}(h;\eta)) = 0$, where $\dif_{\theta} \defeq \dif / \dif\theta$. Using Lemma \ref{lem:dsp_diffable} again, this can be equivalently re-written as
\begin{align}\label{eqn:optimal_threshold_existence_0}
\exx_{\ddist}\dv{\rho}_{\sigma}\left(\loss(h)-\mt_{\rho}(h;\eta)\right) = \eta.
\end{align}
When $\alpha > 1$, the derivative of the dispersion function has unbounded range, i.e., $\dv{\rho}_{\sigma}(\RR) = \RR$. As such, an argument identical to that used in the proof of Lemma \ref{lem:dsp_coercive} implies that for any $\eta \in \RR$, we can always find a $\theta_{\eta}(h) \in \RR$ such that (\ref{eqn:optimal_threshold_existence_0}) holds, recalling that continuity follows via Lemma \ref{lem:dsp_diffable}. Combining this with convexity gives us a valid solution. The special case of $\alpha = 1$ requires additional conditions, since from Lemma \ref{lem:basic_barron}, we know that in this case $\abs{\dv{\rho}_{\sigma}} \leq 1/\sigma$, and thus by an analogous argument, whenever $\abs{\mt_{\rho}(h;\eta)} < 1/\sigma$ we can find a finite solution.

To prove uniqueness under $1 \leq \alpha \leq 2$, direct inspection of the second derivative in (\ref{eqn:barron_d2}) shows us that $\ddv{\rho}_{\sigma}(x) > 0$ on $\RR$ whenever we have
\begin{align*}
\frac{2(x/\sigma)^{2}}{\left(1+\frac{(x/\sigma)^{2}}{\abs{\alpha-2}}\right)} < \frac{\abs{\alpha-2}}{1-(\alpha/2)}.
\end{align*}
Re-arranging the above inequality yields an equivalent condition of $(x/\sigma)^{2}(1-\alpha) < \abs{\alpha-2}$, a condition which holds on $\RR$ if and only if $1 \leq \alpha \leq 2$. Since $\ddv{\rho}_{\sigma}$ is positive on $\RR$, this implies that $\exx_{\ddist}\ddv{\rho}_{\sigma}(\loss(h)-\theta) > 0$ for all $\theta \in \RR$. Using Lemma \ref{lem:dsp_diffable}, we have that $\exx_{\ddist}\ddv{\rho}_{\sigma}(\loss(h)-\theta)$ is equal to the second derivative of $\dsp_{\rho}(h;\theta)$ with respect to $\theta$, which implies that $\theta \mapsto \dsp_{\rho}(h;\theta)$ and $\theta \mapsto \dsp_{\rho}(h;\theta) + \eta\theta$ are strictly convex on $\RR$, and thus their minimum must be unique.\footnote{See for example \citet[Sec.~3.1.4]{boyd2004ConvOpt}.}
\end{proof}

\begin{proof}[Proof of Lemma \ref{lem:axiom_check}]
For random loss $\loss$, using Lemma \ref{lem:dsp_diffable}, first-order optimality conditions require
\begin{align}\label{eqn:axiom_check_0}
\exx_{\ddist}\dv{\rho}_{\sigma}\left(\loss-\mloc_{\rho}(\loss)\right) = 0, \qquad \exx_{\ddist}\dv{\rho}_{\sigma}\left(\loss-\mt_{\rho}(\loss;\eta)\right) = \eta.
\end{align}
If these conditions hold, then from direct inspection, the same conditions will clearly hold if we replace $\loss$ by $\loss + c$, $\mloc_{\rho}(\loss)$ by $\mloc_{\rho}(\loss) + c$, and $\mt_{\rho}(\loss;\eta)$ by $\mt_{\rho}(\loss;\eta)+c$. This implies both translation-invariance of the dispersions and the translation-equivariance of the optimal thresholds. Non-negativity follows trivially from the fact that $\rho(\cdot) \geq 0$. Noting that $\rho(x) > 0$ for all $x \neq 0$, we have that $\dsp_{\rho}(\loss;\theta) = 0$ if and only if $\loss = \theta$ almost surely.\footnote{This fact follows from basic Lebesgue integration theory \citep[Thm.~1.6.6]{ash2000a}.} Since $\dsp_{\rho}(\loss;\mloc_{\rho}(\loss)) \leq \dsp_{\rho}(\loss;\mt_{\rho}(\loss;\eta))$ by the optimality of $\mloc_{\rho}(\loss)$, it follows that for any non-constant $\loss$, we must have $\dsp_{\rho}(\loss;\mt_{\rho}(\loss;\eta)) > 0$. Furthermore, from the optimality condition (\ref{eqn:axiom_check_0}) for $\mt_{\rho}(\loss;\eta)$, even when $\loss$ is constant, we must have $\dsp_{\rho}(\loss;\mt_{\rho}(\loss;\eta)) > 0$ whenever $\eta \neq 0$, since $\dv{\rho}_{\sigma}(x) = 0$ if and only if $x=0$.

In the special case where $1 \leq \alpha \leq 2$, we have that $\ddv{\rho}_{\sigma}$ is positive on $\RR$ (see \S{\ref{sec:proofs_barron_derivs}} and Fig.~\ref{fig:dispersion_barron}). This implies that $\rho_{\sigma}$ is strictly convex, and $\dv{\rho}_{\sigma}$ is monotonically increasing. Let $\loss_{1} \leq \loss_{2}$ almost surely, but say $\mloc_{\rho}(\loss_{1}) > \mloc_{\rho}(\loss_{2})$. Using the optimality condition (\ref{eqn:axiom_check_0}), uniqueness of the solution via Lemma \ref{lem:optimal_threshold_existence}, and the aforementioned monotonicity of $\dv{\rho}_{\sigma}$, we have
\begin{align*}
0 = \exx_{\ddist}\dv{\rho}_{\sigma}\left(\loss_{1} - \mloc_{\rho}(\loss_{1})\right) < \exx_{\ddist}\dv{\rho}_{\sigma}\left(\loss_{1} - \mloc_{\rho}(\loss_{2})\right) \leq \exx_{\ddist}\dv{\rho}_{\sigma}\left(\loss_{2} - \mloc_{\rho}(\loss_{2})\right) = 0.
\end{align*}
This is a contradiction, and thus we must have $\mloc_{\rho}(\loss_{1}) \leq \mloc_{\rho}(\loss_{2})$. An identical argument using the exact same properties proves that $\mt_{\rho}(\loss_{1};\eta) \leq \mt_{\rho}(\loss_{2};\eta)$ also holds. Finally, to prove convexity, take any $\loss_{1}, \loss_{2} \in \LL$, $\theta_{1}, \theta_{2} \in \RR$, and $a \in (0,1)$, and note that
\begin{align*}
\underline{\risk}_{\rho}(a\loss_{1} + (1-a)\loss_{2};\eta) & \leq \dsp_{\rho}(a\loss_{1} + (1-a)\loss_{2};a\theta_{1} + (1-a)\theta_{2}) + \eta\left( a\theta_{1} + (1-a)\theta_{2} \right)\\
& = \exx_{\ddist}\rho_{\sigma}\left( a(\loss_{1} - \theta_{1}) + (1-a)(\loss_{2} - \theta_{2}) \right) + \eta\left( a\theta_{1} + (1-a)\theta_{2} \right)\\
& \leq a\left( \dsp_{\rho}(\loss_{1};\theta_{1}) + \eta\theta_{1} \right) + (1-\alpha)\left( \dsp_{\rho}(\loss_{2};\theta_{2}) + \eta\theta_{2} \right).
\end{align*}
The first inequality uses optimality of the threshold in the definition of $\underline{\risk}_{\rho}$, whereas the second inequality uses the convexity of $\rho_{\sigma}$. Since the choice of $\theta_{1}$ and $\theta_{2}$ here were arbitrary, we can set $\theta_{1} = \mt_{\rho}(\loss_{1};\eta)$ and $\theta_{2} = \mt_{\rho}(\loss_{2};\eta)$ to obtain the desired inequality
\begin{align*}
\underline{\risk}_{\rho}(a\loss_{1} + (1-a)\loss_{2};\eta) \leq a\underline{\risk}_{\rho}(\loss_{1};\eta) + (1-a)\underline{\risk}_{\rho}(\loss_{2};\eta)
\end{align*}
giving us convexity of the threshold risk. As a direct corollary, setting $\eta = 0$ yields the convexity result for $\loss \mapsto \dsp_{\rho}(\loss;\mloc_{\rho}(\loss))$.
\end{proof}

\begin{proof}[Proof of Lemma \ref{lem:unbiased_new}]
The crux of this result is an analogue to Lemma \ref{lem:dsp_diffable} regarding the differentials of $\dsp_{\rho}(h;\theta)$, this time taken with respect to $h$, rather than $\theta$. Fixing arbitrary $g, h \in \HH$, let us start by considering the following sequence of random variables:
\begin{align}\label{eqn:unbiased_0}
f_{k} \defeq %
\frac{\rho_{\sigma}(\loss(h+a_{k}g)-\theta)-\rho_{\sigma}(\loss(h)-\theta)}{a_{k}}
\end{align}
where $(a_{k})$ is any sequence of real values such that $a_{k} \to 0_{+}$ as $k \to \infty$. Before getting into the details, let us unpack the differentiability assumption made on the base loss. Before random sampling, the map $h \mapsto \loss(h)$ is of course a map from $\HH$ to the set of measurable functions $\{\loss(h): h \in \HH\}$, but after sampling, there is no randomness and it is simply a map from $\HH$ to $\RR$. Having sampled the random loss, the property we desire is that for each $h \in \HH$, there exists a continuous linear functional $\dv{\loss}(h): \UU \to \RR$ such that
\begin{align}\label{eqn:unbiased_1a}
\lim\limits_{\Abs{g} \to 0} \frac{\abs{\loss(h+g) - \loss(h) - \dv{\loss}(h)(g)}}{\Abs{g}} = 0.
\end{align}
The differentiability condition in the lemma statement is simply that
\begin{align}\label{eqn:unbiased_1b}
\ddist\{ \,\text{equality (\ref{eqn:unbiased_1a}) holds}\, \} = 1.
\end{align}
On this ``good'' event, since the map $x \mapsto \rho_{\sigma}(x)$ is differentiable by definition, we have that the composition $h \mapsto \rho_{\sigma}(\loss(h)-\theta)$ is also differentiable for any choice of $\theta \in \RR$, and a general chain rule can be applied to compute the differentials.\footnote{See \citet[Thm.~2.47]{penot2012CWOD} for this key fact, where ``$X$'' is $\UU$ here, and both ``$Y$'' and ``$Z$'' are $\RR$ here.} In particular, we have a pointwise limit of
\begin{align}\label{eqn:unbiased_1}
f \defeq \lim\limits_{k \to \infty} f_{k} = \dv{\rho}_{\sigma}(\loss(h)-\theta)\dv{\loss}(h)(g)
\end{align}
which also uses the fact that the Fr\'{e}chet and Gateaux differentials are equal here.\footnote{\citet[\S{7.2}, Prop.~2]{luenberger1969Book}} Technically, it just remains to obtain conditions which imply $\exx_{\ddist}f_{k} \to \exx_{\ddist}f$. In pursuit of a $\ddist$-integrable upper bound on the sequence $(f_{k})$, note that for large enough $k$, we have
\begin{align}
\nonumber
\abs{f_{k}} & \leq %
 \frac{1}{a_{k}} %
 \Abs{a_{k}g}\sup_{0 < a < a_{k}}\Abs{\dv{\rho}_{\sigma}(\loss(h+ag)-\theta)\dv{\loss}(h+ag)}\\
 \nonumber
 & \leq \Abs{g} \sup\left\{\Abs{\dv{\rho}_{\sigma}(\loss(h_{0})-\theta)\dv{\loss}(h_{0})} : h_{0} \in \HH\right\}\\
 \label{eqn:unbiased_2a}
 & \leq \frac{\Abs{g}}{\sigma^{2}} \sup\left\{\abs{\loss(h_{0})-\theta}\Abs{\dv{\loss}(h_{0})} : h_{0} \in \HH\right\}.
\end{align}
The key to the first of the preceding inequalities is a generalized mean value theorem.\footnote{Considering the proof of Lemma \ref{lem:meanvalue_general} due to \citet[\S{7.3}, Prop.~2]{luenberger1969Book}, just generalize the one-dimensional part of the argument from the original interval $[0,1]$ to the interval $[0,a_{k}]$ here.} Both the first and second inequalities also use the fact that $h+a_{k}g \in \HH$ eventually. The final inequality uses the fact that $\dv{\rho}_{\sigma}(x) = \dv{\rho}(x/\sigma)/\sigma \leq \abs{x}/\sigma^{2}$ for any choice of $-\infty \leq \alpha \leq 2$. This inequality suggests a natural condition of
\begin{align}\label{eqn:unbiased_2}
\exx_{\ddist}\left[ \sup_{h_{0} \in \HH} \Abs{\loss(h_{0})\dv{\loss}(h_{0})} \right] < \infty
\end{align}
under which we can apply a standard dominated convergence argument.\footnote{See for example \citet[Thm.~1.6.9]{ash2000a}. If (\ref{eqn:unbiased_2a}) holds for say all $k \geq k_{0}$, then we can just bound $\abs{f_{k}}$ by the greater of $\max_{j \leq k_{0}} \abs{f_{j}}$ (clearly $\ddist$-integrable) and the right-hand side of (\ref{eqn:unbiased_2a}).} In particular, the key implication is that
\begin{align}\label{eqn:unbiased_3}
\text{(\ref{eqn:unbiased_2})} \implies \lim\limits_{k \to \infty} \exx_{\ddist}f_{k} = \exx_{\ddist}f.
\end{align}
Since we have
\begin{align*}
\lim\limits_{k \to \infty} \exx_{\ddist}f_{k} = %
\lim\limits_{a \to 0_{+}} %
\exx_{\ddist}\left[\frac{\rho_{\sigma}(\loss(h+ag)-\theta)-\rho_{\sigma}(\loss(h)-\theta)}{a}\right] = \dv{\dsp}_{\rho}(h;\theta)(g),
\end{align*}
where $\dv{\dsp}_{\rho}(h;\theta): \UU \to \RR$ denotes the gradient of $h \mapsto \dsp_{\rho}(h;\theta)$, we see that by applying the preceding argument (culminating in (\ref{eqn:unbiased_3})) to the modified losses (\ref{eqn:transformed_loss}), we readily obtain the desired result.
\end{proof}

\begin{proof}[Proof of Theorem \ref{thm:sgd_convergence}]
To begin, let us consider the smoothness of the objective $h \mapsto \risk_{\rho}(h;\theta,\eta)$ under the present assumptions. From Lemma \ref{lem:smoothness_limited} and the basic properties of the Barron class of dispersion functions (Lemma \ref{lem:basic_barron}), it follows that this function is $\smooth$-smooth with coefficient
\begin{align}\label{eqn:sgd_convergence_1}
\smooth \defeq \frac{\smooth_{5}}{\sigma} + \frac{\smooth_{2}}{\sigma^{2}}\exx_{\ddist}\Abs{\dv{\loss}}_{\HH},
\end{align}
where $\smooth_{5} \defeq \smooth_{3} + \smooth_{4}$, and
\begin{align*}
\smooth_{3} \defeq \left(\frac{\smooth_{2}}{\sigma}\right)\left[\exx_{\ddist}\Abs{\dv{\loss}}_{\HH}^{2} + \sup_{h \in \HH}\exx_{\ddist}\Abs{\dv{\loss}(h)} \right], \qquad \smooth_{4} \defeq \smooth_{1}\Abs{\dv{\rho}}_{\infty}.
\end{align*}
From here, we can leverage the main argument of \citet[Thm.~2]{cutkosky2021a}, utilizing the smoothness property given by (\ref{eqn:sgd_convergence_1}) above, and the $\Gamma$-bound of (\ref{eqn:bounded_gradients}). For completeness and transparency we include the key details here. First, note that if we define $\epsilon_{t} \defeq G_{t} - \partial_{h}\risk_{\rho}(h_{t};\theta,\eta)$, $\widehat{\epsilon}_{t} \defeq M_{t} - \partial_{h}\risk_{\rho}(h_{t};\theta,\eta)$, and $S(h_{t},h_{t+1}) \defeq \partial_{h}\risk_{\rho}(h_{t};\theta,\eta)-\partial_{h}\risk_{\rho}(h_{t+1};\theta,\eta)$, our definitions imply that for each $t \geq 1$, we have
\begin{align}\label{eqn:sgd_convergence_2}
M_{t+1} = \partial_{h}\risk_{\rho}(h;\theta,\eta) + b \left( \widehat{\epsilon}_{t} + S(h_{t},h_{t+1}) \right) + (1-b)\epsilon_{t+1}
\end{align}
Using the form (\ref{eqn:sgd_convergence_2}), it follows immediately that
\begin{align}\label{eqn:sgd_convergence_3}
\widehat{\epsilon}_{t+1} = b \left( \widehat{\epsilon}_{t} + S(h_{t},h_{t+1}) \right) + (1-b)\epsilon_{t+1}
\end{align}
again for each $t \geq 1$. By setting $M_{0} \defeq 0$ and $h_{0} \defeq h_{1}$, we trivially have
\begin{align*}
\widehat{\epsilon}_{0} = -\partial_{h}\risk_{\rho}(h_{0};\theta,\eta) = -\partial_{h}\risk_{\rho}(h_{1};\theta,\eta)
\end{align*}
and one can then easily check that (\ref{eqn:sgd_convergence_3}) holds for all $t \geq 0$. Expanding the recursion of (\ref{eqn:sgd_convergence_3}), we have
\begin{align}\label{eqn:sgd_convergence_4}
\widehat{\epsilon}_{t+1} = (1-b)\sum_{k=0}^{t}b^{k}\epsilon_{t-k+1} + \sum_{k=1}^{t+1}b^{k}S(h_{t-k+1},h_{t-k+2}) + b^{t+1}\widehat{\epsilon}_{0}.
\end{align}
We take the summands of (\ref{eqn:sgd_convergence_4}) one at a time. Since the stochastic gradients are $\Gamma$-bounded, we have that $b^{k}\Abs{\epsilon_{t-k+1}} \leq 2\Gamma$ for all $0 \leq k \leq t$. Furthermore, we have $\exx_{\ddist}\epsilon_{t-k+1} = 0$ (Lemma \ref{lem:unbiased_new}) and $\exx_{\ddist}(b^{k}\Abs{\epsilon_{t-k+1}})^{2} \leq (2b^{k}\Gamma)^{2}$ for each $k$. These bounds can be passed to standard concentration inequalities for martingales on Banach spaces \citep[Lem.~14]{cutkosky2021a}, which using the smoothness property of $\HH$ that we assumed tell us that with probability no less than $1-\delta$, we have
\begin{align}\label{eqn:sgd_convergence_5}
\AbsLR{\sum_{k=0}^{t}b^{k}\epsilon_{t-k+1}} \leq 10\Gamma\max\left\{1,\log(3\delta^{-1})\right\} + 8\Gamma\sqrt{\max\left\{1,\log(3\delta^{-1})\right\}\sum_{k=0}^{t}b^{2k}}.
\end{align}
Moving on to the second term of (\ref{eqn:sgd_convergence_4}), note that using $\smooth$-smoothness of the risk function with coefficient $\smooth$ given by (\ref{eqn:sgd_convergence_1}), along with the definition of the update procedure (\ref{eqn:update_new_1})--(\ref{eqn:update_new_2}), we have
\begin{align*}
\Abs{S(h_{t},h_{t+1})} \leq \smooth\Abs{h_{t}-h_{t+1}} = \smooth a_{t}\Abs{\widetilde{M}_{t}} = \smooth a_{t}.
\end{align*}
This implies that using a constant step size $a_{t} = a$, we can control the sum as
\begin{align}\label{eqn:sgd_convergence_6}
\AbsLR{\sum_{k=1}^{t+1}b^{k}S(h_{t-k+1},h_{t-k+2})} \leq \smooth\sum_{k=1}^{t+1}a_{t-k+1}b^{k} \leq \frac{a\smooth}{1-b}.
\end{align}
Finally, the third term of (\ref{eqn:sgd_convergence_4}) is easily controlled as $\Abs{\widehat{\epsilon}_{0}} = \Abs{\partial_{h}\risk_{\rho}(h_{1};\theta,\eta)} \leq \Gamma$. Taking this bound along with (\ref{eqn:sgd_convergence_5}) and (\ref{eqn:sgd_convergence_6}), we see that $\Abs{\widehat{\epsilon}_{t+1}}$ can be bounded above by
\begin{align}\label{eqn:sgd_convergence_7}
(1-b)\left(10\Gamma\max\left\{1,\log(3\delta^{-1})\right\} + 8\Gamma\sqrt{\max\left\{1,\log(3\delta^{-1})\right\}\sum_{k=0}^{t}b^{2k}}\right) + \frac{a\smooth}{1-b} + b^{t+1}\Gamma
\end{align}
on the high-probability event mentioned earlier. To make use of the bound (\ref{eqn:sgd_convergence_7}), note that using the $\smooth$-smoothness of the losses, the update procedure used here can be shown \citep[Lem.~1]{cutkosky2021a} to satisfy 
\begin{align}\label{eqn:sgd_convergence_8}
\sum_{k=1}^{t}\Abs{\partial_{h}\risk_{\rho}(h_{t};\theta,\eta)} \leq \frac{\risk_{\rho}(h_{1};\theta,\eta)-\risk_{\rho}(h_{t+1};\theta,\eta)}{a} + 2\sum_{k=1}^{t}\Abs{\widehat{\epsilon}_{k}} + \left(\frac{\smooth a}{2}\right)t.
\end{align}
Using a union bound, we have that the bound of (\ref{eqn:sgd_convergence_7}) holds for all $\widehat{\epsilon}_{1},\ldots,\widehat{\epsilon}_{t}$ with probability no less than $1-t\delta$. To get some clean bounds out of (\ref{eqn:sgd_convergence_7}) and (\ref{eqn:sgd_convergence_8}), first we loosen
\begin{align*}
(1-b)\sqrt{\sum_{k=0}^{t}b^{2k}} \leq \frac{1-b}{\sqrt{1-b^{2}}} \leq \frac{1-b}{\sqrt{1-b}} = \sqrt{1-b}
\end{align*}
and note that $0 < \delta < 1$ and $t \geq 1 \geq \mathrm{e}/3$ implies $\log(3t\delta^{-1}) \geq 1$. This means that with probability no less than $1-\delta$, we have
\begin{align*}
2\sum_{k=1}^{t}\Abs{\widehat{\epsilon}_{k}} \leq (1-b)20\Gamma t \log(3t\delta^{-1}) + 16\Gamma t\sqrt{(1-b)\log(3t\delta^{-1})} + \frac{2a\smooth t}{1-b} + \frac{2\Gamma}{1-b}.
\end{align*}
Dividing both sides by $t$ and plugging in the prescribed settings of $a$ and $b$, we have
\begin{align}\label{eqn:sgd_convergence_9}
\frac{2}{t}\sum_{k=1}^{t}\Abs{\widehat{\epsilon}_{k}} \leq \frac{20\Gamma\log(3t\delta^{-1})}{\sqrt{t}} + \frac{16\Gamma\sqrt{\log(3t\delta^{-1})}}{t^{1/4}} + \frac{2\smooth}{t^{1/4}} + \frac{2\Gamma}{\sqrt{t}}.
\end{align}
Taking this bound and our $a$ setting and applying it to (\ref{eqn:sgd_convergence_8}), we obtain
\begin{align}\label{eqn:sgd_convergence_10}
\frac{1}{t}\sum_{k=1}^{t}\Abs{\partial_{h}\risk_{\rho}(h_{t};\theta,\eta)} \leq \frac{\risk_{\rho}(h_{1};\theta,\eta)-\risk_{\rho}(h_{t+1};\theta,\eta)}{t^{1/4}} + (\text{RHS of (\ref{eqn:sgd_convergence_9})}) + \frac{\smooth}{2t^{3/4}}.
\end{align}
To clean this all up, we have
\begin{align*}
\frac{1}{t}\sum_{k=1}^{t}\Abs{\partial_{h}\risk_{\rho}(h_{t};\theta,\eta)} & \leq \frac{1}{t^{1/4}}\left(\risk_{\rho}(h_{1};\theta,\eta)-\risk_{\rho}(h_{t+1};\theta,\eta) + 16\Gamma\sqrt{\log(3t\delta^{-1})} + 2\smooth\right)\\
& \qquad + \frac{1}{\sqrt{t}}\left(20\Gamma\log(3t\delta^{-1}) + 2\Gamma\right) + \frac{\smooth}{2t^{3/4}}.
\end{align*}
For readability, the proof uses a slightly looser choice of $\smooth_{3}$, and instead of $t$ iterations, it is stated for $T$ iterations.
\end{proof}

\begin{proof}[Proof of Corollary \ref{cor:sgd_convergence}]
Let $\alpha = 0$, and note from \S{\ref{sec:proofs_barron_derivs}} that we have
\begin{align*}
\rho_{\sigma}(x) = \frac{2x}{x^{2}+2\sigma^{2}}.
\end{align*}
It thus follows from (\ref{eqn:transformed_grad}) that
\begin{align*}
\partial_{h}\loss_{\rho}(h;\theta,\eta) = \frac{2(\loss(h)-\theta)}{(\loss(h)-\theta)^{2}+2\sigma^{2}} \dv{\loss}(h).
\end{align*}
In the case of the quadratic loss with a linear model as assumed here, this becomes
\begin{align*}
\partial_{h}\loss_{\rho}(h;\theta,\eta) & = \frac{2(\loss(h)-\theta)}{(\loss(h)-\theta)^{2}+2\sigma^{2}} \left(h(\rdv{X})-\rdv{Y}\right)\rdv{X}\\
& = \frac{2(\loss(h)-\theta)\sqrt{2\loss(h)}}{(\loss(h)-\theta)^{2}+2\sigma^{2}}\sign(h(\rdv{X})-\rdv{Y})\rdv{X}.
\end{align*}
Regarding growth in $\rdv{X}$, note that since $\loss(h) = \OO(\Abs{\rdv{X}}^{2})$, both the numerator and denominator are $\OO(\Abs{\rdv{X}}^{4})$. As for growth in $\loss(h)$ which accounts for the random noise $\varepsilon$ as well, the numerator is $\OO(\abs{\loss}^{3/2})$ whereas the denominator is $\OO(\abs{\loss}^{2})$, and thus accounting for the randomness of $\rdv{X}$ and $\varepsilon$ which can both be potentially unbounded and heavy-tailed, we see that the norm of $\partial_{h}\loss_{\rho}(h;\theta,\eta)$ must be bounded as the norm of the inputs, noise, and/or loss grow large. Trivially, since $\sigma > 0$, the limits where $\loss(h) \to \theta$ also result in $\partial_{h}\loss_{\rho}(h;\theta,\eta)$ with a norm that is almost surely bounded.

Proceeding to the case of the logistic loss, an analogous argument yields the desired result. First note that under the linear model assumed here, the gradient with respect to any $h_{j}$ takes the form
\begin{align*}
\partial_{h_{j}}\loss_{\rho}(h;\theta,\eta) = \frac{2(\loss(h)-\theta)}{(\loss(h)-\theta)^{2}+2\sigma^{2}}(p_{j}(h)-\widetilde{\rdv{Y}}_{j})\rdv{X}
\end{align*}
where $p_{j}(h) \defeq \exp(h_{j}(\rdv{X})) / \sum_{i=1}^{k}\exp(h_{i}(\rdv{X}))$, i.e., the softmax transformation of the score assigned by $h_{j}$. By definition, the coefficients $(p_{j}(h)-\widetilde{\rdv{Y}}_{j})$ are bounded. Furthermore, using our linear model assumption, we have that $\loss(h) = \OO(\Abs{\rdv{X}})$, and as such the numerator and denominator are both $\OO(\Abs{\rdv{X}}^{2})$, implying the desired boundedness.

Taking the previous two paragraphs together, we have that the bound (\ref{eqn:bounded_gradients}) is satisfied for a finite $\Gamma$ under $\alpha = 0$, even when the data is unbounded and potentially heavy-tailed. For $\alpha < 0$, the derivative of $\rho$ shrinks even faster, so the same result follows \textit{a fortiori} from the $\alpha = 0$ case. Finally, the $\smooth_{1}$-smoothness assumption in Theorem \ref{thm:sgd_convergence} follows from direct inspection of the forms of $\dv{\loss}(h)$ given here for each loss, using our assumption of $\Abs{\rdv{X}}^{2}$ having bounded second moments.
\end{proof}

\begin{proof}[Proof of Proposition \ref{prop:risk_relations}]
To begin, recall from \S{\ref{sec:trisk}} the notation $\mloc_{\rho}(h) \defeq \argmin_{\theta \in \RR} \dsp_{\rho}(h;\theta)$ for the M-locations and $\mt_{\rho}(h;\eta) \defeq \argmin_{\theta \in \RR} \risk_{\rho}(h;\theta,\eta)$ for the T-risk optimal thresholds. By Lemma \ref{lem:optimal_threshold_existence}, both $\mloc_{\rho}(h)$ and $\mt_{\rho}(h;\eta)$ have unique solutions, and thus we overload this notation to represent the unique solution. By definition, we have
\begin{align*}
\underline{\risk}_{\rho}(h;\eta) = \risk_{\rho}(h;\mt_{\rho}(h;\eta)) & \leq \eta\mloc_{\rho}(h) + \dsp_{\rho}(h;\mloc_{\rho}(h))\\
& = \eta\risk(h) + \eta(\mloc_{\rho}(h)-\risk(h)) + \dsp_{\rho}(h;\mloc_{\rho}(h)).
\end{align*}
Similarly, we can obtain a lower bound using the optimality of $\mloc_{\rho}(h)$ and $\mt_{\rho}(h;\eta)$ as
\begin{align*}
\underline{\risk}_{\rho}(h;\eta) & = \eta\mt_{\rho}(h;\eta) + \dsp_{\rho}(h;\mt_{\rho}(h;\eta))\\
& \geq \eta\mt_{\rho}(h;\eta) + \dsp_{\rho}(h;\mloc_{\rho}(h))\\
& = \eta\risk(h) + \eta(\mt_{\rho}(h;\eta)-\risk(h)) + \dsp_{\rho}(h;\mloc_{\rho}(h)).
\end{align*}
Taking these two bounds together, we have
\begin{align}\label{eqn:risk_relations_1}
\eta(\mt_{\rho}(h;\eta)-\risk(h)) + \dsp_{\rho}(h;\mloc_{\rho}(h)) \leq \underline{\risk}_{\rho}(h;\eta) - \eta\risk(h) \leq \eta(\mloc_{\rho}(h)-\risk(h)) + \dsp_{\rho}(h;\mloc_{\rho}(h))
\end{align}
for any choice of $\eta \in \RR$. The bounds in (\ref{eqn:risk_relations_1}) are stated for ideal risk quantities for the true distribution under $\ddist$, but an identical argument holds if we replace $\ddist$ by the empirical measure induced by an iid sample $\loss_{1},\ldots,\loss_{n}$. Writing this out explicitly, let $\widehat{\dsp}_{\rho}$, $\underline{\widehat{\risk}}_{\rho}$, and $\widehat{\risk}$ denote the empirical analogues of $\dsp_{\rho}$, $\underline{\risk}_{\rho}$, and $\risk$, and similarly let $\widehat{\mloc}_{\rho}$ and $\widehat{\mt}_{\rho}$ be the empirical analogues of $\mloc_{\rho}$ and $\mt_{\rho}$. From the argument leading to (\ref{eqn:risk_relations_1}), it follows that
\begin{align}\label{eqn:risk_relations_2}
\eta(\widehat{\mt}_{\rho}(h;\eta)-\widehat{\risk}(h)) + \widehat{\dsp}_{\rho}(h;\widehat{\mloc}_{\rho}(h)) \leq \underline{\widehat{\risk}}_{\rho}(h;\eta) - \eta\widehat{\risk}(h) \leq \eta(\widehat{\mloc}_{\rho}(h)-\widehat{\risk}(h)) + \widehat{\dsp}_{\rho}(h;\widehat{\mloc}_{\rho}(h)).
\end{align}
Next, using the lower bound in (\ref{eqn:risk_relations_2}), for any $\eta \geq 0$ we have that
\begin{align}
\nonumber
\eta\risk(h) & = \eta \left(\widehat{\risk}(h) + (\risk(h)-\widehat{\risk}(h))\right)\\
\nonumber
& \leq \eta\left(\widehat{\risk}(h) + \Abs{\risk-\widehat{\risk}}_{\HH}\right)\\
\nonumber
& \leq \underline{\widehat{\risk}}_{\rho}(h;\eta)-\eta(\widehat{\mt}_{\rho}(h;\eta)-\widehat{\risk}(h))-\widehat{\dsp}_{\rho}(h;\widehat{\mloc}_{\rho}(h)) + \eta\Abs{\risk-\widehat{\risk}}_{\HH}\\
\label{eqn:risk_relations_3}
& \leq \underline{\widehat{\risk}}_{\rho}(h;\eta) + \eta(\widehat{\risk}(h)-\widehat{\mt}_{\rho}(h;\eta)) + \eta\Abs{\risk-\widehat{\risk}}_{\HH}.
\end{align}
Letting $\widehat{h}$ be a minimizer of $\underline{\widehat{\risk}}_{\rho}$, using the upper bound in (\ref{eqn:risk_relations_2}) and any choice of $h^{\ast}$, we have
\begin{align}
\nonumber
\underline{\widehat{\risk}}_{\rho}(\widehat{h};\eta) \leq \underline{\widehat{\risk}}_{\rho}(h^{\ast};\eta) & \leq \eta\widehat{\risk}(h^{\ast}) + \eta(\widehat{\mloc}_{\rho}(h^{\ast})-\widehat{\risk}(h^{\ast})) + \widehat{\dsp}_{\rho}(h^{\ast};\widehat{\mloc}_{\rho}(h^{\ast}))\\
\label{eqn:risk_relations_4}
& \leq \eta\risk(h^{\ast}) + \eta(\widehat{\mloc}_{\rho}(h^{\ast})-\widehat{\risk}(h^{\ast})) + \widehat{\dsp}_{\rho}(h^{\ast};\widehat{\mloc}_{\rho}(h^{\ast})) + \eta\Abs{\risk-\widehat{\risk}}_{\HH}.
\end{align}
Combining (\ref{eqn:risk_relations_3}) and (\ref{eqn:risk_relations_4}), we have that $\eta\risk(\widehat{h})$ is bounded above by
\begin{align}\label{eqn:risk_relations_5}
\eta\risk(h^{\ast}) + \eta(\widehat{\mloc}_{\rho}(h^{\ast})-\mt_{\rho}(\widehat{h};\eta)) + \eta(\widehat{\risk}(\widehat{h})-\widehat{\risk}(h^{\ast})) + \widehat{\dsp}_{\rho}(h^{\ast};\widehat{\mloc}_{\rho}(h^{\ast})) + 2\eta\Abs{\risk-\widehat{\risk}}_{\HH}.
\end{align}
Some elementary manipulations let us bound the key differences in (\ref{eqn:risk_relations_5}) as
\begin{align*}
\widehat{\mloc}_{\rho}(h^{\ast})-\mt_{\rho}(\widehat{h};\eta) + \widehat{\risk}(\widehat{h})-\widehat{\risk}(h^{\ast}) \leq 2\Abs{\widehat{\mloc}_{\rho}-\risk}_{\HH} + 2\Abs{\risk-\widehat{\risk}}_{\HH} + \Abs{\widehat{\mloc}_{\rho}-\mt_{\rho}}_{\HH},
\end{align*}
and thus dividing by $\eta > 0$, we end up with a final bound taking the form
\begin{align*}
\risk(\widehat{h}) \leq \risk(h^{\ast}) + \Abs{\widehat{\mloc}_{\rho}-\mt_{\rho}}_{\HH} + 2\Abs{\widehat{\mloc}_{\rho}-\risk}_{\HH} + \frac{1}{\eta}\widehat{\dsp}_{\rho}(h^{\ast};\widehat{\mloc}_{\rho}(h^{\ast}))+ 4\Abs{\risk-\widehat{\risk}}_{\HH}.
\end{align*}
The desired result is just the special case where we set $h^{\ast}$ is the expected loss minimizer.
\end{proof}

\subsection{Smoothness computations (proof of Lemma \ref{lem:smoothness_limited})}\label{sec:proofs_smoothness_limited}

Here we provide detailed computations for the smoothness coefficients used in Lemma \ref{lem:smoothness_limited}. We assume here that the assumptions \ref{asmp:grad_moments}, \ref{asmp:loss_smooth}, and \ref{asmp:dispersion} are satisfied. Starting with the difference of expected gradients, using Jensen's inequality and the smoothness assumption \ref{asmp:loss_smooth}, we have
\begin{align}
\nonumber
\Abs{\exx_{\ddist}\left[\dv{\loss}(h_{1})-\dv{\loss}(h_{2})\right]} & \leq \exx_{\ddist}\Abs{\dv{\loss}(h_{1})-\dv{\loss}(h_{2})}\\
\label{eqn:smoothness_limited_1}
& \leq \smooth_{1}\Abs{h_{1}-h_{2}}.
\end{align}
As discussed in \S{\ref{sec:learning_smoothness}}, differences of gradients modulated by $\dv{\rho}$ are slightly more complicated. In particular, recalling the equality (\ref{eqn:difference_tricky}), the norm of the difference
\begin{align}\label{eqn:smoothness_limited_2}
\exx_{\ddist}\dv{\rho}_{\sigma}(\loss(h_{1})-\theta_{1})\dv{\loss}(h_{1}) - \exx_{\ddist}\dv{\rho}_{\sigma}(\loss(h_{2})-\theta_{2})\dv{\loss}(h_{2})
\end{align}
can be bounded above by the sum of
\begin{align}\label{eqn:smoothness_limited_3}
\exx_{\ddist}\Abs{\dv{\loss}(h_{1})}\abs{\dv{\rho}_{\sigma}(\loss(h_{1})-\theta_{1})-\dv{\rho}_{\sigma}(\loss(h_{2})-\theta_{2})}
\end{align}
and
\begin{align}\label{eqn:smoothness_limited_4}
\exx_{\ddist}\abs{\dv{\rho}_{\sigma}(\loss(h_{2})-\theta_{2})}\Abs{\dv{\loss}(h_{1})-\dv{\loss}(h_{2})}.
\end{align}
We take up (\ref{eqn:smoothness_limited_3}) and (\ref{eqn:smoothness_limited_4}) one at a time. Starting with (\ref{eqn:smoothness_limited_3}), from \ref{asmp:dispersion} we know that the dispersion derivative $\dv{\rho}$ is $\smooth_{2}$-Lipschitz, and thus we have
\begin{align}
\nonumber
\text{(\ref{eqn:smoothness_limited_3})} & \leq \left(\frac{\smooth_{2}}{\sigma}\right)\exx_{\ddist}\Abs{\dv{\loss}(h_{1})}\left( \abs{\loss(h_{1})-\loss(h_{2})}+\abs{\theta_{1}-\theta_{2}} \right)\\
\nonumber
& \leq \left(\frac{\smooth_{2}}{\sigma}\right)\exx_{\ddist}\Abs{\dv{\loss}(h_{1})}\left(\Abs{h_{1}-h_{2}}\sup_{0 < c < 1}\Abs{\dv{\loss}((1-c)h_{1}-ch_{2})}+\abs{\theta_{1}-\theta_{2}}\right)\\
\nonumber
& \leq \left(\frac{\smooth_{2}}{\sigma}\right)\left(\Abs{h_{1}-h_{2}}\exx_{\ddist}\Abs{\dv{\loss}}_{\HH}^{2}+\abs{\theta_{1}-\theta_{2}}\sup_{h \in \HH}\exx_{\ddist}\Abs{\dv{\loss}(h)} \right)\\
\label{eqn:smoothness_limited_5}
& \leq \smooth_{3}\left(\Abs{h_{1}-h_{2}} + \abs{\theta_{1}-\theta_{2}}\right).
\end{align}
Here, the second inequality uses the helper Lemma \ref{lem:meanvalue_general} and our assumption of differentiability, while the third inequality uses our assumption on the expected squared norm of the gradient. We have set the Lipschitz coefficient $\smooth_{3}$ in (\ref{eqn:smoothness_limited_5}) to be
\begin{align*}
\smooth_{3} \defeq \left(\frac{\smooth_{2}}{\sigma}\right)\left[\exx_{\ddist}\Abs{\dv{\loss}}_{\HH}^{2} + \sup_{h \in \HH}\exx_{\ddist}\Abs{\dv{\loss}(h)} \right].
\end{align*}
This gives us a bound on (\ref{eqn:smoothness_limited_3}). Moving on to (\ref{eqn:smoothness_limited_4}), if $\dv{\rho}$ is bounded on $\RR$, then we have
\begin{align}
\nonumber
\text{(\ref{eqn:smoothness_limited_4})} & \leq \Abs{\dv{\rho}}_{\infty}\exx_{\ddist}\Abs{\dv{\loss}(h_{1})-\dv{\loss}(h_{2})}\\
\label{eqn:smoothness_limited_6}
& \leq \smooth_{4}\Abs{h_{1}-h_{2}}
\end{align}
with $\smooth_{4} \defeq \smooth_{1}\Abs{\dv{\rho}}_{\infty}$, recalling the bound (\ref{eqn:smoothness_limited_1}). To summarize, we can use (\ref{eqn:smoothness_limited_5}) and (\ref{eqn:smoothness_limited_6}) to control (\ref{eqn:smoothness_limited_2}) as follows:
\begin{align*}
(\text{\ref{eqn:smoothness_limited_2}}) & \leq (\text{\ref{eqn:smoothness_limited_3}}) + (\text{\ref{eqn:smoothness_limited_4}})\\
& \leq \smooth_{3}\left(\Abs{h_{1}-h_{2}} + \abs{\theta_{1}-\theta_{2}}\right) + \smooth_{4}\Abs{h_{1}-h_{2}}\\
& \leq \smooth_{5}\left(\Abs{h_{1}-h_{2}} + \abs{\theta_{1}-\theta_{2}}\right)
\end{align*}
where $\smooth_{5} \defeq \smooth_{3} + \smooth_{4}$. With these preparatory details organized, it is straightforward to obtain a Lipschitz property on the gradient of $\risk_{\rho}(\cdot)$, as summarized in Lemma \ref{lem:smoothness_limited}, and detailed in the proof below.
\begin{proof}[Proof of Lemma \ref{lem:smoothness_limited}]
Using our upper bounds on (\ref{eqn:smoothness_limited_1}) and (\ref{eqn:smoothness_limited_2}), we have
\begin{align*}
&\Abs{\partial_{h}\risk_{\rho}(h_{1};\theta_{1},\eta) - \partial_{h}\risk_{\rho}(h_{2};\theta_{2},\eta)}\\
& \leq \left(\frac{1}{\sigma}\right)\AbsLR{\exx_{\ddist}\left[ \dv{\rho}\left(\frac{\loss(h_{1})-\theta_{1}}{\sigma}\right)\dv{\loss}(h_{1}) - \dv{\rho}\left(\frac{\loss(h_{2})-\theta_{2}}{\sigma}\right)\dv{\loss}(h_{2}) \right]}\\
& \leq \left(\frac{\smooth_{5}}{\sigma}\right)\left(\Abs{h_{1}-h_{2}} + \abs{\theta_{1}-\theta_{2}}\right).
\end{align*}
Next, let us look at the partial derivative taken with respect to the threshold parameter $\theta$. To bound the absolute value of these differences, using the generalized mean value theorem (Lemma \ref{lem:meanvalue_general}), we have
\begin{align*}
\abs{\partial_{\theta}\risk_{\rho}(h_{1};\theta_{1},\eta) - \partial_{\theta}\risk_{\rho}(h_{2};\theta_{2},\eta)} & \leq \left(\frac{1}{\sigma}\right)\left\lvert \exx_{\ddist}\left[\dv{\rho}\left(\frac{\loss(h_{1})-\theta_{1}}{\sigma}\right)-\dv{\rho}\left(\frac{\loss(h_{2})-\theta_{2}}{\sigma}\right)\right] \right\rvert\\
& \leq \left(\frac{\smooth_{2}}{\sigma^{2}}\right)\left(  \exx_{\ddist}\abs{\loss(h_{1})-\loss(h_{2})} + \abs{\theta_{1}-\theta_{2}} \right)\\
& \leq \left(\frac{\smooth_{2}}{\sigma^{2}}\exx_{\ddist}\Abs{\dv{\loss}}_{\HH}\right) \left(\Abs{h_{1} - h_{2}} + \abs{\theta_{1}-\theta_{2}}\right).
\end{align*}
Taking the preceding upper bounds together, the gradient difference for $\risk_{\rho}$ can be bounded as
\begin{align*}
&\Abs{\dv{\risk}_{\rho}(h_{1};\theta_{1},\eta) - \dv{\risk}_{\rho}(h_{2};\theta_{2},\eta)}\\
& = \Abs{ \partial_{h}\risk_{\rho}(h_{1};\theta_{1},\eta) - \partial_{h}\risk_{\rho}(h_{2};\theta_{2},\eta)} + \abs{\partial_{\theta}\risk_{\rho}(h_{1};\theta_{1},\eta) - \partial_{\theta}\risk_{\rho}(h_{2};\theta_{2},\eta)}\\
& \leq \left(\left(\frac{\smooth_{5}}{\sigma}\right) + \frac{\eta\smooth_{2}}{\sigma^{2}}\exx_{\ddist}\Abs{\dv{\loss}}_{\HH}\right)\left(\Abs{h_{1} - h_{2}} + \abs{\theta_{1} - \theta_{2}}\right),
\end{align*}
noting that the initial equality follows from the fact that we are using the sum of norms for our product space norm here (see also Remark \ref{rmk:product_norm}). These bounds on the gradient differences are precisely the desired result.
\end{proof}

\section{Additional technical facts}\label{sec:additional_facts}

\subsection{Lipschitz properties}\label{sec:additional_facts_lipschitz}

Here we give a fundamental property of differentiable functions that generalizes the mean value theorem.
\begin{lem}\label{lem:meanvalue_general}
Let $\UU$ and $\VV$ be normed linear spaces, and let $f:\UU \to \VV$ be Fr\'{e}chet differentiable on an open set $S \subseteq \UU$. Taking any $u \in S$, we have
\begin{align*}
\Abs{f(u+u^{\prime})-f(u)} \leq \Abs{u^{\prime}} \sup_{0 < c < 1}\Abs{\dv{f}(u+cu^{\prime})}
\end{align*}
for any $u^{\prime} \in \UU$ such that $u + cu^{\prime} \in S$ for all $0 \leq c \leq 1$.
\end{lem}
\begin{proof}
See \citet[\S{7.3}, Prop.~2]{luenberger1969Book}.
\end{proof}
Note that Lemma \ref{lem:meanvalue_general} has the following important corollary: \emph{bounded gradients imply Lipschitz continuity}. In particular, if $\Abs{\dv{f}(u)} \leq \smooth < \infty$ for all $u \in S$, then it follows immediately that $f$ is $\smooth$-Lipschitz on $S$.

A closely related result goes in the other direction. Let $f:\UU \to \overbar{\RR}$ be convex and $\smooth$-Lipschitz. If $f$ is sub-differentiable at a point $u \in \UU$, then we have
\begin{align*}
\abs{\langle \partial f(u), u^{\prime}-u \rangle} \leq \abs{f(u^{\prime}) - f(u)} \leq \smooth\Abs{u^{\prime}-u}.
\end{align*}
As such, for convex, sub-differentiable functions, $\smooth$-Lipschitz continuity implies that all sub-gradients are bounded as $\Abs{\partial f(x)} \leq \smooth$.

\subsection{Convexity}\label{sec:additional_facts_convex}

\begin{lem}\label{lem:convexity_hessian}
Let function $f:\RR^{d} \to \RR$ be twice continuously differentiable. Then $f$ is convex if and only if its Hessian is positive semi-definite, namely when
\begin{align*}
\langle \ddv{f}(v)u, u \rangle \geq 0
\end{align*}
for all $u, v \in \RR^{d}$.
\end{lem}
\begin{proof}
See \citet[Thm.~2.1.4]{nesterov2004ConvOpt}.
\end{proof}

\subsection{Derivatives for the Barron class}\label{sec:proofs_barron_derivs}

Let $\rho(\cdot;\alpha)$ be defined according to (\ref{eqn:barron}). Here we compute derivatives of the map $x \mapsto \rho_{\sigma}(x;\alpha)$, using the shorthand notation $\rho_{\sigma}(x;\alpha) \defeq \rho(x/\sigma;\alpha)$. We denote the first derivative of $\rho_{\sigma}(\cdot;\alpha)$ evaluated at $x \in \RR$ by $\dv{\rho}_{\sigma}(x;\alpha)$, which is computed as
\begin{align}\label{eqn:barron_d1}
\dv{\rho}_{\sigma}(x;\alpha) =
\begin{cases}
x/\sigma^{2}, & \text{if } \alpha = 2\\
2x/(x^{2} + 2\sigma^{2}), & \text{if } \alpha = 0\\
(x/\sigma^{2})\exp\left(-(x/\sigma)^{2}/2\right), & \text{if } \alpha = -\infty\\
\frac{x}{\sigma^{2}}\left(1 + \frac{(x/\sigma)^{2}}{\abs{\alpha-2}} \right)^{(\alpha/2)-1}, & \text{otherwise}.
\end{cases}
\end{align}
In the same way, letting $\ddv{\rho}_{\sigma}(x;\alpha)$ denote the second derivative of $\rho_{\sigma}(\cdot;\alpha)$ evaluated at $x \in \RR$, this is computed as
\begin{align}\label{eqn:barron_d2}
\ddv{\rho}_{\sigma}(x;\alpha) =
\begin{cases}
1/\sigma^{2}, & \text{if } \alpha = 2\\
\frac{2}{x^{2} + 2\sigma^{2}}\left( 1 - \frac{2x^{2}}{x^{2} + 2\sigma^{2}} \right), & \text{if } \alpha = 0\\
(1/\sigma^{2})\exp\left(-(x/\sigma)^{2}/2\right)\left(1 - \left(\frac{x}{\sigma}\right)^{2} \right), & \text{if } \alpha = -\infty\\
\frac{1}{\sigma^{2}}\left(1 + \frac{(x/\sigma)^{2}}{\abs{\alpha-2}} \right)^{(\alpha/2)-1}\left(1 - \frac{1-(\alpha/2)}{\abs{\alpha-2}} \frac{2(x/\sigma)^{2}}{1 + (x/\sigma)^{2}/\abs{\alpha-2}}\right), & \text{otherwise}.
\end{cases}
\end{align}
We emphasize that $\dv{\rho}_{\sigma}(x;\alpha)$ and $\ddv{\rho}_{\sigma}(x;\alpha)$ are not equal to $\dv{\rho}(x/\sigma;\alpha)$ and $\ddv{\rho}(x/\sigma;\alpha)$, but by a simple application of the chain rule are easily seen to satisfy the relations
\begin{align*}
\dv{\rho}_{\sigma}(x;\alpha) = \frac{1}{\sigma}\dv{\rho}(x/\sigma;\alpha), \qquad \ddv{\rho}_{\sigma}(x;\alpha) = \frac{1}{\sigma^{2}}\ddv{\rho}(x/\sigma;\alpha)
\end{align*}
for any $x \in \RR$, $\sigma > 0$, and $\alpha \in [-\infty,2]$.

\subsection{Elementary inequalities}\label{sec:additional_facts_ineqs}

The following elementary inequalities will be of use.
\begin{align}\label{eqn:helper_exponential_1}
\left(1 + \frac{x}{p}\right)^{p} \geq \left(1 + \frac{x}{q}\right)^{q}, \qquad \forall\, x \geq 0, \quad p > q > 0
\end{align}
\begin{align}\label{eqn:helper_exponential_2}
(1+x)^{c} \leq 1 + \frac{cx}{1-(c-1)x}, \qquad -1 \leq x < \frac{1}{c-1}, \quad c > 1
\end{align}
The inequality below is sometimes referred to as Bernoulli's inequality.
\begin{align}\label{eqn:helper_bernoulli}
\left(1 + x\right)^{a} \geq 1 + ax, \qquad \forall \, x > -1, \quad a \geq 1.
\end{align}

\subsection{Expected dispersion is coercive}\label{sec:additional_facts_coercive}

\begin{lem}[Expected dispersion is coercive]\label{lem:dsp_coercive}
Let $f:\RR \to \RR_{+}$ be any non-negative function which is even (i.e., $f(x) = f(-x)$ for all $x \in \RR$) and non-decreasing on $\RR_{+}$. Let $\rdv{X}$ be any random variable such that $\exx_{\ddist}f(\rdv{X}-\theta) < \infty$ for all $\theta \in \RR$. Then, we have
\begin{align*}
\lim\limits_{\lvert \theta \rvert \to \infty} \exx_{\ddist}f(\rdv{X}-\theta) = \lim\limits_{x \to \infty}f(x)
\end{align*}
and note that this includes the divergent case where $f(x) \to \infty$ as $\lvert x \rvert \to \infty$.
\end{lem}
\begin{proof}[Proof of Lemma \ref{lem:dsp_coercive}]
By our assumptions, we have $f(x) \geq 0$ and $f(-x) = f(x)$ for all $u \in \RR$, and $f(x_{1}) \leq f(x_{2})$ whenever $0 \leq x_{1} \leq x_{2}$. With these facts in place, note that for any choice of $a \geq 0$ and $\theta$ such that $\lvert \theta \rvert \geq a$, we have
\begin{align}
\nonumber
\exx_{\ddist}f(\rdv{X}-\theta) & = \exx_{\ddist}f(\lvert \theta-\rdv{X} \rvert)\\
\nonumber
& \geq \exx_{\ddist}f\left( \lvert \lvert \theta \rvert - \lvert \rdv{X} \rvert \rvert \right)\\
\nonumber
& \geq \exx_{\ddist}\indic_{\{\lvert \rdv{X} \rvert \leq a\}}f\left(\lvert \lvert \theta \rvert - \lvert \rdv{X} \rvert \rvert\right)\\
\label{eqn:dispersion_coercive_1}
& \geq f\left( \lvert \theta \rvert - a \right)\prr\{\lvert \rdv{X} \rvert \leq a\}.
\end{align}
For readability, let us write $f(+\infty)$ for the limit of $f(x)$ as $\lvert x \rvert \to \infty$. Trivially, we know that $\exx_{\ddist}f(\rdv{X}-\theta) \leq f(+\infty)$. Using the preceding inequality (\ref{eqn:dispersion_coercive_1}), we have a lower bound of $\exx_{\ddist}f(\rdv{X}-\theta) \geq f(+\infty)\prr\{\lvert \rdv{X} \rvert \leq a\}$ that holds for any $a \geq 0$. When $f(+\infty) = \infty$, the desired result is immediate. When $f(+\infty) < \infty$, simply note that $\{\lvert \rdv{X} \rvert \leq a\} \uparrow \samplespace$ as $a \uparrow \infty$, and thus using the continuity of probability measures, we have $\prr\{\lvert \rdv{X} \rvert \leq a\} \to 1$ as $a \to \infty$.\footnote{All countably additive set functions on $\sigma$-fields satisfy such continuity properties \citep[Thm.~1.2.7]{ash2000a}.} Thus, the lower bound (\ref{eqn:dispersion_coercive_1}) can be taken arbitrarily close to $f(+\infty)$, implying the desired result.
\end{proof}

\end{document}